%% file: arxiv_draft.tex
\documentclass{article}

\input{jw_defs.tex}

\numberwithin{equation}{section}
\begin{document}

\title{Toward Guaranteed Illumination Models \\ for Non-Convex Objects}
\author{Yuqian Zhang$^1$, Cun Mu$^2$, Han-wen Kuo$^1$, John Wright$^1$ \vspace{.1in} \\ $^1$Department of Electrical Engineering, Columbia University \\ $^2$Department of Industrial Engineering and Operations Research, Columbia University}
\maketitle

\begin{abstract}
Illumination variation remains a central challenge in object detection and recognition. Existing analyses of illumination variation typically pertain to convex, Lambertian objects, and guarantee quality of approximation in an average case sense. We show that it is possible to build $\mc V$(vertex)-description convex cone models with worst-case performance guarantees, for nonconvex Lambertian objects. Namely, a natural verification test based on the angle to the constructed cone guarantees to accept any image which is sufficiently well-approximated by an image of the object under some admissible lighting condition, and guarantees to reject any image that does not have a sufficiently good approximation. The cone models are generated by sampling point illuminations with sufficient density, which follows from a new perturbation bound for point images in the Lambertian model. As the number of point images required for guaranteed verification may be large, we introduce a new formulation for {\em cone preserving dimensionality reduction}, which leverages tools from sparse and low-rank decomposition to reduce the complexity, while controlling the approximation error with respect to the original cone.
\end{abstract}

\section{Introduction}

Illumination variation remains a central challenge in object detection and recognition. Changes in lighting can dramatically change the appearance of the object, rendering simple pattern recognition techniques such as nearest neighbor ineffective. Various approaches have been proposed to mitigate this problem, for example, using nonlinear features based on gradient orientation \cite{Lowe2004-IJCV}, using quotient images \cite{Shashua2001-PAMI} or total variation regularization \cite{Chen2005}. These approaches are often effective in practice, but can break down under extreme illumination. Moreover, because of the nonlinearity of the feature extraction step, clearly characterizing their domain of applicability is challenging.

An alternative approach is to attempt to explicitly characterize the set of images of the object that can be generated under varying lighting. The seminal work \cite{Belhumeur1998-IJCV} argues that images of a given object with fixed pose and varying illumination should lie near a convex cone in the high-dimensional image space. This conic structure arises as a consequence of nonnegativity of light and linearity of light transport. Many subsequent works have attempted to capture the gross structure of this cone using low-dimensional convex cone or linear subspace models. Motivated by empirical evidence of low-dimensional linear structure in image sets taken under varying illumination (e.g., \cite{Epstein1995}), \cite{Basri2003-PAMI} and \cite{Ramamoorthi2002-PAMI} used an elegant interpretation of the Lambertian reflectance as spherical convolution to argue that for a {\em convex, Lambertian object}, a linear subspace of nine dimensions may suffice to capture most of the variance due to lighting. These models have been used for recognition in many subsequent works \cite{Georghiades2001-PAMI,Wang2009-PAMI,Wright2009-PAMI,Wagner2012-PAMI}, and have been extended in a number of directions \cite{Frolova04-ECCV,Ramamoorthi2005-PAMI}. The promise of subspace or cone models, compared to feature-based approaches described above, is that, by reasoning carefully about the image formation process, it might be possible to guarantee to well-approximate all images of the object under clearly delineated conditions.

It is worth asking, then, {\em what approximation guarantees do current results afford us?} For convex, Lambertian objects, it can be shown that for one or more uniformly random point sources, a nine dimensional spherical harmonic approximation captures {\em on average} about 98\% of the energy \cite{Basri2003-PAMI, Frolova04-ECCV}.
However, per discussion in \cite{Basri2003-PAMI}, low-dimensional linear models do not guarantee quality of approximation for arbitrary extreme illumination conditions. Moreover, for more general nonconvex objects, cast shadows bring in discontinuous changes in radiosity, which render spherical harmonic approximations ineffective \cite{Ramamoorthi2005-PAMI}. Strictly speaking, no rigorous guarantees on quality of approximation are currently known for general nonconvex objects.

In this work, we ask whether it is possible to build models for illumination variation with the following desirable characteristics:
\begin{enumerate}
\item[(i)] UNIFORM GUARANTEES: Guaranteed robustness to worst case lighting, over some clearly specified class of admissible lighting conditions.
\item[(ii)] NONCONVEXITY: Work even for nonconvex objects, with a representation complexity that is adaptive to the complexity of the object of interest.
\item[(iii)] EFFICIENCY: Low storage and computational complexity.
\end{enumerate}
We study these questions in the context of a model problem in object instance verification, in which one is given an object $\obj$ at a fixed pose, and ask whether the input image is an image of this object under some valid illumination condition. We develop rigorous guarantees for this problem, for general (including nonconvex) Lambertian objects. Our results show how to build a model that guarantees to accept every image that can be interpreted as an image of the object under some lighting condition, and to reject every image that is sufficiently dissimilar to all images of the object under valid lighting conditions.

Similar to \cite{Lee2005-PAMI, Mei2009-ICCV, Wagner2012-PAMI}, and many other previous works, we construct a $\mc V$-approximation to the illumination cone, which approximates this cone with the conic hull of a finite collection of images taken under point illuminations. Previous empirical work has suggested that the number of images required for an accurate representation can be large \cite{Mei2009-ICCV}. However, again, for this representation no quantitative results on quality of approximation are currently known. We start from the goal of building a provably correct algorithm for instance verification, and show that in this setting, this reduces to approximating the illumination cone in Hausdorff sense. We derive, in terms of the properties of the object and the scene, sufficient sampling densities for this goal to be met. Our bounds depend on properties of the scene and the object -- in particular, they depend on the level of ambient illumination, and a notion of {\em convexity defect}. They make precise the intuitions that (i) it is more difficult to operate in low-light scenarios, and (ii) nonconvex objects are more challenging than convex objects.

The number of images required to guarantee performance can be large. To address this problem, we introduce a new approach to {\em cone preserving complexity reduction}. This approach uses tools from convex programming -- in particular, sparse and low-rank decomposition \cite{Candes2011-JACM,Chandrasekharan2011-SJO} -- but introduces a new constrained formulation which guarantees that the conic hull of the output will well-approximate the conic hull of the input. The low-rank and sparse decomposition leverages our {\em qualitative} understanding of the physical properties of images (low-dimensionality, sparsity of cast shadows) \cite{Basri2003-PAMI,Ramamoorthi2002-PAMI,Wright2009-PAMI,Candes2011-JACM,Wu2010-ACCV}, while the constraint ensures that the output of this algorithm is {\em always} a good approximation to the target cone. Empirically, we find that the output is often of much lower complexity than the input. This suggests a methodology for building instance verifiers that are both robust to worst case illumination, and computationally efficient.

\section{Problem Formulation and Methodology} \label{sec:cone-app}

\paragraph{Cone Models for Illumination.}

We consider images of size $w \times h$, and let $m = wh$. Each image can be treated as a vector $\mb y \in \Re^m$. We are interested in the set of images of an object $\mc O$ that can be generated under distant illumination. These images form a subset $C_0 \subseteq \Re^m$. Each distant illumination can be identified with a nonnegative function $f : \bb S^2 \to \Re_+$, whose value $f(\mb u)$ is the intensity of light from direction $\mb u$. We use the notation $\mc F$ for the set of nonnegative, Riemann integrable functions on $\sphere^2$.\footnote{To be clear, we call $f$ Riemann integrable iff it is integrable in spherical coordinates: writing $W = [0,2 \pi] \times [0, \pi]$ and $\eta :  W \to \sphere^2$ via $\eta(\theta,\phi) = (\cos \theta \sin \phi, \sin \theta \sin \phi, \cos \phi)$, $f$ is Riemann integrable iff $f \circ \eta \sin \phi$ is Riemann integrable as a function on $W \subseteq \reals^2$. We let $\int_{\mb u} f(\mb u)  d \mb u = \int_W f \circ \eta \sin \phi d(\theta,\phi)$, where the right hand side is a Riemann integral. We reserve the related notation $\int f(\mb u) d \sigma(\mb u)$ for the (Lebesgue) integral with respect to the spherical measure. When $f \in \mc F$, these two integrals coincide.} Mathematically, $\mc F$ is a convex cone: sums of nonnegative, integrable functions are again nonnegative and integrable.

We assume a linear sensor response: the image is a linear function of the incident irradiance.\footnote{This model neglects saturation and quantization.} By linearity of light transport and linearity of the sensor response, the observed image $\mb y \in \Re^m$ is a linear function $\mb y[f]$ of the illumination $f$: if the object is subjected to the superposition $f = f_1 + f_2$ of two illuminations $f_1$ and $f_2$, we observe $\mb y[f_1+f_2] = \mb y[f_1] + \mb y[f_2]$. Since $f$ resides in the convex cone, the set $C_0 \doteq \mb y[ \mc F ] \subset \reals^m$ of possible images is also a convex cone. Note, however, that the fact that $C_0$ is a convex cone holds under very mild assumptions.

The detailed properties of $C_0$ were first studied in \cite{Belhumeur1998-IJCV}, and a great deal of subsequent work has been devoted to understanding its properties \cite{Ramamoorthi2002-PAMI,Basri2003-PAMI,Frolova04-ECCV}. Most of this body of work has been devoted to simple, analytically tractable models such as convex, Lambertian objects. As discussed above, for such simple models, interesting qualitative statements can be made about the gross shape of $C_0$.

The cone $C_0$ can be interpreted as the set of all images of the object under different distant lighting conditions. Intuitively speaking, we expect the problem of representing images $\mb y$ under different illuminations to be more challenging when the light has a stronger directional component. To capture the relative contribution of directional and ambient components of light, we introduce a family of function classes $\mc F_\alpha$, indexed by parameter $\alpha \in [0,\infty)$. Illuminations in $\mc F_\alpha$ consist of an ambient component $\alpha \omega$, where $\omega(\mb u) = 1 / \area{\sphere^2}$ is the constant function on the sphere, and an arbitrary (possibly directional) component $f_d$:
\(
\mc F_\alpha = \set{ f_d + \alpha \omega \mid f_d \in \mc F, \; \norm{f_d}{L_1}  \le 1   },
\)
For each ambient level $\alpha$, we have a cone
\(
C_\alpha \;\doteq\; \reals_+ \cdot \mb y[ \mc F_\alpha ].
\)
For any $\alpha \le \alpha'$, $C_{\alpha'} \subseteq C_{\alpha}$. In this sense, the choice of $\alpha$ induces a tradeoff: as $\alpha$ becomes smaller, $C_\alpha$ becomes more complicated to compute with, but can represent broader illumination conditions. Our complexity bounds in Section  \ref{sec:perturbation} will make this intuition precise. Figure \ref{fig:alpha} shows rendered images of a face under various ambient levels $\alpha \ge 0$. Our methodology is compatible with any choice of $\alpha > 0$.

\begin{figure}[t]
\hspace{10mm}
\begin{minipage}{3in}
\centerline{\includegraphics[width=3in]{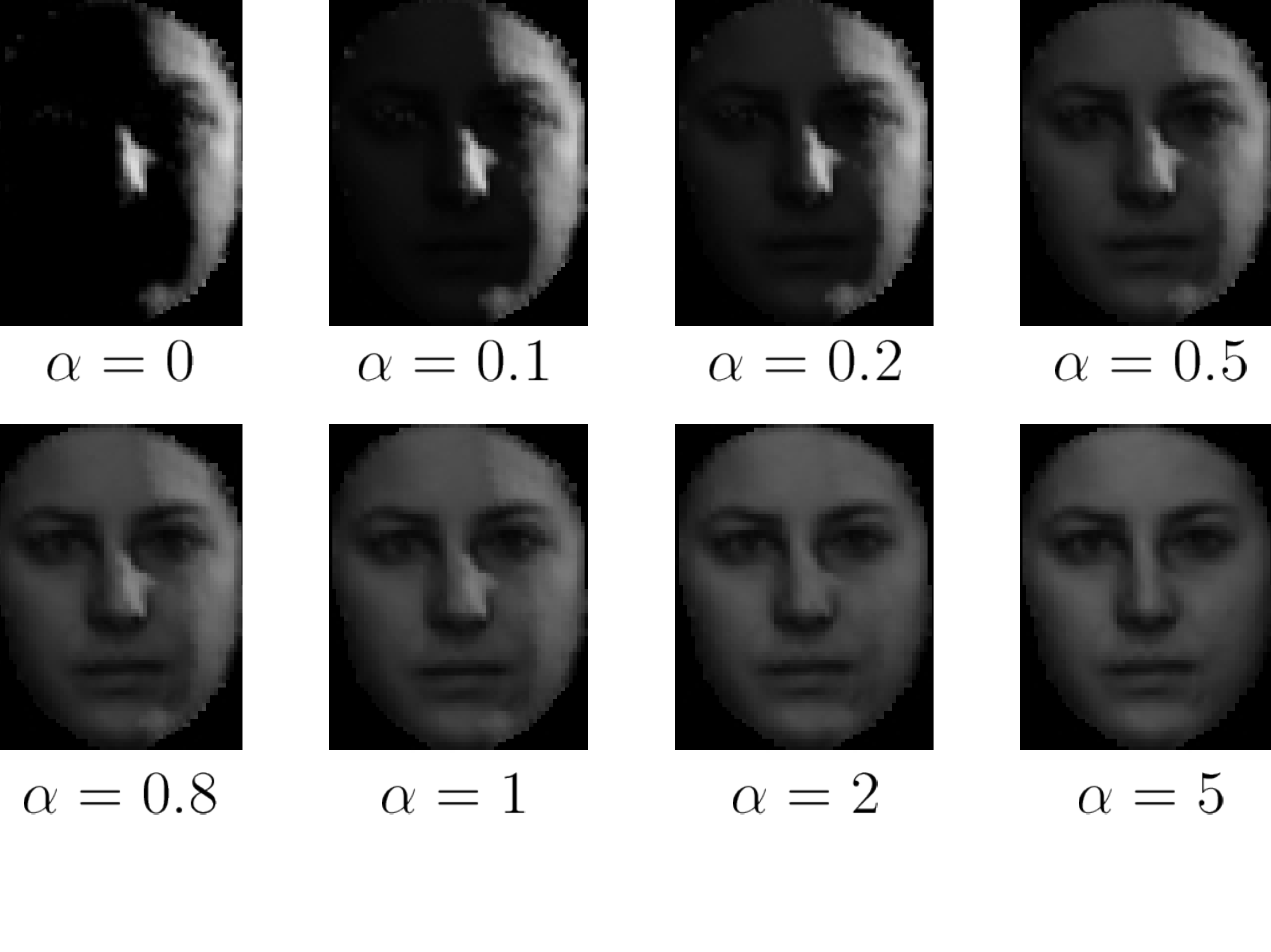}}
\vspace{-10mm}
\end{minipage}
\hspace{20mm}
\begin{minipage}{2in}
\centerline{\includegraphics[width=2.5in]{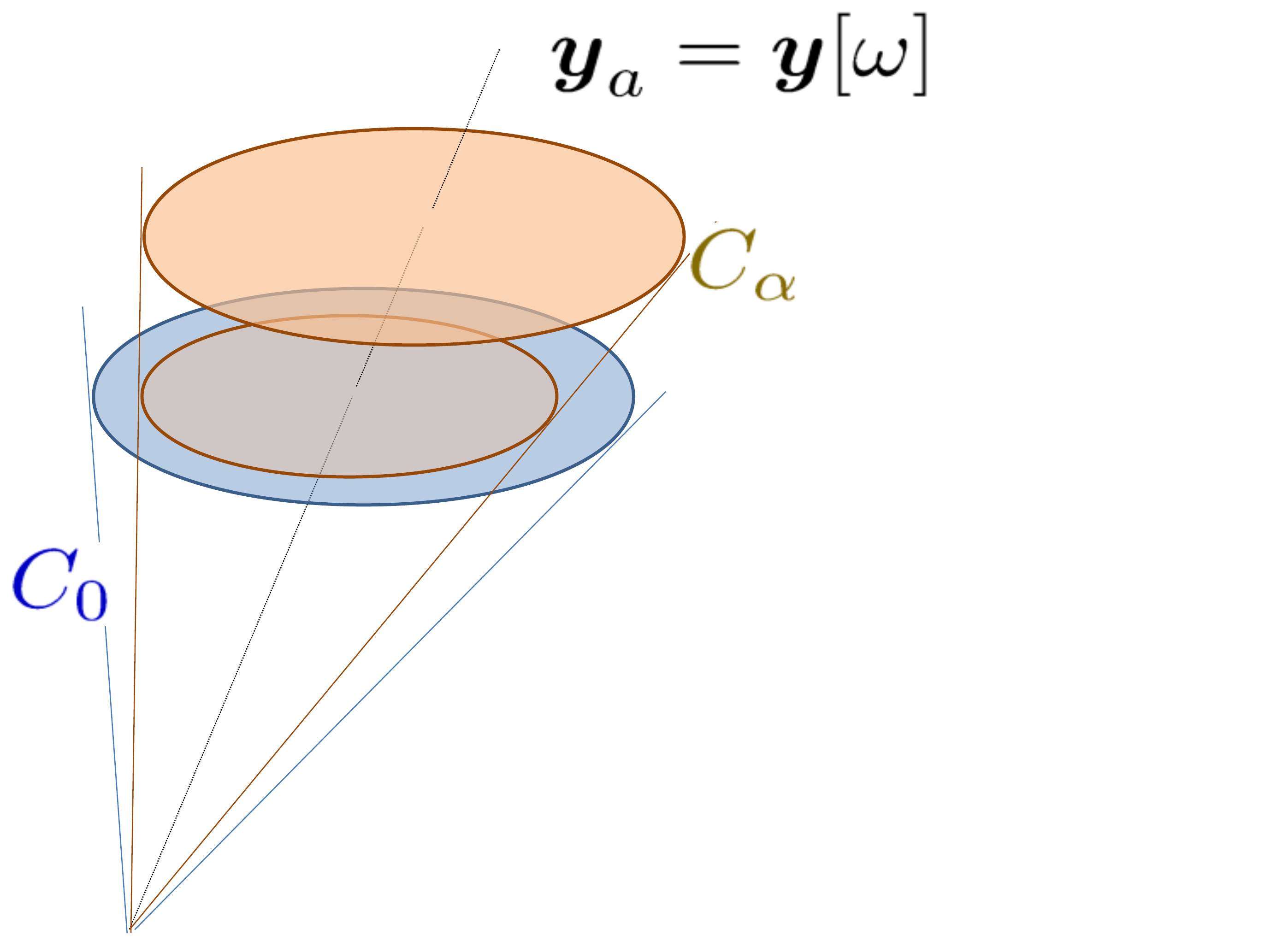}}
\end{minipage}
\caption{{\bf Ambient level $\alpha$.} Left: typical images from the cone $C_\alpha$, for ambient levels $\alpha = 0$ up to $\alpha = 5$. In each example $f_d$ is an extreme directional illumination. Images rendered from \cite{Bosphorus}. Right: illumination cones $C_\alpha$  with varying ambient level $\alpha$.} \label{fig:alpha}
\end{figure}

\paragraph{Verification using Convex Cones.} Our methodology asks the system designer to select a target level of ambient illumination $\alpha$, and hence choose a target cone $C = C_\alpha$. At test time, we are given a new input image $\mathbf{y}\in\Re^{m}$.
The verification problem asks us to decide if $\mathbf{y}$ could be
an image of object $\mathcal{O}$: {\em Is $\boldsymbol{y}$ an element
of $C$?} Real images contain noise and other imperfections. Hence,
in practice, it is more appropriate to ask whether $\mathbf{y}$ is
{\em sufficiently close} to $C$. The distance from $\boldsymbol{y}$ to
$C$ in $\ell^{2}$-norm is
\[
d \left(\boldsymbol{y},C \right) \doteq \inf \left\{ \|\boldsymbol{y}-\boldsymbol{y}'\|_{2} \mid \boldsymbol{y}'\in C \right\}.
\]
Any cone $C$ is nonnegatively homogeneous: if $\mathbf{z}\in C$,
$t\mathbf{z}\in C$ for all $t\ge0$. To obtain a criterion for verification which is scale invariant, rather than working directly with the distance $d(\mb y, C)$, we work with the angle
\[
\angle\left(\mathbf{y},C\right)\doteq\mathrm{asin}\left(\frac{d\left(\mathbf{y},C\right)}{\|\mathbf{y}\|_{2}}\right).
\]
This leads to a simple, natural criterion for verification:
\begin{definition}
The {\em \bf angular detector (AD)} $\mathfrak{D}_{\tau}^{C}:\,\Re_{+}^{m}\to\left\{ \rm{ACCEPT,\,\ REJECT}\right\} $
with threshold $\tau$ is the decision rule
\begin{equation}
\mathfrak{D}_{\tau}^{C}\left(\mathbf{y}\right)=\begin{cases}
\rm{ACCEPT} & \angle\left(\mathbf{y},\, C\right)\le\tau,\\
\rm{REJECT} & \angle\left(\mathbf{y},\, C\right)>\tau.
\end{cases}
\label{eq:angular}
\end{equation}

\end{definition}
\noindent This rule has a simple interpretation: we accept $\mb y$ if and only if it can be interpreted as an image of $\obj$ plus a noise perturbation, and the signal-to-noise ratio is sufficiently large.

If $C$ is a polyhedral cone, the decision rule (\ref{eq:angular})
can be implemented via nonnegative least squares. This is efficient if the number $n$ of extreme rays of $C$ is small. If
$\mathcal{O}$ is a convex polyhedron with only a few faces, this is the case. However, in general, the number of extreme rays in a
$\mathcal{V}$(vertex)-description can be large or even unbounded.\footnote{For convex, Lambertian objects, in a point sampling model of image formation, the best known bound on the number of extreme rays in a $\mathcal{V}$-representation of $C$ is quadratic in the number of image pixels: $n=O(m^{2})$ \cite{Belhumeur1998-IJCV}. For nonconvex objects or more realistic sampling models, $C$ may not even be polyhedral.}
One remedy is to relax the definition slightly:

\begin{definition} $f : \,\Re_{+}^{m}\to\left\{ \rm{ACCEPT,\,\ REJECT}\right\}$ is
an $\eta$-{\em \bf approximate angular detector ($\eta-$AAD)} if
\begin{equation}
f(\mb{y}) =\begin{cases}
\rm{ACCEPT} & \angle\left(\mathbf{y},\, C\right)\le\tau,\\
\rm{REJECT} & \angle\left(\mathbf{y},\, C\right)>\left(1+\eta\right)\tau.
\end{cases}
\end{equation}
We let $\widehat{\mathbb{D}}_{\tau,\eta}^{C}$ denote the set of all such $f$.
\end{definition}

\begin{figure}
\centerline{
\begin{minipage}{3in}
\centerline{\scalebox{.75}{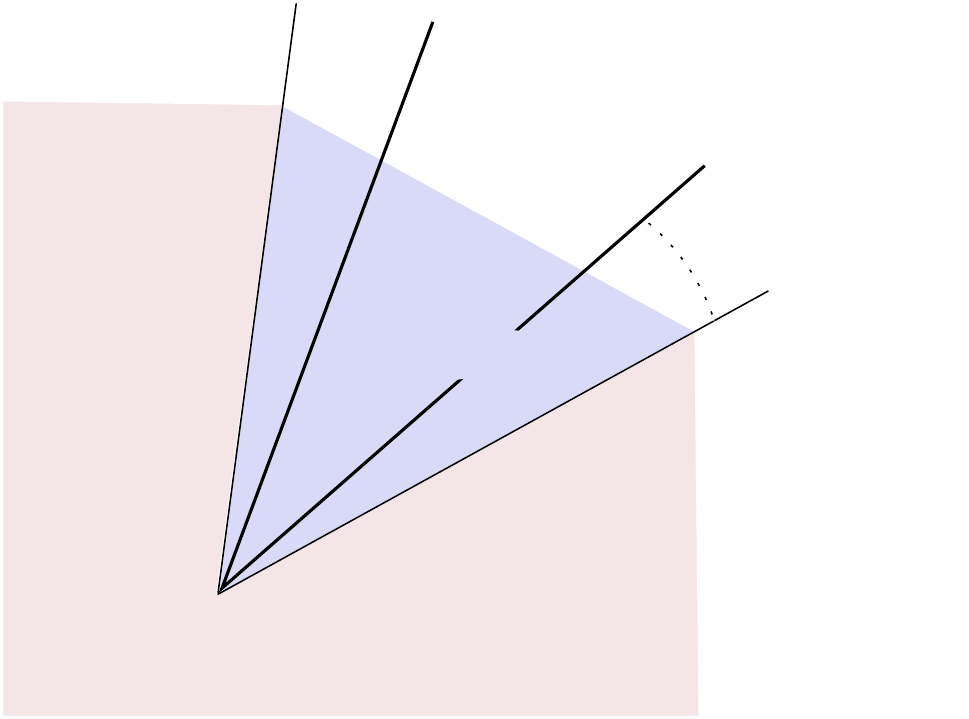}}
\centerline{{\bf \footnotesize Angular Detector}}
\end{minipage}
\begin{minipage}{2.5in}
\centerline{\scalebox{.75}{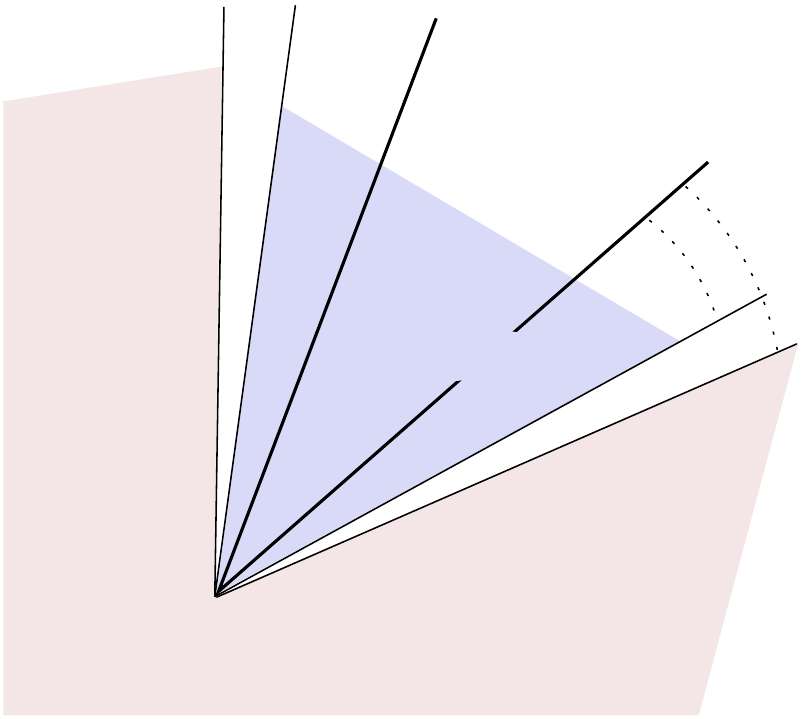} }
\centerline{{\bf \footnotesize $\eta$-Approximate Angular Detector}}
\end{minipage}
}
\caption{{\bf Two detection rules.} The angular detector accepts points based on their angle with the cone $C$. An approximate angular detector guarantees to accept any point within angle $\tau$ of $C$, and to reject any point with angle greater than $(1+\eta) \tau$. In the intermediate region (white) there are no restrictions on its behavior.}
\label{fig:tradeoff}
\end{figure}

\noindent Figure \ref{fig:tradeoff} displays
the AD and its $\eta$-relaxation. We can regard $\eta$-AAD as a relaxed version of AD in the sense
that when $\angle\left(\mathbf{y},\, C\right)\in(\tau,\left(1+\eta\right)\tau]$,
no demands are placed on the output of the algorithm.  This buffer zone allows us to work
with a surrogate cone $\widehat{C}$ with much simpler structure, enabling computationally tractable (even efficient!) verification.
For example, if we form a polyhedral approximation $\widehat{C}=\mbox{cone}(\widehat{\mathbf{A}})$, the distance to $\widehat{C}$ is just the optimal value of the nonnegative least squares problem
\begin{equation}
d(\mathbf{y},\widehat{C})=\min_{\mathbf{x}\ge\mathbf{0}}\|\mathbf{y}-\widehat{\mathbf{A}}\mathbf{x}\|_{2}^{2}.
\label{eq:NNLS}
\end{equation}
To implement the angular detector $\mathfrak{D}_{\xi}^{\widehat{C}}$ for $\widehat{C}$, we just need to solve \eqref{eq:NNLS} and compare the optimal value to a threshold.

It should come as no surprise that whenever $\widehat{C}$ approximates $C$ sufficiently well, we have detector $\mathfrak{D}_{\xi}^{\widehat{C}}\in\widehat{\mathbb{D}}_{\tau,\eta}^{C}$,
with $\xi$ chosen appropriately. In words, applying the angular test with $\widehat{C}$ gives an approximate angular detector for the original cone $C$. To make this precise, we need a notion of approximation. We will work with the following discrepancy $\delta$:
\(
\delta\left(C,\widehat{C}\right) \;=\;\max \left\{ \sup_{\mathbf{y}\in C, \| \mb y \|  = 1 } d(\mathbf{y},\widehat{C}), \; \sup_{\mathbf{y}\in\widehat{C},\|\mathbf{y}\|=1} d(\mathbf{y},C)\right\}.
\label{eq:cone_distance}
\)
This is just the Hausdorff distance between $C \cap \ball{\mb 0}{1}$ and $\widehat{C} \cap \ball{\mb 0}{1}$. It therefore satisfies the triangle inequality: $\forall \, \bar{C}$,
\( \label{eqn:triangle-ineq}
\delta(C,\widehat{C})\;\le\;\delta(C,\bar{C})\,+\,\delta(\bar{C},\widehat{C}).
\)
If $\delta(C,\widehat{C})$ is small, we indeed lose little in working with $\widehat{C}$:

\begin{lemma} \label{lem:cone-appx-gives-AAD-long}
Given cone $C$, $\tau>0$ and $\eta\ge0$ with $\left(1+\eta\right)\tau\in\left(0,\,\frac{\pi}{2}\right)$, and another cone $\widehat{C}$, we have $\mathfrak{D}_{\xi}^{\widehat{C}}\in\widehat{\mathbb{D}}_{\tau,\eta}^{C}$
whenever
\begin{equation} \label{eqn:aad-condition}
\delta\left(C,\widehat{C}\right)\;\le\;\tfrac{1}{2}\left(\sin\left(\tau+\eta\tau\right)-\sin\tau\right)
\end{equation}
and
\begin{equation} \label{eqn:xi-int}
\xi\in\left[{\rm asin}\left(\sin\tau+\delta\left(C,\widehat{C}\right)\right),{\rm asin}\left(\sin\left(\tau+\eta\tau\right)-\delta\left(C,\widehat{C}\right)\right)\right].
\end{equation}
\end{lemma}
\begin{proof} Please see Appendix \ref{sec:s2}.
\end{proof}
%
\noindent So, if $\delta(C,\widehat{C})$ is small, we can simply apply an angular test with cone $\widehat{C}$, and this will implement an approximate angular detector for $C$. Notice that whenever \eqref{eqn:aad-condition} is satisfied, we may satisfy \eqref{eqn:xi-int} by setting $\xi = \mathrm{asin}\left( \tfrac{1}{2}\sin \tau \,+\, \tfrac{1}{2} \sin ( \tau + \eta \tau) \right)$.

\paragraph{Goals and Methodology.} From the above discussion, if we want to provide a detector that guarantees to accept any image that has a valid interpretation as an image of the object under some lighting $f \in \mc F_\alpha$, and reject any image that cannot be plausibly interpreted as an image under $f \in \mc F_\alpha$, it is enough to build an approximation $\widehat{C}$ to the cone $C_\alpha$, and the correct notion of approximation is the Hausdorff distance. The question, then, is how to build such an approximation: how complicated does $\widehat{C}$ have to be to guarantee $\delta( \widehat{C}, C_\alpha ) \le \gamma$? This leads to a way of formalizing several fundamental questions in illumination-robust detection and recognition: {\em What information do we need to guarantee robust verification performance? How does this sample complexity depend on the complexity of the class of illuminations the system must handle? How does it depend on the properties of the object?}

In the sequel, we will show how to build a $\mc V$-approximation $\bar{C} = \mathrm{cone}(\bar{\mb A})$ to $C_\alpha$, where $\bar{\mb A} \in \reals^{m \times n}$ is a matrix whose columns are images under point illumination. The underlying question is how large $n$ needs to be, in terms of ambient illumination level $\alpha$ and the desired quality of approximation $\eps$. We will show that for Lambertian objects,
\(n(\alpha,\eps)\;=\; \frac{\mathrm{const}( \mathtt{sensor}, \mathtt{object} ) }{\alpha^4\eps^{4}}  \label{eqn:n-alpha-eps}\)
examples suffice.  The numerator depends only on physical properties of the object and of the imaging system, which we will make precise below. It is worth remarking that the fact that a polynomial dependence on $\eps^{-1}$ is possible at all can be considered remarkable here -- this is certainly not the case for general high-dimensional convex cones. The reason that such a result is possible at all is that the extreme rays of our cone of interest will turn out to have much lower dimensional structure: they are generated by point illuminations, which are indexed by the sphere. Turning this intuition into a rigorous result will require a detailed analysis of the extreme rays of $C$, which we carry out below.

Section \ref{sec:extreme-rays} characterizes the extreme rays of the cone $C_\alpha$. Section \ref{sec:physics} describes our imaging model in detail. Section \ref{sec:perturbation} describes several new perturbation bounds which lead to the estimate of sample complexity \eqref{eqn:n-alpha-eps}. In Section \ref{sec:Complexity-Reduction}, by solving a convex optimization problem, we form cone $\widehat{C}$, a $\gamma$-approximation to $\bar{C}$, but with much lower complexity. From \eqref{eqn:triangle-ineq}, our resulting cone $\widehat{C}$ $(\eps+\gamma)$-approximates $C_\alpha$: $\delta(C_\alpha,\widehat{C})\le\varepsilon+\gamma$. Finally, Section \ref{sec:Numerical-Experiment} presents several numerical experiments.

\newcommand{\directb}[1]{\bar{\mc D}[ #1 ]}


\section{Extreme Rays of $C_\alpha$} \label{sec:extreme-rays}

In the previous section, we saw that for guaranteed verification with a cone $C$, it was enough to approximate that cone in Hausdorff sense. For computational purposes,  perhaps the most natural approximation is a $\mc V$ (vertex) approximation -- we would like to write $\bar{C} = \cone{ \bar{\mb A} }$ for some matrix $\bar{\mb A}$. To this end, we need to characterize the extreme rays of $C_\alpha$, for $\alpha \ge 0$. We will see below\footnote{For a rigorous argument, please see section \ref{sec:perturbation}.} that for our models of interest, the linear function $\mb y[f]$ can be written as
\(
\label{eqn:int-eq}
\mb y[f] \;=\; \int_{\mb u \in \sphere^2} \, \bar{\mb y}[\mb u] \, f(\mb u) \, d\mb u,
\)
where $\bar{\mb y} : \sphere^2 \to \reals^m$ is a continuous function. In this expression, we have used the natural extension of the Riemann integral to vector-valued functions on the sphere, which simply integrates each of the $m$ coordinate functions.

We begin by characterizing the extreme rays of $C_0 = \mb y[ \mc F ]$. These turn out to simply be the vectors $\barmb{y}[\mb u]$:
\begin{lemma} \label{lem:C0-ext} Suppose that the imaging map $\mb y$ satisfies \eqref{eqn:int-eq}, with $\bar{\mb y}[\cdot] : \sphere^2 \to \reals^m$ continuous. Then if $C_0 = \mb y[ \mc F ]$,
\(
\delta\left( \;C_0\;, \;\, \cone{ \set{\bar{\mb y}[\mb u] \mid \mb u \in \sphere^2 } } \;\, \right) = 0.
\)
\end{lemma}
\begin{proof}
Please see Appendix \ref{app:integrals-cones}.
\end{proof}
\noindent In the physical imaging models we consider, the $\bar{\mb y}[\mb u]$ can be considered images of $\obj$ under {\em point illumination} from direction $\mb u$. With this interpretation, the previous lemma simply asserts that any image $\mb y[f]$ under distant, Riemann integrable illumination $f$ can be arbitrarily well approximated using a conic combination of images under point illumination.\footnote{Informal variants of Lemma \ref{lem:C0-ext} are stated in many previous works in this area; see, e.g., \cite{Belhumeur1998-IJCV}.} The conic hull of these extreme images is equal to the cone $C_0$ of images of $\obj$ under arbitrary Riemann integrable illumination, up to a set of measure zero.

We would like a similar expression that works when the ambient level is larger than zero -- we would like to also approximate the extreme rays of $C_\alpha$. The following lemma says that that we can use images of the form $\bmb{y}[\mb u] = \bar{\mb y}[\mb u] + \alpha \mb y_a$, where $\mb y_a$ is the image of $\obj$ under ambient illumination:

\begin{lemma} \label{lem:ext-amb} Suppose that $\mb y[f]$ satisfies \eqref{eqn:int-eq} with $\bar{\mb y}[\cdot]$ continuous. Set $\bmb{y}[\mb u] = \bar{\mb y}[\mb u] + \alpha \mb y_a$, with
\(
\mb y_a = \frac{1}{\area{\sphere^2}} \int_{\mb u} \bar{\mb y}[\mb u] \, d\mb u,
\)
and $\breve{C} = \cone{ \set{ \bmb{y}[\mb u] \mid \mb u \in \sphere^2 } }$. Then, we have $\delta( C_\alpha, \breve{C} ) = 0$.
\end{lemma}
\begin{proof}
Please see Appendix \ref{app:integrals-cones}.
\end{proof}

\noindent This lemma says that to work with $C_\alpha$, we can simply work with a modified set of extreme images $\bmb{y}[\mb u]$, which are sums of images under point illumination and the ambient image $\mb y_a$. We still need to build a computationally tractable representation for $C_\alpha$. A natural approach is to discretize the set of illumination directions, by choosing a finite set $\mb u_1, \dots, \mb u_N$.  The following lemma asserts that as long as the $\bar{\mb y}[\mb u_i]$ can approximate any point illumination $\bar{\mb y}[\mb u]$ in an absolute sense, the cone generated by the finite set and the cone $C_\alpha$ will not differ too much:

\begin{lemma} \label{lem:ext-appx} Let $\bar{C} = \cone{ \bmb{y}[\mb u_1] , \dots, \bmb{y}[\mb u_N], \mb y_a }$,
and
\(
\label{eqn:delta-point}
 \delta( C_{\alpha}, \bar{C} ) \;=\; \delta( \breve{C}, \bar{C} ) \;\le\; \frac{2 \sup_{\mb u \in \sphere^2} \min_i \norm{ \bar{\mb y}[\mb u] - \bar{\mb y}[\mb u_i] }{2}}{\eta_\star \alpha \norm{\mb y_a}{2}}.
\)
here $\eta_\star \;=\; \sup_{\norm{\mb w}{2} \le 1} \inf_{\mb u} \innerprod{ \mb w }{\tfrac{\bmb{y}[\mb u]}{\norm{\bmb{y}[\mb u]}{2}}} \;\ge\; m^{-1/2}$ measures the angular spread of $C_\alpha$.
\end{lemma}
\begin{proof}
Please see Appendix \ref{app:integrals-cones}.
\end{proof}

This substantially simplifies the problem of approximating $C_\alpha$: to control the error over {\em all} possible images, it is enough to control the error over images under {\em point illumination}. Below, we will see that this is possible, even for nonconvex objects, provided the object's reflectance is Lambertian.

\begin{figure}[h]
\centering
\includegraphics[width=0.5\textwidth]{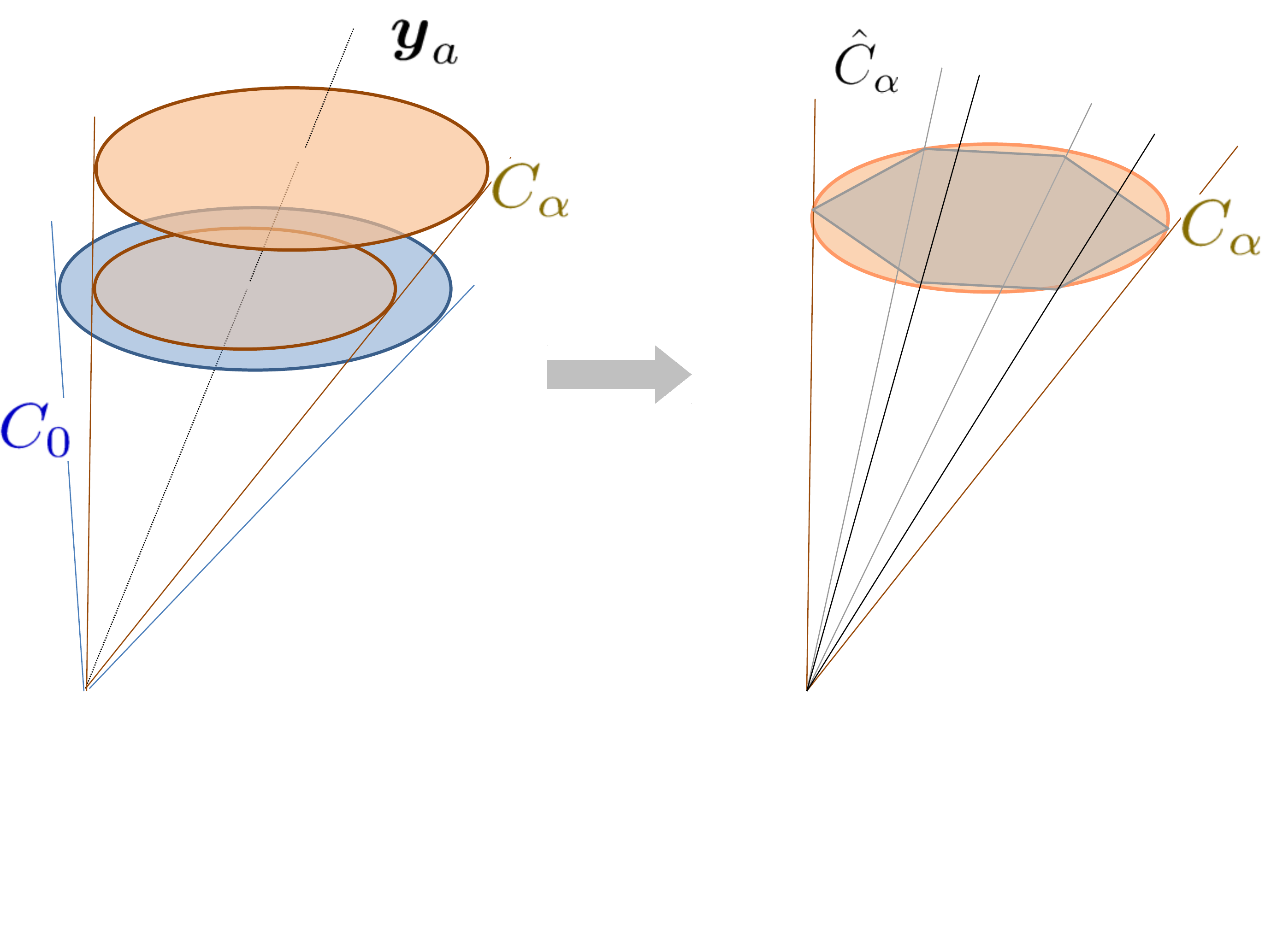}
\vspace{-15mm}\caption{{\bf Cone Approximation:} $C_\alpha$ (left) and its $\mc V$-approximation (right).}
\label{fig:cone_appr}
\end{figure}

\section{Physical Assumptions: Lambertian Objects} \label{sec:physics}

We will introduce a set of hypotheses on the object and the image formation process. Under these hypotheses, we obtain rigorous bounds for the error $\norm{\bar{\mb y}[\mb u] - \bar{\mb y}[\mb u']}{2}$ incurred by approximating an image $\bar{\mb y}[\mb u]$ under point illumination $\mb u$ with another image $\bar{\mb y}[\mb u']$ under point illumination $\mb u'$.  From the results in the previous sections, a good approximation of images under distant point illumination will be sufficient to ensure a good approximation to the cone of all images of the object under distant illumination in Hausdorff sense.

\paragraph{Object Geometry.} Our bounds pertain to {\em triangulated} objects, whose boundary is a union of finitely many oriented triangles:
\begin{definition}[Triangulated object] We say that $\obj \subset \reals^3$ is {\em triangulated} if for some integer $N$,
\begin{eqnarray*}
&& \objbdy = \cup_{i=1}^N \Delta_i, \; \\
&& \forall \, i, \;  \Delta_i = \conv\set{ \mb v_i^{(1)}, \mb v_i^{(2)}, \mb v_i^{(3)} }, \;  \dim{\Delta_i} = 2, \\
&&\forall \, i \ne j, \; \Delta_i \cap \Delta_j \in \set{ \emptyset } \cup \mc V \cup \mc E,
\end{eqnarray*}
where where $\mc V$ and $\mc E$ are sets of vertices and edges:
\begin{eqnarray*}
\mc V = \set{ \set{ \mb v_i^{(k)} } \mid i \in [N], \, k \in [3] }, \quad \mc E = \set{ \conv\set{ \mb v_i^{(k_1)}, \mb v_i^{(k_2)} } \mid i \in [N], \, k_1 \ne k_2 },
\end{eqnarray*}
and each face $\Delta_i$ has a unique outward normal $\mb n_i \in \sphere^2$.
\end{definition}

\noindent This geometric assumption captures most of the object models that are interesting for computer graphics and vision. Notice that $N$ above can be arbitrarily large -- and hence this model can approximate smooth objects.

The normal vectors $\mb n_i$ play an important role in describing how light interacts with the object. If we let
\(
\Phi \;=\; \bigcup_{\Delta_i} \, \relint{ \Delta_i }
\)
be the union of the relative interiors of faces of the object, for $\mb x \in \Phi$, the outward normal $\mb n$ is uniquely defined, and we can write it as $\mb n(\mb x) \in \sphere^2$. We write
\(
E \;\doteq\; \objbdy \setminus \Phi \;=\; \bigcup_{e \in \mc E} e
\)
 for the remaining points. This is the set of all points contained in some edge $e$.

We will introduce two indicator functions that describe how object obstructs the ``view'' from a given point $\mb x \in \objbdy$:
\begin{itemize}
\item The {\em point-direction visibility indicator} $\nu : \objbdy \times \sphere^2 \to \set{0,1}$ indicates those directions $\mb u$, which when viewed from point $\mb x$, are not obstructed by other points of the object:
\(\label{eqn:point-dire-vis}
\nu(\mb x, \mb u) \;=\;
\begin{cases}
1 & \left(\set{\mb x} + \reals_+ \mb u \right) \cap \obj = \set{\mb x}, \\
0 & \text{else}.
\end{cases}
\)

\item The {\em point-point visibility indicator} $V : \objbdy \times \objbdy \to \set{0,1}$ indicates those point pairs $(\mb x,\mb x') \in \objbdy \times \objbdy$ that are mutually visible:
\(\label{eqn:point-point-vis}
V(\mb x, \mb x') \;=\;
\begin{cases}
1 & [\mb x,\mb x'] \cap \obj = \set{\mb x, \mb x'}, \\
0 & \text{else}.
\end{cases}
\)
\end{itemize}

\paragraph{Integrating on $\objbdy$.} To clearly describe how light interacts with the object $\mc O$ to produce an image, we need to be able to integrate on $\objbdy$. This is conceptually straightforward. In this section, we simply introduce notation for this integral; a detailed construction is given in Appendix \ref{app:int-obj-bdy}. There, we formally construct a measure space $(\Phi, \Sigma_{\objbdy}, \mu_{\objbdy})$. For $g : \objbdy \to \reals$, the Lebesgue integral with respect to this measure will be written as
\(
\int g(\mb x) \, d \mu_{\objbdy}(\mb x).
\)
We can define a vector space
\(
L^2[\objbdy] = \set{ g : \objbdy \to \reals \mid g^2 \; \text{is integrable } }.
\)
For $g \in L^2[ \objbdy ]$, we define
\(
\norm{g}{L^2} \;\doteq\; \left( \int g(\mb x)^2 \, d \mu_{\objbdy}(\mb x) \right)^{1/2}.
\)

\paragraph{Object Reflectance.}

We will consider a Lambertian reflectance model. In this model, the object is fully described by its geometry and its {\em albedo}
\(
\rho : \objbdy \to (0,1],
\)
which is the fraction of incoming light that is reflected at each point $\mb x \in \objbdy$. {We assume that the albedo is positive everywhere, and that it is $\Sigma_{\objbdy}$-measurable.} In the Lambertian model, the key quantity linking the illumination $f$ and the image $\mb y$ is the outgoing irradiance (radiosity) at each point $\mb x \in \objbdy$:
\(
g : \objbdy \to \reals_+.
\)
Informally speaking, the irradiance $g(\cdot)$ is generated as follows: light from the source impinges on the surface of the object; some is absorbed, while some is reflected. This reflected light can itself illuminate the object, as can further reflections of the reflected light. Then the map from distant illumination $f$ to outgoing irradiance $g$ can be described in terms of two operators.

The {\bf direct illumination} operator $\mc D : L^2[ \sphere^2 ] \to L^2[ \objbdy ]$ describes the object's reflectance after the first bounce of light from illumination function $f(\mb u)$:
\(
     \direct{f}(\mb x) \;=\; \int \directb{\mb u}(\mb x) \, f(\mb u) \, d\sigma(\mb u), \qquad \mb x \in \Phi.
\)
Here, $\sigma(\cdot)$ is the spherical measure. For Lambertian objects, direct reflectance under point illumination $\bar{\mc D}$ can be expressed as: 
\(
\label{eqn:D-def}
\directb{\mb u}(\mb x) \;=\;
\begin{cases}
\rho(\mb x) \< \mb n(\mb x), \mb u \>_+  \nu (\mb x,\mb u) & \mb x \in \Phi, \\
0 & \text{else}.
\end{cases}
\)

The {\bf interreflection} operator $\mc T$: $L^2[\partial \mc O] \to L^2[\partial \mc O]$ describes how light reflected off the object illuminates the object itself again:
\(
\label{eqn:T-def}
\mc T[g](\mb x) = \begin{cases} \int \kappa( \mb x, \mb x' ) g(\mb x') \, d \mu_{\objbdy}(\mb x') & \mb x \in \Phi \\ 0 & \mb x \in E = \objbdy \setminus \Phi, \end{cases}
\)
where the kernel $\kappa$ is given by
\(
\label{eqn:kappa-def}
\kappa(\mb x,\mb x') \;=\; \frac{\rho(\mb x)}{\pi} \frac{\<\mb n(\mb x'),\mb x-\mb x'\>\<\mb n(\mb x),\mb x'-\mb x\>}{\|\mb x-\mb x'\|^4}\, V(\mb x,\mb x').
\)
For all of the models that we consider, the operator norm of $\mc T$ will be strictly smaller than one, and so the operator $\mc I - \mc T$ will be invertible. Under this assumption the outgoing irradiance on the surface of the object can be written as a convergent series
\begin{eqnarray}
g[f] &=& \direct{f} + \mc T \direct{f}  +\mc T^2 \direct{f} + \dots \nonumber \\
               &=& (\mc I - \mc T)^{-1} \direct{f}.
\end{eqnarray}

\paragraph{Sensor Model.}
We consider a perspective camera, with a thin lens model commonly adopted in computer vision with focal length $f_c$ and lens diameter $d_c$ \cite{Horn}.\footnote{The main idealization in the model \eqref{eqn:image-irradiance} is that it neglects defocus due to depth. In fact, our methodology is compatible with more sophisticated imaging models, as well as simpler idealizations such as orthographic models. However, the bounds claimed in Lemma \ref{lemma:P} will change. } We assume the imaging sensor is composed of $m$ non-overlapping squares $I_i$ with side length $s_c$, then the value of the $i$-th pixel is generated by integrating the irradiance  over region $I_i$:
\(
\label{eqn:image-irradiance}
y_i \;=\; \mc P_i[ g ] \;\doteq\; \frac{\gamma_c}{4}\left( \frac{d_c}{f_c} \right)^2 \int_{\mb z \in I_i } g( \mf p^{-1}(\mb z) ) \innerprod{\frac{\mb z}{\norm{\mb z}{2}}}{\mb e_3}^4 \, d \mu(\mb z).
\)
Here, $\mf p$ represents perspective projection; its inverse maps an image point to the corresponding point on $\objbdy$ and $\gamma_c$ is the camera gain. Combining the expressions for pixels $1 \dots m$, we can describe the image vector as a whole as a linear function of $g$ via
\(
\label{eqn:P-def}
\mb y \;=\; \mc P[g] \;=\; \left[ \begin{array}{c} \mc P_1[g] \\ \vdots \\ \mc P_m[g] \end{array} \right] \;\in\; \reals^m.
\)

\paragraph{The Imaging Operator.}


Combining the definitions and descriptions in the previous paragraphs, we can give a description of the imaging operator $\mb y[f]$ as whole. When $\norm{\mc T}{}< 1$ (i.e., the object is not perfectly reflective), we have
\(
\label{eqn:imaging}
\mb y[f]\;=\;\mc P\sum_{i=0}^{\infty}\mc T^i\mc D(f)\;=\;\mc P(I-\mc T)^{-1}\mc D[f].
\)
Using the definition of $\direct{\cdot}$ and $\directb{\cdot}$, we have

\begin{lemma}\label{lem:imaging} Under the imaging model \eqref{eqn:imaging}, with $\mc P$ as in \eqref{eqn:P-def}, $\mc T$ as in \eqref{eqn:T-def} and $\mc D$ as in \eqref{eqn:D-def}, if $\norm{\mc T}{L^2 \to L^2} < 1$, then for any Riemann integrable $f$ we have
\(
\label{eqn:imaging-lem-1}
\mb y[ f ] \;=\; \int \bar{\mb y}[\mb u] \, f(\mb u) \, d\mb u,
\)
with
\(
\bar{\mb y}[\mb u] \;=\; \mc P ( \mc I - \mc T )^{-1} \directb{\mb u}.
\)
\end{lemma}

\noindent The quantity $\bar{\mb y}[\mb u] \in \reals^m$ in Lemma \ref{lem:imaging} can be interpreted as the image of $\obj$ under point illumination from direction $\mb u$. We will see below that under reasonable hypotheses, $\bar{\mb y}[\mb u]$ is continuous in $\mb u$. From Lemma \ref{lem:ext-appx}, if we can approximate these $\bar{\mb y}[\mb u]$ well, we will well-approximate the cone as a whole.

This proof of the lemma uses Fubini's theorem and monotone convergence to change the order of intergration, and then uses the fact that $\barmb{y}[\mb u]$ is continuous in $\mb u$ to conclude that the integrand in \eqref{eqn:imaging-lem-1} is Riemann integrable. The continuity of $\barmb{y}[\cdot]$ will follow from perturbation bounds in the next section. In Appendix \ref{app:imaging-lem-pf}, we use these results to give a formal proof of Lemma \ref{lem:imaging}.

\section{Perturbation Bounds and Sufficient Sample Densities} \label{sec:perturbation}

Based on the assumptions laid out above, we will discuss the properties of the linear operators $\mc P$, $\mc T$, and $\mc D$, and show how to control $\norm{\bar{\mb y}[\mb u]-\bar{\mb y}[\mb u']}{2}$ in terms of $\norm{\mb u-\mb u'}{2}$. The relationship between $\bar{\mb y}[\mb u]$ and $\mb u$ obviously depends on detailed properties of the object $\obj$.  In particular, it depends on two complementary quantities measuring the {\em nonconvexity} of $\obj$:
\vspace{.1in}

\noindent The {\bf pointwise visibility} is fraction of directions that are visible at point $\mb x$, weighted by $\< \mb n, \mb u\>$:
\(
\label{eqn:nu-til-def}
\tilde{\nu}(\mb x)\doteq{\frac{1}{\pi} \int_{\<\mb u, \mb n(\mb x) \> \ge 0} \hspace{-8mm}\innerprod{\mb n(\mb x)}{\mb u} \nu(\mb x, \mb u) \, d\sigma( \mb u ) }\in [0,1].
\)
here $\nu$ is the point-direction visibility indicator function in equation (\ref{eqn:point-dire-vis}).
\begin{figure}[H]
\centering
\includegraphics[width=0.6\textwidth]{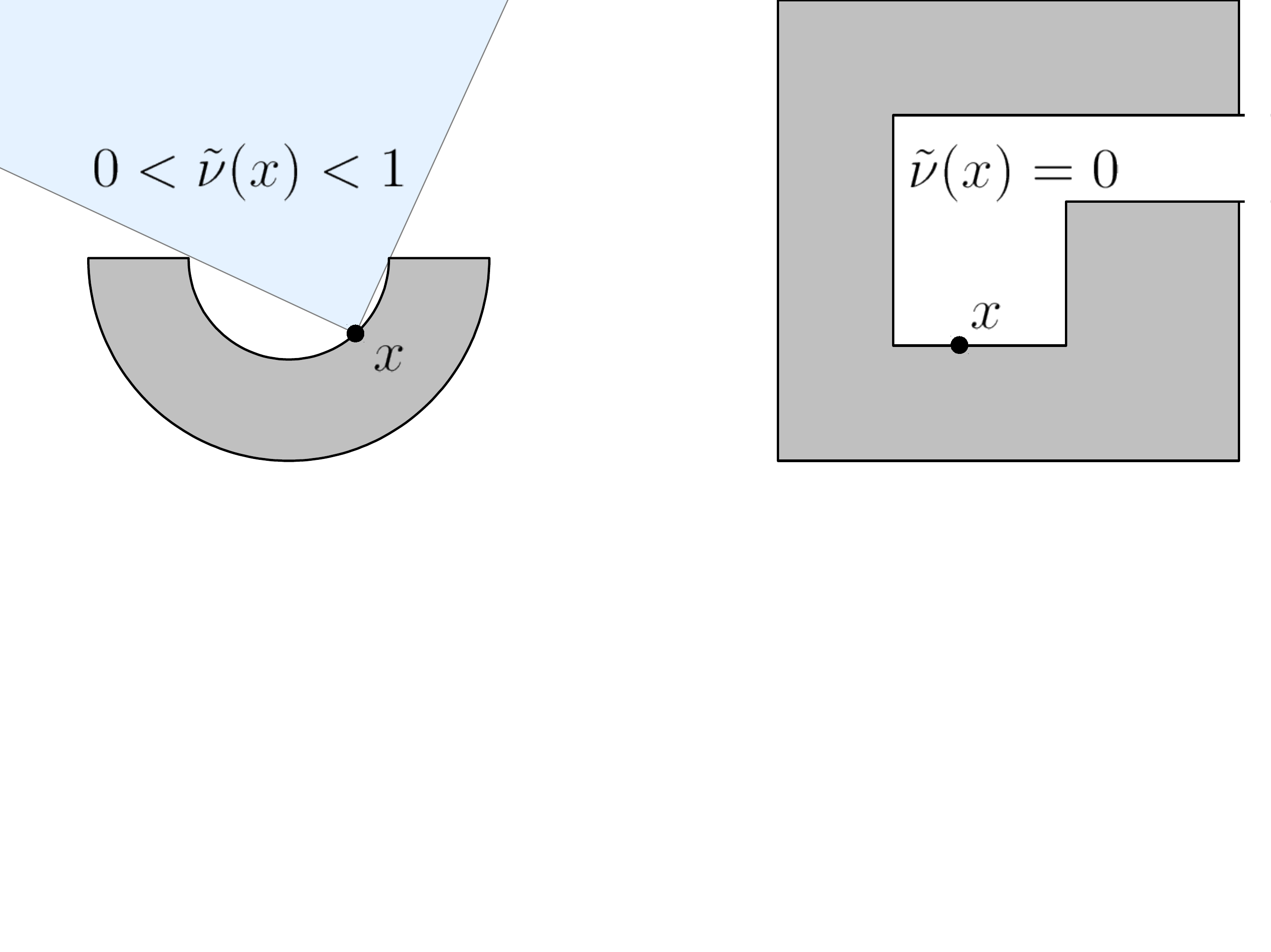}\vspace{-35mm}
\caption{{\bf Pointwise Visibility $\tilde\nu(\mb x)$}}
\label{fig:nu_demo}
\end{figure}
The pointwise visibility $\nu(\mb x)$ is a localized nonconvexity measurement, depending on properties of the object perceived from a point $\mb x$: smaller value of this quantity suggests more complex geometry around $\mb x$. For convex objects, $\nu(\mb x)=1$ for any point $\mb x\in \objbdy$.
\vspace{.1in}

\noindent The other crucial quantity is the total length of the edges that cast shadows on $\obj$ itself, when $\obj$ is illuminated from direction $\mb u$. We call this the {\bf gnomon length} associated with direction $\mb u$.\footnote{The ``gnomon'' is the part of a sundial that casts the shadow.} We reserve the notation $\chi[\mb u]$ for the collection of edges that cast shadows, when the object is illuminated from direction $\mb u$. We will define this quantity formally in the next section, after we have introduced some necessary technical machinery.  For now, Figure \ref{fig:chi_demo}(left) gives a visual example of $\chi[\mb u]$: the edges in this set are highlighted in yellow.
\begin{figure}[H]
\centering
\includegraphics[width=0.6\textwidth]{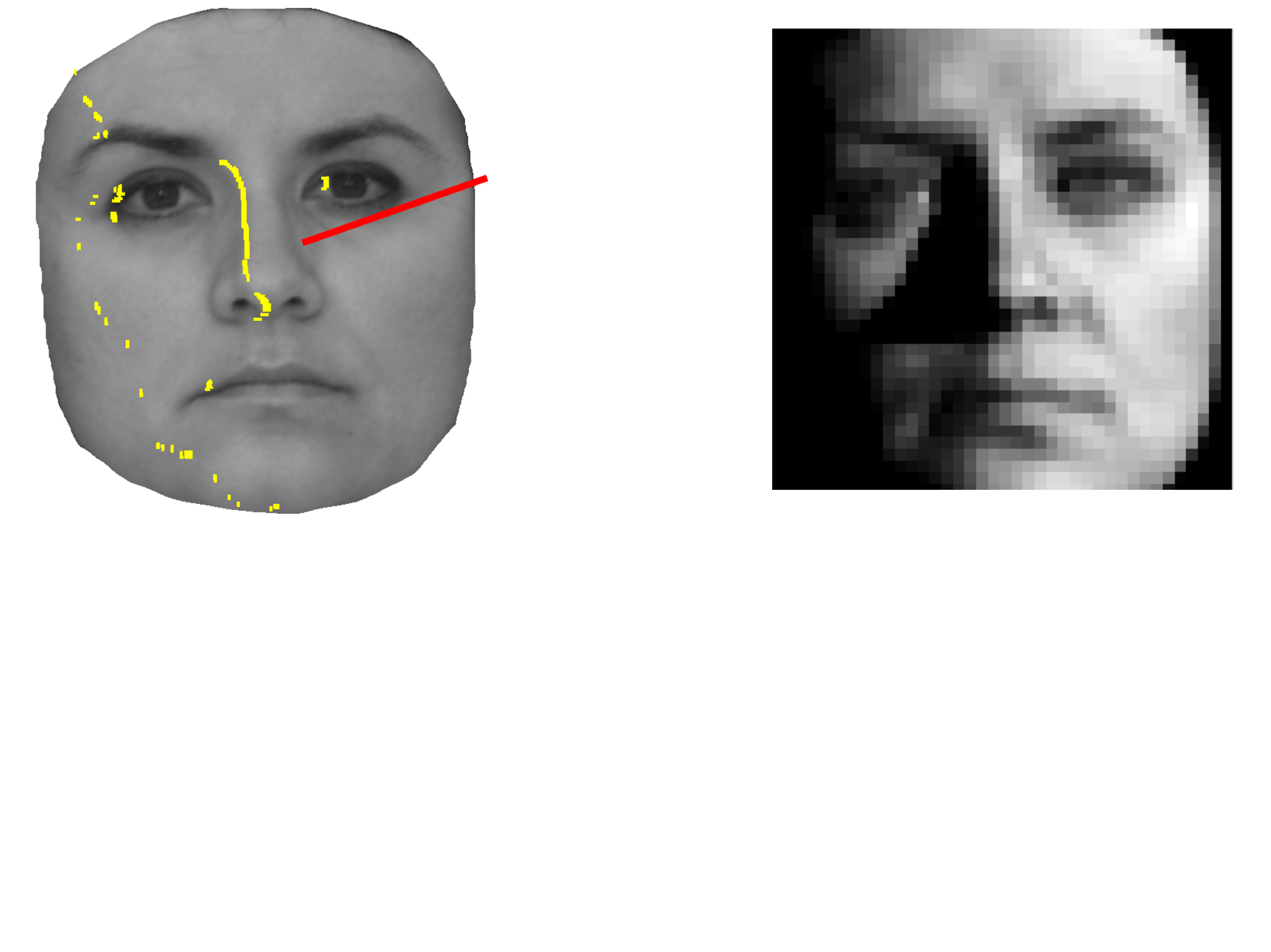}\vspace{-30mm}
\caption{{\bf Shadowing Edges $\chi[\mb u]$} (yellow) under point illumination $\mb u$ (red), with corresponding image on the right.}
\label{fig:chi_demo}
\end{figure}
Compared to $\nu(\mb x)$, $\chi[\mb u]$ is a more global measurement of nonconvexity, depending on the overall geometry of the object: longer {\em shadowing edges length} implies more apparent cast shadows. For convex objects, $\chi[\mb u]=0$ always holds for any illumination direction $\mb u\in\bb S^2$. For nonconvex objects, this quantity helps to bound the change in the image induced by cast shadows, which is a source of considerable difficulty. To state our results more precisely, we begin by introducing some notations and technical machinery for reasoning about the boundary of the shadow region.

\subsection{Shadow Boundaries} \label{subsec:shadow-boundary}
Under lighting direction $\mb u$, the region that is shadowed (not illuminated) is\footnote{Here, the {\em support} $\supp{f} = \set{\mb x \mid f(\mb x) \ne 0 }$ of a function is its set of nonzeros.}
\[
S[\mb u] \;\doteq\; \supp{\directb{\mb u}}^c.
\]
We would like to talk about the boundary of the shadowed region. The follow lemma, which says that the shadowed region $S[\mb u]$ is closed in the relative topology on $\objbdy$, allows us to do so:
\begin{lemma}\label{lem:rc} For all $\mb u \in \sphere^2$, $S[\mb u] \subseteq \objbdy$ is a relatively closed set.
\end{lemma}
\begin{proof} Please see Appendix \ref{app:shadow-proofs}.
\end{proof}
\noindent With this in mind, we can let
\(
\shadowboundary{\mb u} \;\doteq\; \relbdy{S[\mb u]}
\)
denote the shadow boundary, and note that $\shadowboundary{\mb u} \subseteq S[\mb u]$. Points on the shadow boundary  $\shadowboundary{\mb u}$ can be separated into those that come from cast shadows and those that come from attached shadows. For $\mb x \in \objbdy$ and $\mb u \in \sphere^2$, let
\begin{eqnarray}
t_\star(\mb x, \mb u) &\doteq& \inf \set{ t > 0 \mid \mb x - t \mb u \in \objbdy } \;\in\; { [}0,+\infty], \\
t^\star(\mb x, \mb u) &\doteq& \inf \set{ t > 0 \mid \mb x + t \mb u \in \objbdy } \;\in\; { [}0,+\infty],
\end{eqnarray}
where we adopt the standard convention that the infimum of the empty set is $+\infty$. We set
\begin{eqnarray}
\mb x_{\mb u} &=& \mb x - t_\star(\mb x,\mb u) \mb u,  \qquad \forall \; (\mb x,\mb u) \;\; \text{s.t.} \;\; t_\star(\mb x,\mb u) < +\infty. \\
\mb x^{\mb u} &=& \mb x + t^\star(\mb x, \mb u) \mb u,  \qquad \forall \; (\mb x,\mb u) \;\; \text{s.t.} \;\; t^\star(\mb x,\mb u) < +\infty.
\end{eqnarray}
We call $\mb x_{\mb u}$ the {\em shadow projection} of $\mb x$ along direction $\mb u$, and $\mb x^{\mb u}$ the {\em shadow retraction} of $\mb x$ along direction $\mb u$. For light direction $\mb u$, the physical interpretation of the shadow projection of $\mb x$ is that it is the first point that is shadowed by $\mb x$. Conversely, the shadow retraction is the first point that could cast a shadow on $\mb x$. In particular,
\(
\nu(\mb x, \mb u) = 0 \; \iff \; t^\star(\mb x, \mb u) < + \infty.
\)
Notice that because $\obj$ is closed, whenever they exist, we have $\mb x_{\mb u} \in \objbdy$ and $\mb x^{\mb u} \in \objbdy$.

The notion of a shadow retraction allows us to associate to each point $\mb x$ that lies in a cast shadow a point $\mb x^{\mb u}$ which prevents the source from directly illuminating $\mb x$. In particular, if $\mb x$ is in the boundary of a cast shadow, we will see that $\mb x^{\mb u}$ is necessarily an edge point: $\mb x^{\mb u} \in E$. The following technical lemma carries this through precisely:
\begin{lemma}\label{lem:cast-attach} Set $C[\mb u] \;\doteq\; \partial S[\mb u] \cap \Phi$. Then for each $\mb x \in C[\mb u]$, $\mb x^{\mb u}$ exists. If we let
$\chi[\mb u] \doteq \set{ \mb x^{\mb u} \mid \mb x \in C[\mb u] \; }$,
 then $\chi[\mb u] \subseteq E$.
\end{lemma}
\begin{proof} Please see Appendix \ref{app:shadow-proofs}.
\end{proof}

\noindent The physical interpretation is that $C[\mb u]$ contains the boundaries of the {\em cast shadows}. $\chi[\mb u]$ consists of those edges that cast the shadows. The important (and physically quite intuitive) point here is that every point on the boundary of a cast shadow can be identified with an edge point in $\chi[\mb u] \subseteq E$. Hence, it is meaningful to talk about the length of the collection of points $\chi[\mb u]$ that cast shadow edges. 

\subsection{Perturbation bounds}\label{subsec:perturb-bound}
With all the definitions above, we are ready to show how our bounds are phrased in terms of the extreme values of three physical quantities:
\begin{align}
&\text{\bf Maximum length of shadowing edges:} &\chi_\star &\doteq \quad \sup_{\mb u \in \sphere^2} \length{\chi[\mb u]}, \\
&\text{\bf Minimum visibility:} &\nu_\star &\doteq \quad \inf_{\mb x \in \objbdy} \tilde{\nu}(\mb x) \quad\ge\quad  0, \\
&\text{\bf Maximum albedo:} &\rho_\star &\doteq \quad \sup_{\mb x \in \objbdy} \rho(\mb x) \quad\le\quad 1.
\end{align}
For convex objects, $\nu_\star = 1$ and $\chi_\star = 0$. For general objects, $1- \nu_\star$ and $\chi_\star$ can be interpreted as measures of nonconvexity.  In terms of these quantities, we obtain perturbation bounds on $\bar{\mc D}$, $\mc T$ and $\mc P$, which can be combined to bound the error in approximating $\bar{\mb y}[\mb u]$:

\begin{theorem}[Perturbation of direct illumination] \label{thm:D} Suppose that $\obj$ is a triangulated object, and $\rho(\mb x) : \objbdy \to (0,1]$ is strictly positive. Let $\directb{\mb u} \in L^2[\objbdy]$ be as in \eqref{eqn:D-def}. Then for all $\mb u, \mb u' \in \sphere^2$ with $\norm{\mb u - \mb u'}{2} \le \sqrt{2}$, we have
\begin{eqnarray} \label{eqn:D-bound}
\norm{\directb{\mb u} -\directb{\mb u'} }{L^2}^2 \quad\le\quad 2 \, \rho_\star^2 \, \area{\partial \mc O} \norm{\mb u - \mb u'}{2}^2 \;+\; 32\sqrt{2} \, \rho_\star^2 \,\diameter{\mc O} \chi_\star \norm{\mb u - \mb u' }{2}. \qquad
\end{eqnarray}
If $\mc O$ is convex, we have the tighter bound
\begin{equation}  \label{eqn:D-bound-convex}
\norm{\directb{\mb u} - \directb{\mb u'} }{L^2}^2 \;\le \;  \rho_\star^2 \, \area{\partial \mc O} \norm{\mb u - \mb u'}{2}^2.
\end{equation}
\end{theorem}
\noindent The first term of \eqref{eqn:D-bound} accounts for continuous changes induced by $\innerprod{\mb n(\mb x)}{\mb u}_+$. The second term accounts for nonsmooth changes due to cast shadows, which are reflected in the term $\nu(\mb x, \mb u)$. The proof of Theorem \ref{thm:D} is somewhat technical. We delay it to Appendix \ref{sec:direct}.

{After direct illumination, the object is also subject to interreflection, $\mc T$. This operator is also bounded:}

\begin{lemma} \label{lemma:T}
The operator $\mc T$ satisfies $\norm{\mc T}{L^2 \to L^2}\le \rho_\star\cdot(1 - \nu_\star) $.
\end{lemma}
\noindent We prove this bound in Appendix \ref{app:T}. In practice, $\nu_\star$ is bounded away from zero, and $\rho_\star$ is bounded away from one. This implies that $\norm{\mc T}{L^2 \to L^2}<1$, and
\(
\norm{(\mc I-\mc T)^{-1}}{L^2 \to L^2} \;\le\; ( 1-\rho_\star\cdot(1-\nu_\star) )^{-1}.
\)
This inequality controls the total effect of interreflection for all bounces. For convex objects,  $\nu_\star=1$, and $\mc T$ does not participate in the image formation process.

{At last, the effect of projection and sampling can be controlled in terms of the properties of the sensor:}

\begin{lemma}\label{lemma:P} Under our sensor model, let $\ell = \min \set{ \< \mb e_3, \mb x \> \mid \mb x \in \obj }$ be the depth of the object, and set $\beta_c = \frac{\gamma_c}{4}\frac{d_c^2}{f_c^2}$. The projection and sampling operator $\mc P$ satisfies
\begin{equation}
\norm{\mc P}{L^2 \to \ell^2}\;\,\le\;\, 2^{1/4}\, \beta_c f_c s_c / \ell.
\end{equation}
\end{lemma}

Finally, putting these three bounds together, we obtain a perturbation for the images $\barmb{y}[\cdot]$ of $\obj$ under point illumination below:
\begin{theorem}\label{thm:main-perturbation} Under our hypotheses, for any $\mb u$, $\mb u' \in \sphere^2$ with $\norm{\mb u - \mb u'}{2} \le \sqrt{2}$,
\[
\norm{\bar{\mb y}[\mb u] - \bar{\mb y}[\mb u']}{2} \;\le\; 2^{1/4}\frac{ \, \beta_c \rho_\star f_c s_c}{\ell ( 1 - \rho_\star(1-\nu_\star))} \times \left( 2\, \area{\objbdy} \norm{\mb u - \mb u'}{2}^2 +32\sqrt2 \, \diameter{\obj} \chi_\star \norm{\mb u - \mb u'}{2} \right)^{1/2},
\]
\end{theorem}

\paragraph{Number of Sample Images.}
To our knowledge, Theorems \ref{thm:D}-\ref{thm:main-perturbation} are new, and could be useful for other problems in vision and graphics. This bound depends only on properties of the object and imaging system that can be known or estimated. In conjunction with Lemma \ref{lem:ext-appx}, it gives a guideline for choosing the sampling density that {\em guarantees} a representation that works for every illumination $f \in \mc F_\alpha$.

In particular, as $\norm{\mb u - \mb u'}{2}\to 0$, Theorem \ref{thm:main-perturbation} suggests that $\norm{\bar{\mb y}[\mb u] - \bar{\mb y}[\mb u']}{2}$ is proportional to $\norm{\mb u - \mb u'}{2}^{1/2}$. We can deduce that for guaranteed $\eps$-approximation verification with ambient illumination level $\alpha$, it would require
\(
\label{eqn:sample_number_order}
n(\alpha,\eps)\;=\; \frac{\mathrm{const}( \mathtt{sensor}, \mathtt{object} ) }{(\alpha\eps)^{4}}
\)
sample images --  polynomial in the approximation error $\eps$, ambient level $\alpha$ and dimension $m$.  This is possible due to the very special structure of the extreme rays $\bar{\mb y}[\mb u]$ of $C_\alpha$: they lie on a submanifold of dimension $2$. In contrast, general cone approximation in $\reals^m$ requires a number of samples exponential in $m$ \cite{Bronshteyn1976}.

\section{Cone Preserving Complexity Reduction\label{sec:Complexity-Reduction}}
Although the sample complexity $n(\alpha,\eps)$ in \eqref{eqn:sample_number_order} is polynomial in $\eps^{-1}$, it can still be very large. This makes working directly with the dictionary $\bar{\mb A} \in \reals^{m \times n}$ problematic in practice.
Hence, we would like to find a surrogate $\widehat{\mathbf{A}}$ that is structured in such a way as to enable efficient computation, while still belonging to the set
\(
\Omega_{0} \! \doteq\!\left\{\! \widehat{\mathbf{A}}\!\in\!\Re^{m\times n}\;\middle|\; \delta\left({\rm cone(\bar{\mb A})},{\rm cone(\widehat{\mathbf{A}})}\right)\!\le\!\gamma\right\}
\)
for guaranteed verification. If $\widehat{\mb A}$ can be expressed as $\mb L + \mb S$, where $\mb L$ has rank $r$ and $\mb S$ has $k$ nonzero entries, product $\widehat{\mb A} \mb x$ can be computed in time $O( (m+n)r + k)$, much smaller than $O(mn)$. Empirical evidence suggests that this model gives a good approximation for images under varying illuminations (Figure \ref{fig:ls_demo}): the low-rank term captures the smooth variations \cite{Basri2003-PAMI}, while cast shadows are often sparse \cite{Wright2009-PAMI}. The effectiveness of such model has been noted, e.g., in \cite{Candes2011-JACM}, and exploited for robust photometric stereo by \cite{Wu2010-ACCV}.
\begin{figure}[h]
\centering
\includegraphics[width=4in]{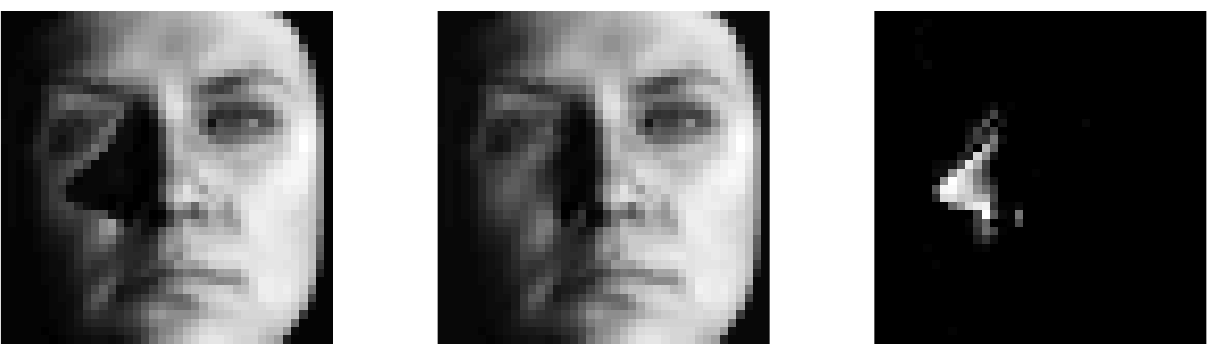}
\caption{{\bf Low Rank + Sparse Decomposition.} Left: input $\bar{\mb A}$. Middle: low-rank term $\mb L$. Right: sparse term $\mb S$.}
\label{fig:ls_demo}
\end{figure}

To build a framework for complexity reduction with guaranteed approximation quality, we start with the following problem, which seeks the {\em lowest-complexity} pair $(\mb L,\mb S)$ that suffice for guaranteed verification:
\begin{eqnarray}
\label{eq:low_rank+sparse}
 & \min_{\left(\mathbf{L},\mathbf{S}\right)} & {\rm rank}\left(\mathbf{L}\right)+\lambda\|\mathbf{S}\|_{0}\nonumber \\
 & {\rm {s.t.}} & \mathbf{L}+\mathbf{S}=\widehat{\mathbf{A}}\in\Omega_{0}.
\end{eqnarray}
Note that the constraint illustrates a basic difference between our setting here and all of the aforementioned works on low-rank and sparse recovery. Previous works aimed at statistical estimation of $\mb L$ and $\mb S$, and hence worked with simple constraints of the form $\norm{ \mb L + \mb S - \bar{\mb{A}} }{F} \le \eps$. Here, we care about preserving the performance guarantee for detection -- in particular, ensuring that $\cone{\mb L + \mb S}$ and $\cone{\bar{\mb{A}}}$ are close in Hausdorff sense. This forces us to work with a more complicated set $\Omega_0$ of matrices, giving a very different optimization problem. The following result shows how to convexify $\Omega_0$, to obtain a tractable convex optimization problem whose solution is guaranteed to well-approximate $\cone{\bar{\mb{A}}}$, {\em in the sense required by the application}.

One immediate relaxation is to replace nonconvex objectives rank and $\ell^{0}$-norm with their convex envelope nuclear norm $\|\cdot\|_{*}$ (sum of all singular values) and the $\ell^{1}$-norm (sum of absolute values of all entries) respectively.
For the nonconvex domain $\Omega_{0}$, we will instead work on one of its convex subsets $\Omega_1$, which is defined via a bound on a supremum of convex functions of $\widehat {\mb A}$ as follows:

\begin{lemma}\label{lemma:lrsd}
If $\gamma'\le\frac{\gamma}{\gamma+1},$ we have $\Omega_{1}\subseteq\Omega_{0}$, where
\(
\Omega_{1}\doteq\left\{ \widehat{\mathbf{A}}\!\in\!\Re^{m\times n}\;\middle|\;\max_{\mathbf{x}\ge\mathbf{0},\norm{\bar{\mb A}\mathbf{x}}{2}\le1}\norm{\bar{\mb A}\mathbf{x}-\widehat{\mathbf{A}}\mathbf{x} }{2}\le\gamma'\right\}.  \label{eqn:omega1}
\)
\end{lemma}
While $\Omega_1$ is convex, it does not have a tractable description, because of the quadratic maximization involved. We use a standard lifting trick, writing $\mathbf{X}=\mathbf{x}\mathbf{x}^{T} \succeq \mb 0$, to relax this quadratic maximization to a (convex) linear maximization over a semidefinite unknown $\mb X$. This gives an upper bound on the maximum in \eqref{eqn:omega1}, giving another convex subset $\Omega_2 \subseteq \Omega_1$, which {\em does} admit a tractable representation:
\begin{lemma} \label{lemma:lifting}
Consider \[\Omega_{2}\doteq\left\{ \widehat{\mathbf{A}}\!\in\!\Re^{m\times n}\;\middle|\;\max_{\mathbf{X}\in\mathcal{X}}\left\langle(\widehat{\mathbf{A}}-\bar{\mb{A}})^{T}(\widehat{\mathbf{A}}-\bar{\mb{A}}),\,\mathbf{X}\right\rangle\le(\gamma')^2\right\},\]
where $\mathcal{X}\doteq\left\{ \mathbf{X}\!\in\!\Re^{n\times n}\;\middle|\; \mathbf{X}\ge\mathbf{0},\,\mathbf{X}\succeq0,\,\left\langle\bar{\mb{A}}^T\bar{\mb{A}},\,\mathbf{X}\right\rangle\le1\right\} .$
Then we have $\Omega_{2}\subseteq\Omega_{1}.$
\end{lemma}

\noindent Finally, we reformulate $\Omega_{2}$ via the dual problem of $\max_{\mathbf{X}\in\mathcal{X}}\left\langle(\widehat{\mathbf{A}}-\bar{\mb{A}})^{T}(\widehat{\mathbf{A}}-\bar{\mb{A}}),\mathbf{X}\right\rangle$:
\begin{theorem}\label{thm:cp-lrsd}
Let $(\mb L^\star, \mb S^\star)$ solve
\begin{eqnarray}
\label{eqn:cp-lrsd}
 & \min_{\left(\mathbf{L},\mathbf{S},\mathbf{\mu}\right)} & \|\mathbf{L}\|_{*}+\lambda\|\mathbf{S}\|_{1}\nonumber \\
 & {\rm {s.t.}} & \left[\begin{smallmatrix}
\mathbf{I} & \mathbf{L}+\mathbf{S}-\bar{\mb A}\\
\left(\mathbf{L}+\mathbf{S}-\bar{\mb A}\right)^{T} &\bar{\gamma}\bar{\mb A}^{T}\bar{\mb A}-\mathbf{\mu}
\end{smallmatrix}\right]\succeq\mathbf{0}\label{eq:reform_prob}, \; \mathbf{\mathbf{\mu}}\ge\mathbf{0}. \qquad
\vspace{-4mm}\end{eqnarray}
with $\bar{\gamma}=(\gamma')^2 \le (\frac{\gamma}{1+\gamma})^2$, then $\delta\bigl( \mathrm{cone}(\bar{\mb A}), \mathrm{cone}(\mb L^\star + \mb S^\star) \bigr) \;\le\; \gamma$.
\end{theorem}

In contrast to existing matrix decompositions (e.g., \cite{Chandrasekharan2011-SJO,Candes2011-JACM}), which aim at statistical estimation, and measure quality of approximation in Frobenius norm, we guarantee approximation in Hausdorff distance $\delta(\cdot,\cdot)$. This is precisely the measure required for worst case verification. We call \eqref{eq:reform_prob} a {\bf \em cone-preserving low-rank and sparse decomposition}. It can be computed efficiently using the Linearized Alternating Direction Method of Multipliers (L-ADMM)(\cite{zhang2010bregmanized, ADMM_L_MA}), which converges globally with rate $O\left(1/k\right)$ \cite{he20121}.  We describe the numerical implementation and convergence theory associated with this approach in more detail in Appendix \ref{sec:algorithm}.

\section{Numerical Experiment} \label{sec:Numerical-Experiment}
We render images from 3D triangulated object models following a simplified imaging process $\mb y[f]\;=\;\mc P\mc D[f]$. Thus, our simulations include cast shadows, but not interreflection.\footnote{This approximation neglects the nontrivial interreflection terms, $\mc T + \mc T^2 + \dots$. These terms are at most on the order of $\rho_\star$, the maximum albedo. The approximation $\mb y[f] = \mc P \mc D[f]$ can be (loosely) considered to be the limiting case as $\rho_\star$ becomes small. Here, we make this approximation to make it easier to efficiently simulate the imaging process.}
Camera parameters $\gamma_c=f_c=d_c=1$ and $s_c=0.003$ are fixed throughout our experiments.

\paragraph{Verifying the Perturbation Bound.}
We compare the bound in Theorem \ref{thm:main-perturbation}, denoted as $\text{PerturbationBound}(\mb u,\mb u')$, to the actual difference $\norm{\bar{\mb y}[\mb u] - \bar{\mb y}[\mb u']}{2}$ for three different object shapes. The maximum ratio between those two quantities can be expressed as
\(
r \quad \doteq \quad \max_{\mb u,\mb u' \;\text{adjacent}}\set{\frac{\norm{\bar{\mb y}[\mb u] - \bar{\mb y}[\mb u']}{2}}{\text{PerturbationBound}(\mb u,\mb u')}}
\)
In our experiment, the set of point illuminations is generated using a uniform grid, $\theta = pi/360, 2 \pi / 360, \dots$
and $\phi = \pi / 360, 2 \pi / 360, \dots$ in spherical coordinates. Results are listed in Table \ref{tbl:perturbation_bound}: the ratio is always bounded by $1$, corroborating our theoretical results.

\begin{table}[H]
\begin{center} \begin{tabular}{|c|| c | c | c |}\hline
Object	&Vase	 	&Face 		&Bunny\\
&\includegraphics[width=1in]{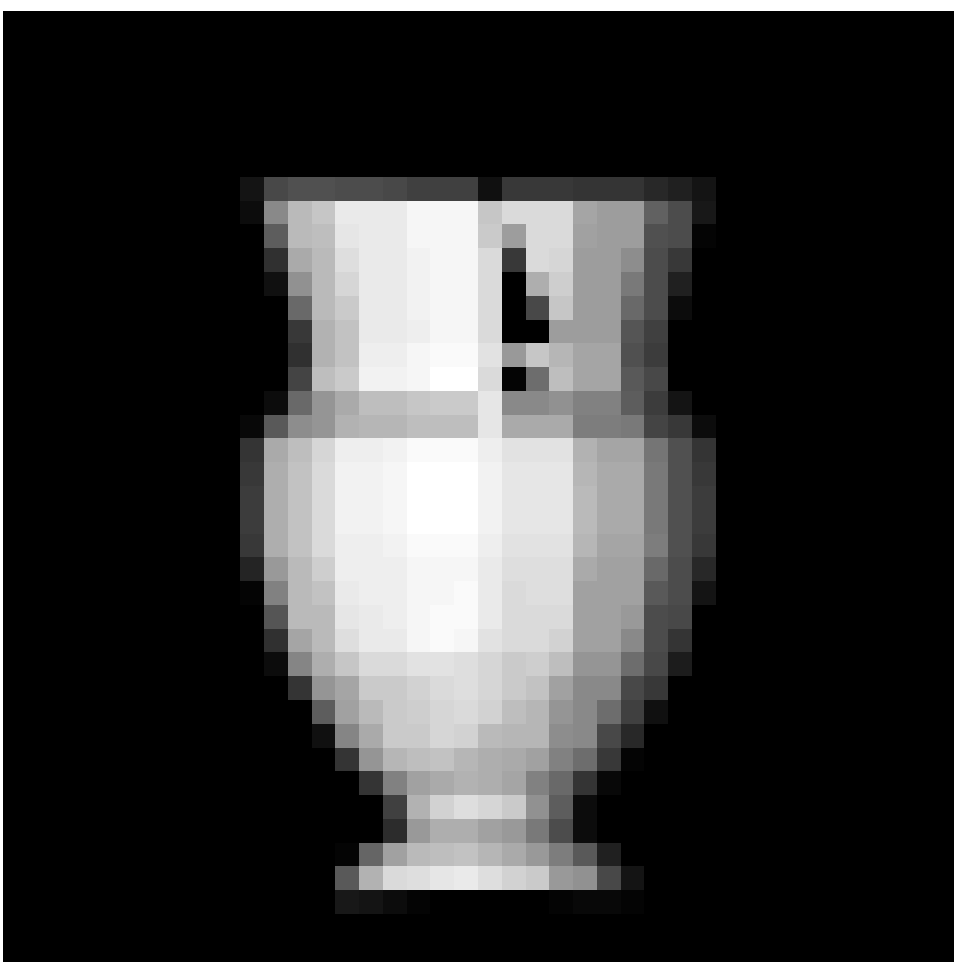}
&\includegraphics[width=1in]{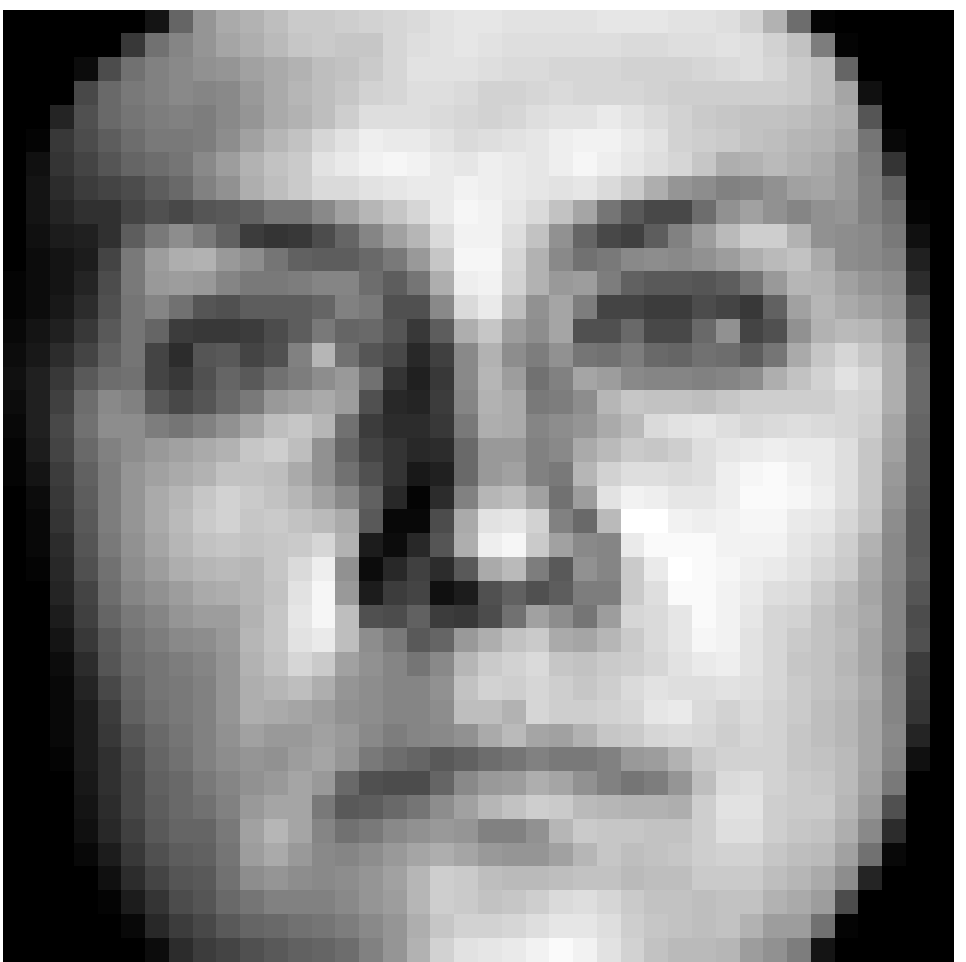}
&\includegraphics[width=1in]{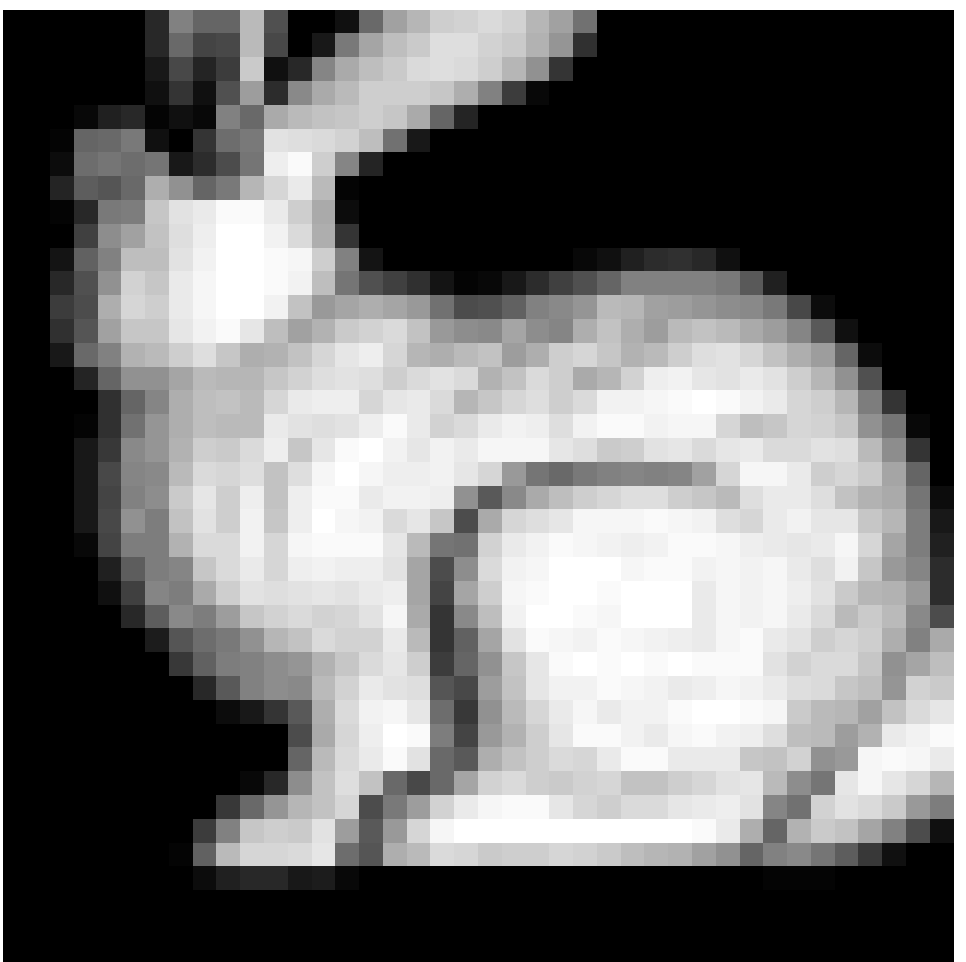}\\    \hline
$r$		&0.1809		&0.0302		&0.0290	\\    \hline
\end{tabular}\end{center}
\caption{{\bf Tightness of the bounds.} Largest ratio $r$  between experimental observation and theoretical upper bound for three different objects. The bound holds in all cases, and is tightest for the vase.}
\label{tbl:perturbation_bound}
\end{table}

\paragraph{Order of Perturbation Bound.}

We next consider the behavior of our bounds when $\norm{\mb u - \mb u'}{2} \to 0$. Our bounds predict that in the worst case, the change in irradiance $\bar{\mc D}[\mb u]$ should be proportional to $\norm{\mb u - \mb u'}{2}^{1/2}$. We investigate this using a toy object composed of two perpendicular surfaces $S_1$ and $S_2$ shown in Figure \ref{fig:toy_example} with $\mb u$ fixed, and  $\mb u'$ changing slowly. Figure \ref{fig:toy_example} (right) shows how $\norm{\bar{\mc D}[\mb u'] - \bar{\mc D}[\mb u]}{}$ depends on $\norm{\mb u-\mb u'}{2}$. Both the theoretical prediction and the computed value appear to be proportional to $\norm{\mb u - \mb u'}{2}^{1/2}$.\footnote{In Figure \ref{fig:toy_example}, we rescale the theoretical prediction for clearer comparison -- our goal is only to show that the exponent is $1/2$.} This suggests that in the worst case, our theory may be tight up to constant factors.
The theoretical prediction curve neglects constants, and simply draws $\diam{\obj} \times \chi_\star \times\sqrt{\norm{\mb u-\mb u'}{2}}$.

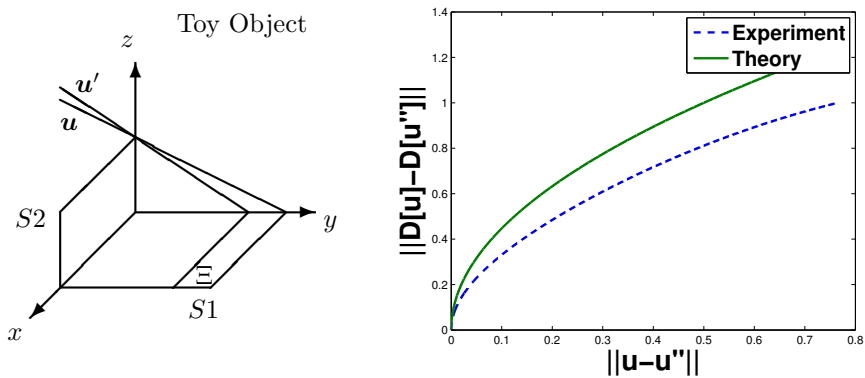
\begin{figure}[h]
\centerline{\hspace{10mm}
\begin{minipage}{2in}
\centerline{\qquad\qquad\quad Toy Object}\vspace{-7mm}
\setlength{\unitlength}{1cm}
\begin{picture}(6,5)
\thicklines
\put(2,2){\vector(0,1){2}} \put(1.8,4.2){$z$}
\put(2,2){\vector(1,0){2.4}} \put(4.5,1.8){$y$}
\put(2,2){\vector(-1,-1){1.4}} \put(0.3,0.3){$x$}
\put(4,2){\line(-1,-1){1}}
\put(1,1){\line(1,0){2}}\put(2.7,0.6){$S1$}
\put(2,3){\line(-1,-1){1}}
\put(1,1){\line(0,1){1}}\put(0.4,1.8){$S2$}
\put(4,2){\line(-2,1){3}}\put(1.0,3.1){$\mb u$}
\put(3.5,2){\line(-3,2){2.5}}\put(1.2,3.6){$\mb u'$}
\put(3.5,2){\line(-1,-1){1}}\put(2.8,1.05){$\Xi$}
\end{picture}\end{minipage}
\hspace{-5mm}
\begin{minipage}{3in}
\centerline{\includegraphics[width=2.5in]{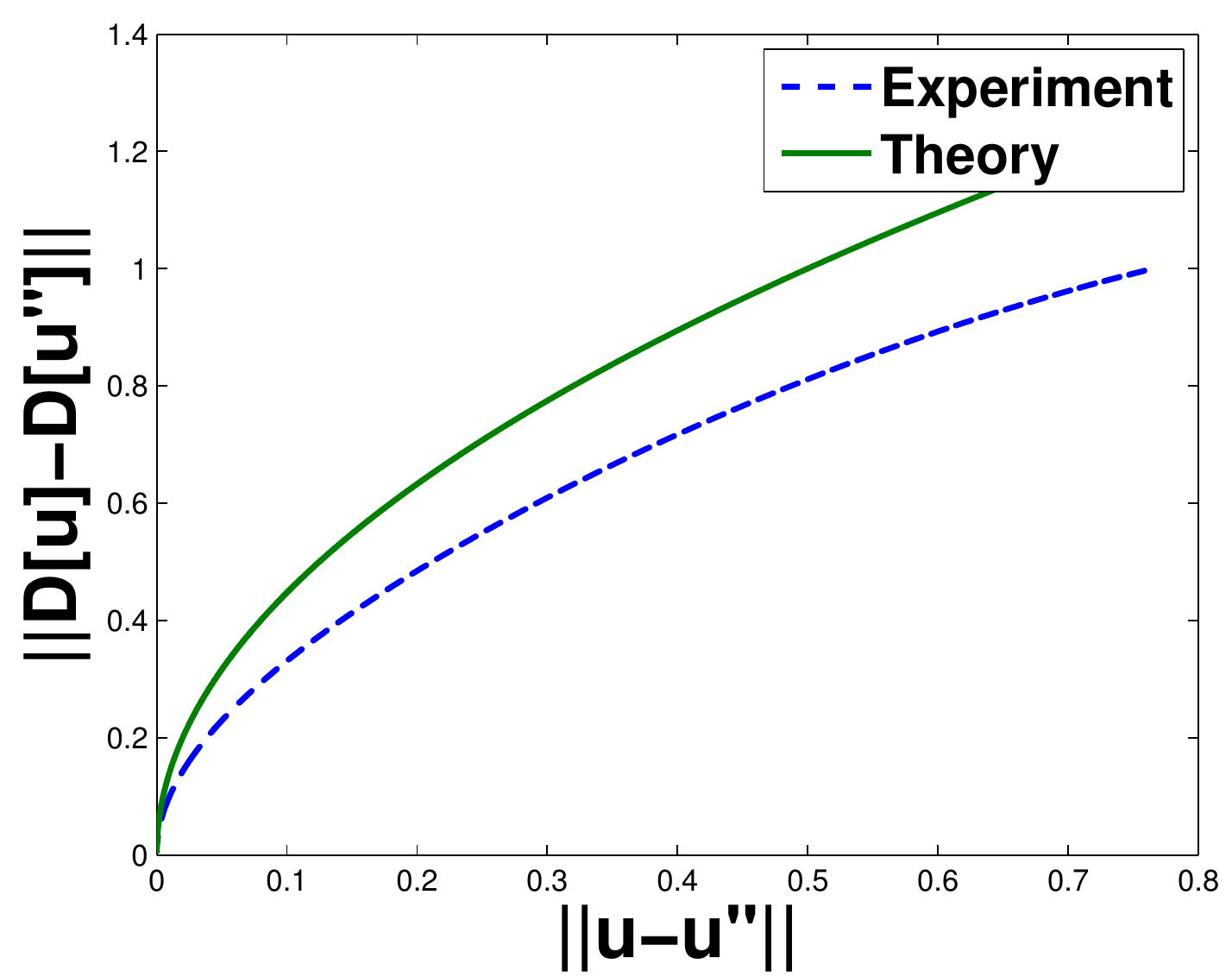} }
\end{minipage}
}
\caption{{\bf Order of Perturbation Bound}. Here, in both theory and simulation the change in $\bar{\mc D}[\mb u]$ is proportional to $\norm{\mb u - \mb u'}{}^{1/2}$. }
\label{fig:toy_example}
\end{figure}


\paragraph{Cone Preserving Complexity Reduction.}
We demonstrate the ability of our solution to \eqref{eq:reform_prob} to reduce the complexity of the representation, while preserving the conic hull.  We start with $n =648$ images of a face under point illuminations, with resolution $40 \times 40$. We solve the low-rank and sparse cone approximation problem in Theorem \ref{thm:cp-lrsd} for varying cone distances $\gamma$ with $\lambda$ simply chosen as $\sqrt{\max(m,n)}$. Figure \ref{fig:complexity_reduction} plots the ratio complexity of $\widehat{\mathbf{A}}$ and $\mb A$, or $\frac{(m+n)r+s}{mn}$, where $r$ is the rank of the recovered low-rank term and $s$ is the number of nonzero entries in the recovered sparse term. The decomposition reduces the complexity in all cases; the reduction becomes more pronounced as $\alpha$ increases. This is expected, since the cone $C_\alpha$ becomes smaller as $\alpha$ increases. This reduction in complexity suggests that although the number of extreme rays is large, there may be quite a bit of additional structure {\em across} the extreme rays, that could be exploited by more sophisticated cone constructions.

\begin{figure}[h]
\centerline{
\begin{minipage}{3in}
\centerline{\includegraphics[width=2.5in]{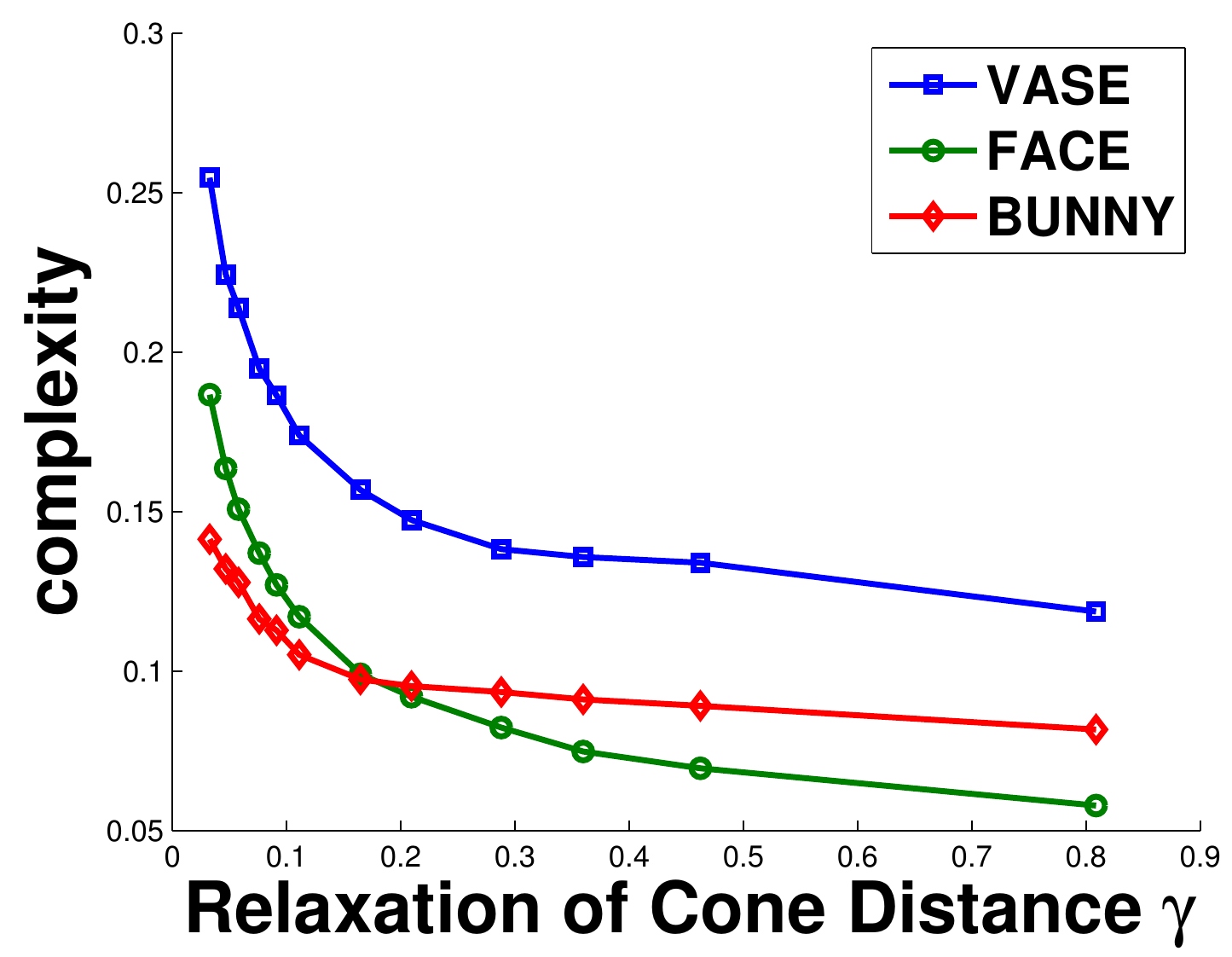}}
\end{minipage}
\hspace{-12mm}
\begin{minipage}{3in}
\centerline{\includegraphics[width=2.5in]{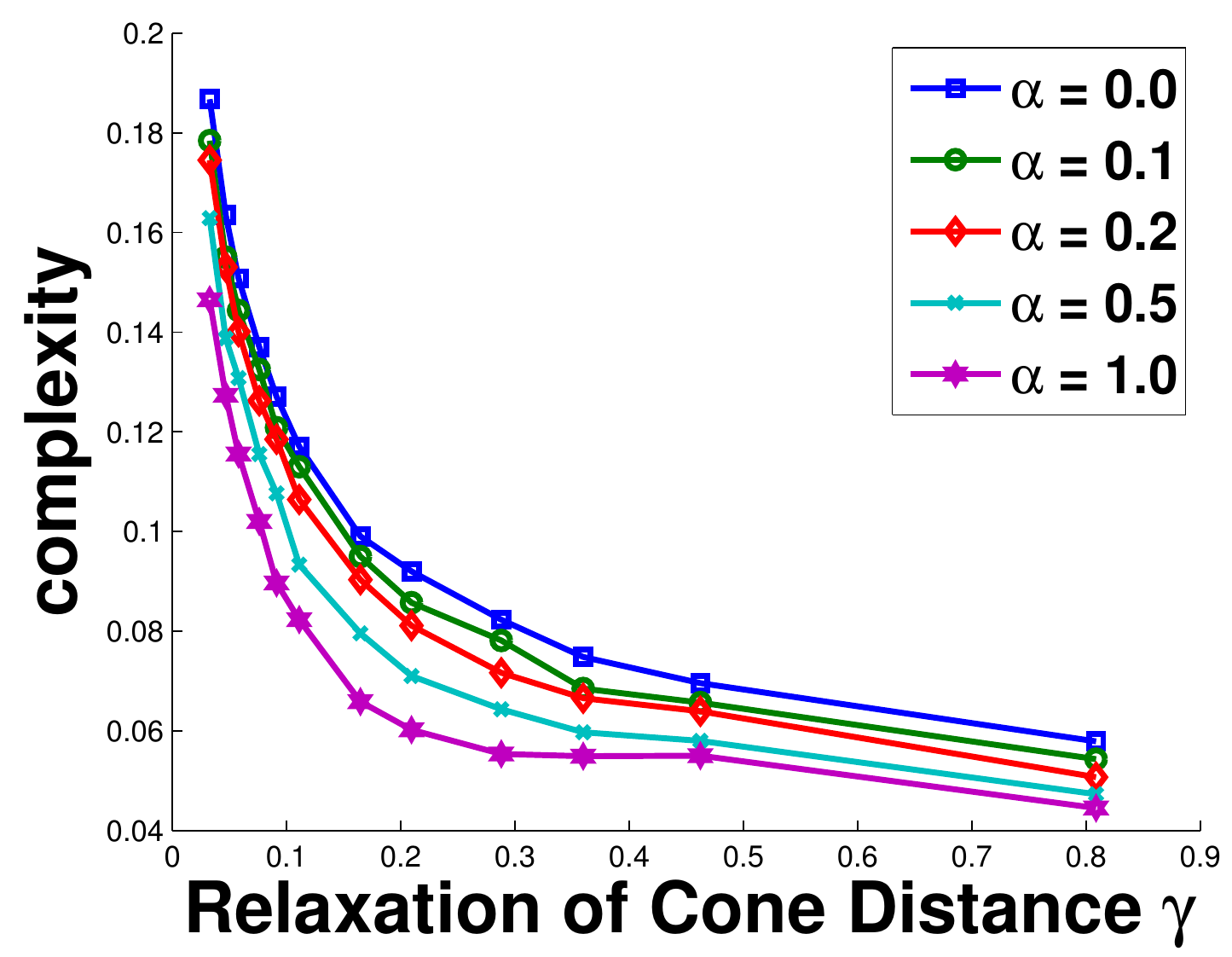} }
\end{minipage}
}
\caption{{\bf Cone Complexity Reduction} for different nonconvex objects under zero ambient level ($\alpha =0$) (left) and for {\em face} under different ambient illumination levels (right).}
\label{fig:complexity_reduction}
\end{figure}

\paragraph{Application Sketch.}
To conceptually justify the advantage of our cone approximation methodology in verification under poor illumination conditions, we compare the receiver operating characteristic (ROC) curves for 5 verification dictionaries obtained from same 3D face model under ambient level $\alpha=0.1$: convex cone $C_1$ composed of 2592 images, corresponds to the $\eps$-approximation of the original illumination cone; $C_2$ is the $\gamma$-approximation of $C_1$ with $L+S$ structure ($\gamma=0.11$); $C_3$ is rendered under 19 illumination directions corresponding to subsets $1$ and $2$ of Yale B \cite{Georghiades2001-PAMI} (roughly, the setting of \cite{Wright2009-PAMI}); $C_4$ is rendered under all 64 illumination directions considered in \cite{Georghiades2001-PAMI}. Finally, motivated by \cite{Basri2003-PAMI}, we also consider the subspace $S$ spanned by 9 principal components of $C_1$.

\begin{figure}[h]
\centerline{
\begin{minipage}{3in}
\centerline{\includegraphics[width=2.5in]{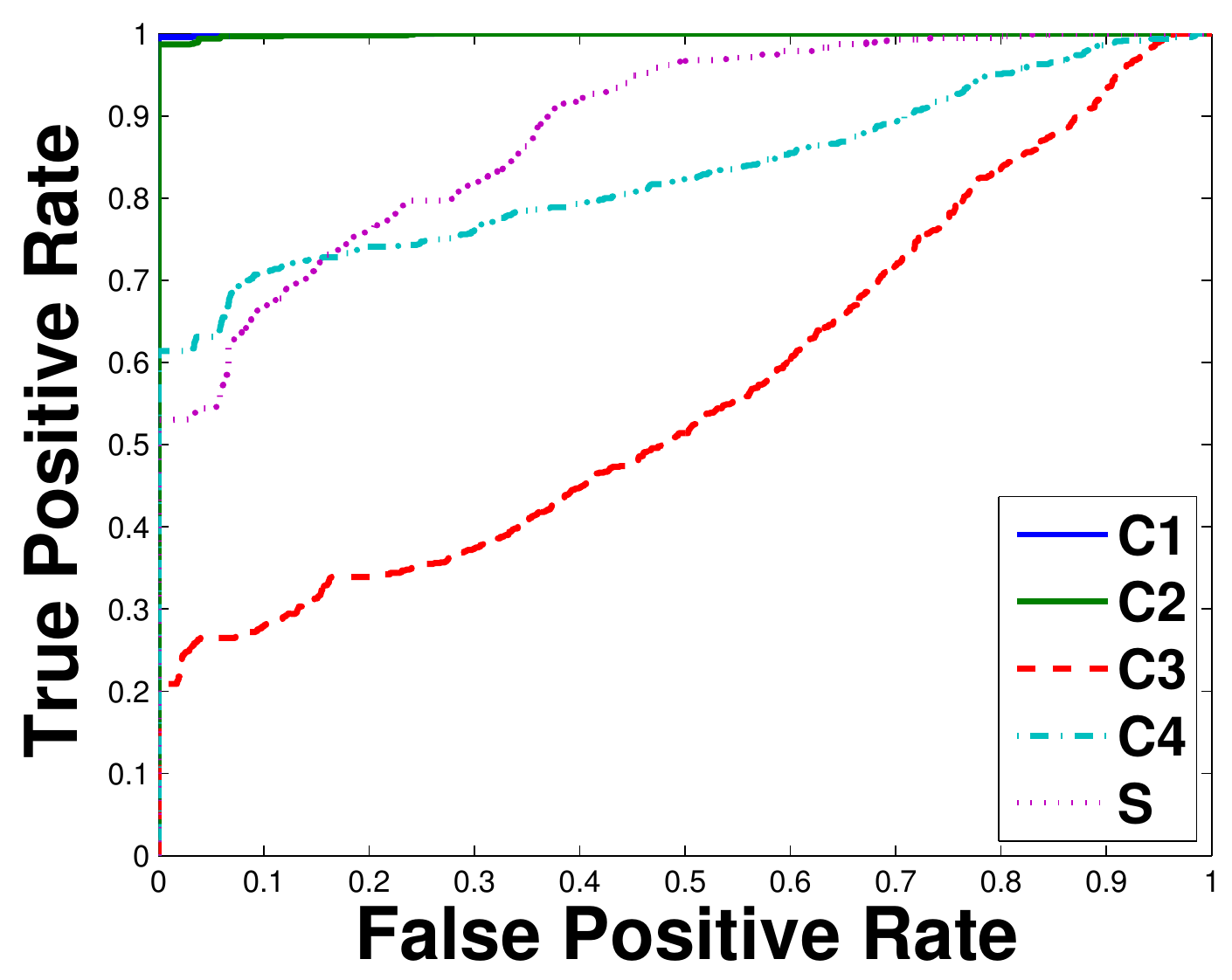}}
\end{minipage}
\hspace{-12mm}
\begin{minipage}{3in}
\centerline{\includegraphics[width=2.5in]{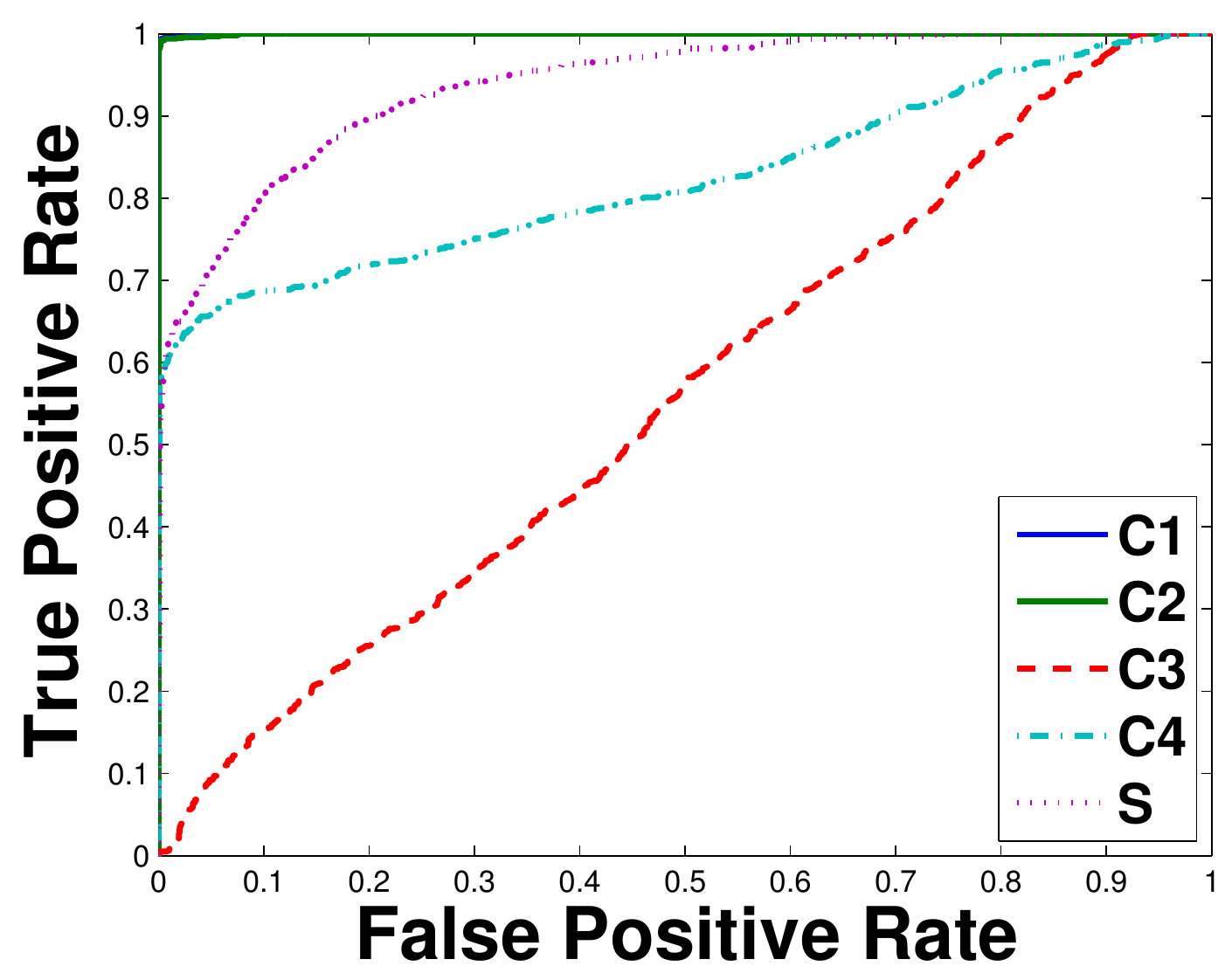} }
\end{minipage}
}
\caption{{\bf ROC Curves} for different dictionaries, with test images under uniform random illumination (left) and extreme illumination (right). The dictionaries are $C_1$: $\eps$ approximation, $C_2$: low-rank and sparse approximation, $C_3$-$C_4$: point illuminations distributed similar to \cite{Georghiades2001-PAMI}, $S$: nine-dimensional linear subspace.}
\label{fig:detection_roc}
\end{figure}

Our test data consist of 1,000 positive images under 1,000 illumination directions and 3,000 negative images of 3 other subjects. We consider two distributions for the illumination directions -- uniform on the sphere (roughly corresponding to the ``average case''), and uniform on the set of $\mb u \in \sphere^2$ for which $-0.1 \le u_3 \le 0.4$. Here, the $u_3$ axis is the camera axis. Arguably, the second set is more challenging. Figure \ref{fig:detection_roc} shows the ROC for a simple verification test based on the distance to the models. As suggested by our theoretical results, both $C_1$ and $C_2$ perform almost perfectly. The simpler models $C_3, C_4, S$ perform better than chance, but still break down frequently. We view this result as illustrating a tradeoff in illumination representation: uniformly good performance is possible, {\em if} we can afford a more complex representation. The cone preserving low-rank and sparse decomposition gives a way to control the complexity of the representation, while still maintaining this good performance.

\section{Discussion}
There are several directions for future work. Although our cone construction scheme guarantees worst case verification, the number of samples required is likely to be very large, in particular for small $\eps$: when $\eps = 0.01$, our theorem requires about $10^{25}$ images under ambient level $\alpha=1$. Our experiments on complexity reduction suggest that there should exist a simpler representation, if we can take advantage of the structure of shadows.

To use the results in a practical recognition system, we need to account for variations in object pose as well. This can be done using local optimization heuristics, or simply building models at a set of reference poses \cite{Georghiades2001-PAMI}. It will be important to have very concise models for each pose; the complexity reduction by convex programming is one means of achieving this.

Here, we have considered object instance {\em verification}, rather than object instance {\em recognition}. The ``yes/no'' question in verification forces us to confront basic questions about the set of images of the object. Nevertheless, we believe our methodology will be useful for recognition as well. For example, one could build models $\widehat{C}$ for each class and assigning the test sample to the closest model in angle. For recognition problems, the formulations and goals for sampling and complexity reduction may also change.

We anticipate three classes of practical application of our results. The first is in instance detection/recognition using 3D models and 2D test images. The second is in instance detection/recognition with active acquisition of training data, e.g., in face recognition for access control \cite{Wagner2012-PAMI}. The final, more speculative application is in instance detection/recognition with large families of objects with similar gross shape and appearance. In face recognition, learned models for physical variabilities (albedo and illumination) are often used in conjunction with deformable models \cite{AAM}. In many practical settings, this approach mitigates the difficulties associated with small training datasets -- they can work with as few as one image \cite{Wang2009-PAMI,Zhuang2013-CVPR}. Our results could give a way of learning a set of canonical illumination models, which capture effects such as cast shadows, and which could be adapted to each new input subject.

\section*{Acknowledgment} It is a great pleasure to acknowledge conversations with Donald Goldfarb (Columbia), Yi Ma (MSRA), Stefano Soatto (UCLA), Andrew Wagner (K.U.\ Leuven), Bin Yu (Berkeley), Zhengdong Zhang (MIT), Zihan Zhou (Penn.\ State). JW also gratefully acknowledges support from Columbia University and the Office of Naval Research.

\bibliographystyle{alpha}
\bibliography{egbib}

\appendix
\input{appendices}

\end{document}

%% file: jw_defs.tex
\usepackage[pdftex]{graphicx,color}
\usepackage{amsmath}
\usepackage{amssymb}
\usepackage{graphicx}
\usepackage{color}
\usepackage{url}
\usepackage{amsfonts, amscd, amsthm}
\usepackage{algorithm}
\usepackage{algorithmic}
\usepackage{subfigure}
\usepackage{bbm}

\newcommand{\bb}{\mathbb}
\newcommand{\mb}{\mathbf}
\newcommand{\mc}{\mathcal}
\newcommand{\mf}{\mathfrak}

\newtheorem{theorem}{Theorem}[section]

\newtheorem{lemma}[theorem]{Lemma}

\newtheorem{definition}[theorem]{Definition}

\renewcommand{\Re}{\mathbb R}

\newcommand{\conv}{\mathrm{conv}}
\newcommand{\eps}{\varepsilon}

\newcommand{\supp}[1]{\mathrm{supp}\left(#1\right)}
\newcommand{\<}{\langle}
\renewcommand{\>}{\rangle}
\newcommand{\aff}{\mathrm{aff}}
\newcommand{\relint}[1]{\mathrm{rel int}\left(#1\right)}

\renewcommand{\mathbf}{\boldsymbol}

\renewcommand{\dim}[1]{\mathrm{dim}(#1)}
\newcommand{\norm}[2]{\left\| #1 \right\|_{#2}}

\newcommand{\innerprod}[2]{\left\< #1, #2 \right\>}
\newcommand{\magnitude}[1]{\left|#1\right|}

\newcommand{\set}[1]{\left\{ #1 \right\}}

\newcommand{\objbdy}{\partial \mc O}
\newcommand{\obj}{\mc O}

\newcommand{\sphere}{\mathbb{S}}
\newcommand{\reals}{\mathbb{R}}

\newcommand{\relbdy}[1]{\mathrm{relbdy}(#1)}

\newcommand{\cone}[1]{\mathrm{cone}\left( #1 \right)}
\renewcommand{\magnitude}[1]{\left| #1 \right|}
\newcommand{\length}[1]{\mathrm{length}\left( #1 \right) }
\renewcommand{\set}[1]{\left\{ #1 \right\}}
\newcommand{\area}[1]{\mathrm{area}(#1)}

\newcommand{\cl}[1]{\mathrm{cl}\left( #1 \right)}

\newcommand{\volume}[1]{\mathrm{vol}\left( #1 \right) }
\newcommand{\ball}[2]{\mf B(#1,#2)}

\renewcommand{\(}{\begin{equation}}
\renewcommand{\)}{\end{equation}}

\newcommand{\diameter}[1]{\mathrm{diam}\left(#1\right)}

\newcommand{\indicator}[1]{\mathbbm{1}_{#1}}

\newcommand{\direct}[1]{\mc D\left[ #1 \right] }

\newcommand{\relativeinterior}[1]{\mathrm{relint}\left( #1 \right)}

\newcommand{\diam}[1]{\mathrm{diam}\left(#1\right)}

\newcommand{\bmb}[1]{\breve{\mb{#1}}}

\newcommand{\innerproduct}{\innerprod}
\newcommand{\shadowboundary}[1]{\partial S\left[ #1 \right] }
\newcommand{\barmb}[1]{{\bar{\mb#1}}}

\renewcommand{\aff}[1]{\mathrm{aff}\left( #1 \right)}

%% file: exact_angular_detector.pdftex_t
\begin{picture}(0,0)%
\includegraphics{exact_angular_detector.pdf}%
\end{picture}%
%
%
\setlength{\unitlength}{3947sp}%
\begingroup\makeatletter\ifx\SetFigFont\undefined%
\gdef\SetFigFont#1#2#3#4#5{%
  \reset@font\fontsize{#1}{#2pt}%
  \fontfamily{#3}\fontseries{#4}\fontshape{#5}%
  \selectfont}%
\fi\endgroup%
\begin{picture}(4646,3444)(994,-3718)
\put(3181,-661){\makebox(0,0)[lb]{\smash{{\SetFigFont{12}{14.4}{\rmdefault}{\mddefault}{\updefault}{\color[rgb]{0,0,0}$C$}%
}}}}
\put(4651,-1621){\makebox(0,0)[lb]{\smash{{\SetFigFont{12}{14.4}{\rmdefault}{\mddefault}{\updefault}{\color[rgb]{0,0,0}Decision}%
}}}}
\put(3376,-2686){\makebox(0,0)[lb]{\smash{{\SetFigFont{12}{14.4}{\rmdefault}{\mddefault}{\updefault}{\color[rgb]{1,1,1}${\color{red} \mathtt{REJECT}}$}%
}}}}
\put(4336,-1531){\makebox(0,0)[lb]{\smash{{\SetFigFont{12}{14.4}{\rmdefault}{\mddefault}{\updefault}{\color[rgb]{0,0,0}$\tau$}%
}}}}
\put(4771,-1816){\makebox(0,0)[lb]{\smash{{\SetFigFont{12}{14.4}{\rmdefault}{\mddefault}{\updefault}{\color[rgb]{0,0,0}boundary}%
}}}}
\put(2971,-2011){\makebox(0,0)[lb]{\smash{{\SetFigFont{12}{14.4}{\rmdefault}{\mddefault}{\updefault}{\color[rgb]{1,1,1}${\Huge \color{blue} \mathtt{ACCEPT}}$}%
}}}}
\end{picture}%

%% file: approximate_angular_detector.pdftex_t
\begin{picture}(0,0)%
\includegraphics{approximate_angular_detector.pdf}%
\end{picture}%
%
%
\setlength{\unitlength}{3947sp}%
\begingroup\makeatletter\ifx\SetFigFont\undefined%
\gdef\SetFigFont#1#2#3#4#5{%
  \reset@font\fontsize{#1}{#2pt}%
  \fontfamily{#3}\fontseries{#4}\fontshape{#5}%
  \selectfont}%
\fi\endgroup%
\begin{picture}(3849,3437)(1001,-3711)
\put(4603,-1597){\makebox(0,0)[lb]{\smash{{\SetFigFont{12}{14.4}{\rmdefault}{\mddefault}{\updefault}{\color[rgb]{1,1,1}$\color{black} (1+\eta) \tau$}%
}}}}
\put(3266,-2882){\makebox(0,0)[lb]{\smash{{\SetFigFont{12}{14.4}{\rmdefault}{\mddefault}{\updefault}{\color[rgb]{1,1,1}$\color{red} \mathtt{MUST} \; \mathtt{REJECT}$}%
}}}}
\put(3181,-661){\makebox(0,0)[lb]{\smash{{\SetFigFont{12}{14.4}{\rmdefault}{\mddefault}{\updefault}{\color[rgb]{0,0,0}$C$}%
}}}}
\put(4168,-1653){\makebox(0,0)[lb]{\smash{{\SetFigFont{12}{14.4}{\rmdefault}{\mddefault}{\updefault}{\color[rgb]{0,0,0}$\tau$}%
}}}}
\put(2640,-2059){\makebox(0,0)[lb]{\smash{{\SetFigFont{12}{14.4}{\rmdefault}{\mddefault}{\updefault}{\color[rgb]{1,1,1}$\color{blue} \mathtt{MUST} \; \mathtt{ACCEPT}$}%
}}}}
\end{picture}%

%% file: appendices.tex
\section*{Appendix}
%

\section{Proof of Lemma \ref{lem:cone-appx-gives-AAD-long}} \label{sec:s2}

\begin{proof}
For all $\mb y$ of norm one such that $\angle\left(\mathbf{y},\, C\right)\le\tau$,
$\exists$ $\bar{\mathbf{y}}\in{\rm cl}\left(C\right)$ with
$\|\mathbf{y}-\bar{\mathbf{y}}\|\le\sin\tau$. Moreover, $\norm{\barmb{y}}{2} \le 1$. By (\ref{eq:cone_distance}),
we know that $\exists$ $\hat{\mathbf{y}}\in{\rm cl}\left(\widehat{C}\right)$
s.t. $\|\bar{\mathbf{y}}-\hat{\mathbf{y}}\|\le\delta\left(C,\widehat{C}\right)$.
By triangle inequality, $\|\mathbf{y}-\hat{\mathbf{y}}\|\le\sin\tau+\delta\left(C,\widehat{C}\right)$,
which implies that $\angle\left(\mathbf{y},\,\widehat{C}\right)\le{\rm asin}\left(\sin\tau+\delta\left(C,\widehat{C}\right)\right)\le\xi$.
Therefore $\mathfrak{D}_{\xi}^{\widehat{C}}\left(\mathbf{y}\right)={\rm ACCEPT}$.
\vspace{.1in}

Conversely, if $\mf D_\xi^{\widehat{C}}(\mb y) = \mathrm{ACCEPT}$, then $\angle\left(\mb y,\widehat{C}\right) \le \xi$. If $\norm{\mb y}{2} \le 1$, this implies that there exists $\hat{\mb y} \in \cl{\widehat{C}}$ of norm at most one such that $\norm{\hat{\mb y} - \mb y }{2} \le \sin(\xi)$. Moreover, from the definition of $\delta$, there exists $\bar{\mb y} \in C \cap \ball{0}{1}$ such that $\norm{ \bar{\mb y} - \hat{\mb y} }{2} \le \delta(C,\widehat{C})$. By the triangle inequality,
\[
\norm{\mb y - \bar{\mb y}} \;\le\; \sin(\xi) + \delta( C, \widehat{C} ).
\]
Moreover, if
\( \label{eqn:asin-cond}
\mathrm{asin}\left( \sin(\xi) + \delta( C, \widehat{C} ) \right) \le (1+\eta) \tau,
\)
 we have $\angle(\mb y,C) \le (1+\eta) \tau$. Hence, whenever \eqref{eqn:asin-cond} holds, for every $\mb y$ such that $\mf D_\xi^{\widehat{C}}(\mb y) = \mathrm{ACCEPT}$, we have $\angle(\mb y,C) \le (1+\eta)\tau$, and hence $\mf D_\xi^{\widehat{C}}(\mb y) = \mathrm{REJECT}$ for all $\mb y$ with $\angle(\mb y,C) > (1+\eta)\tau$. This condition holds whenever
\(
\xi \;\le\; \mathrm{asin}( \sin( \tau + \eta \tau ) - \delta(C,\widehat{C}) ).
\)
Take together with the first paragraph, this condition and $\xi \ge \mathrm{asin}\left( \sin \tau + \delta(C,\widehat{C}) \right)$ imply that $\mf D_\xi^{\widehat{C}} \in \widehat{\bb D}_{\tau,\eta}^C$, which establishes our claim.
\end{proof}

\section{Proofs from Section \ref{sec:extreme-rays}}  \label{app:integrals-cones} 

The definition of the Riemann integral on $\reals^2$ gives the following:
\begin{lemma} \label{lem:nn-appx}
Let $\mb h : \sphere^2 \to \reals^m$ be a vector-valued function which is nonnegative and Riemann integrable. Then for every $\eps > 0$, there exists $N \in \bb Z_+$, $\mb u_1 \dots \mb u_N \in \sphere^2$, and $\lambda_1 \dots \lambda_N \ge 0$ such that
\(
\norm{ \int_{\mb u \in \sphere^2} \mb h[\mb u] \, d\mb u \; -\; \sum_{i=1}^N \lambda_i \mb h[\mb u_i] }{2} \;\le\; \eps.
\)
\end{lemma}
\begin{proof} {Again letting $W = [0,2\pi] \times [0,\pi]$, and let $\eta(\theta, \phi) = ( \cos \theta \sin \phi, \sin \theta \sin \phi, \cos \phi )$ denote the spherical coordinate map.} Consider a single coordinate $j$. From the definition of the Riemann integral, we have
\(
\int_{\mb u \in \sphere^2} \mb h_j[\mb u ] \, d \mb u \;=\; \int_W \mb h_j\left[ \eta(\theta, \phi)\right] \; \sin \phi \; d(\theta,\phi),
\)
where the right hand side is a Riemann integral on $\reals^2$. For every $\eta > 0$, there exists a partition $\Pi_j$ of $[0,2\pi] \times [0,\pi]$ such that
\[
\sum_{R \in \Pi_j} \left[ \sup_{(\theta,\phi) \in R} \tilde{\mb h}_j[\eta(\theta,\phi)] \, \sin \phi \right] \volume{ R } - \eta \;\le\; \int_{\mb u} \mb h_j[\mb u] \, d \mb u \;\le\;   \sum_{R \in \Pi_j} \left[ \inf_{(\theta,\phi) \in R} \tilde{ \mb h}_j[\eta(\theta,\phi)] \, \sin \phi \right]   \volume{ R }  + \eta.
\]
Choose such a partition $\Pi_j$ for each $j$, and let $\Pi = \set{ R_1, \dots, R_L }$ be the common refinement. Then for any choice of $(\theta_i,\phi_i) \in R_i$, we have
\(
\norm{ \sum_{i=1}^L {\mb h}_j[\eta(\theta_i,\phi_i)] \sin \phi_i \times \volume{R_i} - \int_{(\theta,\phi)} \mb h[\eta(\theta,\phi)] \,\sin\phi \; d (\theta,\phi) }{\infty} \;\le\; \eta.
\)
For all $i \in I$, set $\mb u_i = \eta(\theta_i,\phi_i)$, and $\lambda_i = \volume{R_i} \sin \phi_i$. Then
\(
\norm{ \sum_{i=1}^L \lambda_i \mb h[\mb u_i]  - \int_{\mb u} \mb h[\mb u] \, d \mb u }{\infty} \;\le\; \eta.
\)
Setting $\eta = \eps / \sqrt{m}$, we obtain the result.
\end{proof}

\paragraph{Proof of Lemma \ref{lem:C0-ext}.}  Below, we prove Lemma \ref{lem:C0-ext}.

\vspace{.1in}

\begin{proof} For notational convenience, let $\Psi = \cone{\set{ \bar{\mb y}[\mb u] \mid \mb u \in \sphere^2 } }$. Consider $\mb y_0 \in C_0$. Then $\mb y_0 \;=\; \mb y[f] \;=\; \int_{\mb u}\,  f(\mb u) \, \bar{\mb y}[\mb u] \, d \mb u$, for some nonnegative, Riemann integrable $f$. The vector valued function $\bar{\mb y}[\mb u] f(\mb u)$ is Riemann integrable. By Lemma \ref{lem:nn-appx}, for every $\eps > 0$, there exists $\hat{\mb y}_\eps \in \Psi$ with $\norm{ \mb y_0 - \hat{\mb y}_\eps }{2} \;\le\; \eps$. Hence,
\(
\mb y_0 \;\in\; \cl{ \Psi },
\)
and $C_0 \subseteq \cl{ \Psi }$.

We complete the proof by showing that $\Psi \subseteq \cl{C_0}$. By continuity of $\bar{\mb y}[\cdot]$, for any $\mb u_0 \in \sphere^2$, and any $\eps > 0$, there exists $\eta > 0$ such that
\(
\norm{\bar{\mb y}[\mb u_0] - \bar{\mb y}[\mb v]}{2} \;\le\; \eps \qquad \forall \mb v \in \ball{\mb u_0}{ \eta} \cap \sphere^2,
\)
where $\ball{\mb u}{r}$ is the $\ell^2$ ball of radius $r$ around $\mb u$. Let $f_\eps : \sphere^2 \to \reals$ via
\(
f_\eps( \mb v ) \;=\; \frac{1}{\area{\ball{\mb u_0}{\eta} \cap \sphere^2}} \indicator{\norm{\mb v - \mb u_0}{2} \;\le\; \eta }.
\)
Then $f_\eps$ is Riemann integrable, and
\begin{eqnarray}
\norm{\mb y[f_\eps] - \bar{\mb y}[\mb u_0]}{2}  &=&  \norm{ \frac{1}{\area{\ball{\mb u_0}{\eta} \cap \sphere^2 }} \int_{\| \mb u - \mb u_0\| \le \eta} \bar{\mb y}[\mb u] d\mb u \;\,-\;\, \bar{\mb y}[\mb u_0] }{2} \nonumber  \\
&\le& \frac{1}{\area{\ball{\mb u_0}{\eta}\cap \sphere^2 } } \int_{\| \mb u - \mb u_0\| \le \eta} \norm{ \bar{\mb y}[\mb u] \;-\; \bar{\mb y}[\mb u_0]}{2} d\mb u \nonumber \\
&\le& \frac{1}{\area{\ball{\mb u_0}{\eta}\cap \sphere^2 } } \int_{\| \mb u - \mb u_0\| \le \eta} \eps \;  d \mb u \nonumber \\
&\le& \eps.
\end{eqnarray}
Since this is true for every $\eps > 0$,  $\bar{\mb y}[\mb u] \in \cl{C_{0}}$, and so $\Psi \subseteq \cl{C_{0}}$, completing the proof.
\end{proof}

\paragraph{Proof of Lemma \ref{lem:ext-amb}.} Below, we prove Lemma \ref{lem:ext-amb}.

\vspace{.1in}

\begin{proof}
For $f = \alpha \omega + f_d \in \mc F_a$, write
\begin{eqnarray}
\mb y[f]  &=& \int_{\mb u} \bar{\mb y}[\mb u] \, f(\mb u) \, d\mb u \quad=\quad \int_{\mb u} \bar{\mb y}[\mb u] (\alpha \omega(\mb u) + f_d(\mb u)) \, d\mb u \\
 &=& \alpha \mb y_a + \int_{\mb u} \bar{\mb y}[\mb u] f_d(\mb u) \, d\mb u, \\
 &=& \alpha \, ( 1 - \norm{f_d}{L^1} ) \; \mb y_a + \int_{\mb u} \breve{\mb y}[\mb u] f_d(\mb u) d  \mb u. \label{eqn:y-f-sep}
\end{eqnarray}
Repeating arguments in the proof of Lemma 3.1, and using that $\breve{\mb y}[\mb u]$ is continuous in $\mb u$, we have
\(
\int_{\mb u} \breve{ \mb y}[\mb u] \, f_d(\mb u) \, d\mb u \;\in\; \cl{ \cone{ \set{ \breve{\mb y}[\mb u] \mid \mb u \in \sphere^2 } } }.
\)
Hence, from \eqref{eqn:y-f-sep}, $\norm{f_d}{L^1} \le 1$, and the fact that 
\(
\mb y_a \in \cl{ {\cone{\set{ \breve{\mb y}[\mb u] \mid \mb u \in \sphere^2 }} }}
\)
we have $\mb y[f] \in \cl{ {\cone{\set{ \breve{\mb y}[\mb u] \mid \mb u \in \sphere^2 }} }}$, and so $C_\alpha \subseteq \cl{\breve{C}}$.

Conversely, repeating arguments of Lemma \ref{lem:C0-ext}, we can show that the generators $\set{ \breve{\mb y}[\mb u] \mid \mb u \in \sphere^2 }$ and $\mb y_a$ are all elements of $\cl{C_\alpha}$, and hence $\breve{C} \subseteq \cl{C_\alpha}$, completing the proof.
\end{proof}

\paragraph{Proof of Lemma \ref{lem:ext-appx}.}

\begin{proof} Set
\(
\eps = \sup_{\mb u \in \sphere^2} \min_i \, \norm{ \frac{\breve{\mb y}[\mb u]}{\norm{\breve{\mb y}[\mb u]}{2}} - \frac{\breve{\mb y}[\mb u_i]}{\norm{\breve{\mb y}[\mb u_i]}{2}} }{2},
\)
and let $\mb w_\star$ realize the supremum in the definition of
\(
\label{eqn:eta-def}
\eta_\star = \sup_{\norm{\mb w}{2} \;\le\; 1} \inf_i \; \innerprod{ \mb w }{ \frac{\breve{\mb y}[\mb u_i]}{ \norm{ \breve{ \mb y}[ \mb u_i ] }{2} } }.
\)
\noindent Since $\bar{C} \subseteq \breve{C}$, we have
\(
\delta( \breve{C}, \bar{C} ) \;=\; \sup_{\mb y \in \breve{C} \setminus \set{\mb 0} } \frac{d(\mb y, \bar{C} )}{\norm{\mb y}{2}}.
\)
By Caratheodory's theorem, for every $\mb y \in \breve{C}$, there exist $\mb v_1 \dots \mb v_m \in \sphere^2$ and scalars $\lambda_1 \dots \lambda_m \ge 0$, $\zeta \ge 0$ such that
\(
\mb y \quad =\quad \sum_{j=1}^m \lambda_j \, \frac{\breve{\mb y}[\mb v_j]}{\norm{\breve{\mb y}[\mb v_j]}{2}} \;+\; \zeta \mb y_a.
\)
For each $j$, choose $\mb u_{i_j}$ such that $\norm{\frac{\breve{\mb y}[\mb v_j]}{\norm{\breve{\mb y}[\mb v_j]}{2}} - \frac{\breve{\mb y}[\mb u_{i_j}]}{\norm{\breve{\mb y}[\mb u_{i_j}]}{2}} }{2}$ is minimal. Then we have
\begin{eqnarray}
d(\mb y, \bar{C} ) &\le& \norm{\sum_j \lambda_j \frac{\breve{\mb y}[\mb v_j]}{\norm{\breve{\mb y}[\mb v_j]}{2}} + \zeta \mb{y}_a  - \sum_j \lambda_j \frac{\breve{\mb y}[\mb u_{i_j}]}{\norm{\breve{\mb y}[\mb u_{i_j}]}{2}} - \zeta \mb y_a }{2} \nonumber \\
 &\le& \sum_j \lambda_j \norm{ \frac{\breve{\mb y}[\mb v_j]}{\norm{\breve{\mb y}[\mb v_j]}{2}} - \frac{\breve{\mb y}[\mb u_{i_j}]}{\norm{\breve{\mb y}[\mb u_{i_j}]}{2}} }{2} \nonumber \\
 &\le& \sum_j \lambda_j \eps \nonumber \\
 &\le& \frac{\eps}{\eta_\star} \sum_{j=1}^m \lambda_j \innerprod{ \mb w_\star }{ {\frac{\breve{\mb y}[\mb v_{j}]}{\norm{\breve{\mb y}[\mb v_{j}]}{2}} }  } \nonumber \\
 &=&  \frac{\eps}{\eta_\star} \innerprod{ \mb w_\star }{ \sum_j \lambda_j {\frac{\breve{\mb y}[\mb v_{j}]}{\norm{\breve{\mb y}[\mb v_{j}]}{2}}}  } \nonumber \\
 &\le& \frac{\eps}{\eta_\star} \norm{ \sum_j \lambda_j {\frac{\breve{\mb y}[\mb v_{j}]}{\norm{\breve{\mb y}[\mb v_{j}]}{2}} }  }{2} \nonumber \\
&\le& \frac{\eps}{\eta_\star} \norm{ \mb y }{2}
\end{eqnarray}

Hence, $d(\mb y, \bar{C}) / \norm{\mb y}{2} \;\le\; \eps / \eta^\star$. We finish the proof by noting that
\begin{eqnarray}
 \norm{ \frac{\breve{\mb y}[\mb u]}{\norm{\breve{\mb y}[\mb u]}{2}} - \frac{\breve{\mb y}[\mb u']}{\norm{\breve{\mb y}[\mb u']}{2}} }{2} &\le& \norm{ \frac{\breve{\mb y}[\mb u]}{\norm{\breve{\mb y}[\mb u]}{2}} -  \frac{\breve{\mb y}[\mb u]}{\norm{\breve{\mb y}[\mb u']}{2}}  +  \frac{\breve{\mb y}[\mb u]}{\norm{\breve{\mb y}[\mb u']}{2}}  -  \frac{\breve{\mb y}[\mb u']}{\norm{\breve{\mb y}[\mb u']}{2}} }{2} \nonumber \\
&\le&  \norm{ \frac{\breve{\mb y}[\mb u]}{\norm{\breve{\mb y}[\mb u]}{2}} -  \frac{\breve{\mb y}[\mb u]}{\norm{\breve{\mb y}[\mb u']}{2}}  }{2} + \norm{ \frac{\breve{\mb y}[\mb u]}{\norm{\breve{\mb y}[\mb u']}{2}}  -  \frac{\breve{\mb y}[\mb u']}{\norm{\breve{\mb y}[\mb u']}{2}} }{2} \nonumber \\
&\le& \norm{\breve{\mb y}[\mb u]}{2}  \magnitude{\frac{\norm{\breve{\mb y}[\mb u']}{2} - \norm{\breve{\mb y}[\mb u]}{2} }{\norm{\breve{\mb y}[\mb u']}{2}\norm{\breve{\mb y}[\mb u]}{2}  } } + \frac{\norm{\breve{\mb y}[\mb u]  -  \breve{\mb y}[\mb u']}{2}}{\norm{\breve{\mb y}[\mb u']}{2}} \nonumber \\
&\le& 2 \frac{ \norm{ \breve{\mb y}[\mb u] - \breve{\mb y}[\mb u'] }{2} }{\norm{\breve{\mb y}[\mb u']}{2}} \nonumber \\
&\le& \frac{ 2 \norm{ \breve{\mb y}[\mb u] - \breve{\mb y}[\mb u'] }{2} }{ \alpha \norm{\mb y_a}{2} } \nonumber \\
&=& \frac{ 2 \norm{ \bar{\mb y}[\mb u] - \bar{\mb y}[\mb u'] }{2} }{ \alpha \norm{\mb y_a}{2} },
\end{eqnarray}
and hence
\(
\eps \;\le\; \sup_{\mb u} \min_i \frac{2 \norm{ \bar{\mb y}[\mb u] - \bar{\mb y}[\mb u_i] }{2} }{ \alpha \, \norm{\mb y_a}{2} },
\)
completing the proof. For the bound $\eta_\star \ge 1/\sqrt{m}$, note that if we choose $\mb w = m^{-1/2} \mb 1$ in the right hand side of \eqref{eqn:eta-def}, then since the $\breve{\mb y}[\mb u]$ are elementwise nonnegative, we for each $\mb u \in \sphere^2$
\(
\innerprod{\mb w}{\frac{\breve{\mb y}[\mb u]}{\norm{\breve{\mb y}[\mb u]}{2}} } \;=\; \frac{1}{\sqrt{m}} \frac{\norm{\breve{\mb y}[\mb u]}{1}}{\norm{\breve{\mb y}[\mb u]}{2}} \;\ge\; \frac{1}{\sqrt{m}}.
\)
This completes the proof.
\end{proof}

\section{Integrating on $\objbdy$} \label{app:int-obj-bdy}

 For each triangle 
\(
\Delta_i \;=\; \conv\set{ \mb z_i, \mb z'_i, \mb z_i'' }, 
\)
we can find an open triangle $U_i \subset \reals^2$, and an isometry $\varphi_i : U_i \to \mathrm{relint}(\Delta_i)$. To make this more concrete, we can let $\mb B_i \in \reals^{3 \times 2}$ be a matrix whose columns are an orthonormal basis for the $\mathrm{span}\set{ \mb z_i' - \mb z_i, \mb z_i'' - \mb z_i }$, write 
\(
\varphi_i( \mb w ) \;=\; \mb z_i + \mb B_i \mb w,
\)
and 
\(
U_i = \varphi_i^{-1} [ \mathrm{relint}(\Delta_i ) ].
\)
We construct an integral $\objbdy$ as follows. For each $i$, define a $\sigma$-algebra $\Sigma_i$ consisting of all sets of the form $\varphi_i[ S ]$, where $S \subseteq U_i$ is Lebesgue measurable. Let $\Sigma_{\objbdy}$ be the smallest $\sigma$-algebra containing each of the $\Sigma_i$. Define a measure $\mu_{\objbdy} : \Sigma_{\objbdy} \to \reals_+$ via 
\(
\mu_{\objbdy}(S) \;=\; \sum_{i} \; \mu\left( \; \varphi_i^{-1} [ S \cap \relint{\Delta_i} ]\; \right),
\)
where $\mu$ is the Lebesgue measure on $\reals^2$. It is easy to verify that $\mu_{\objbdy}$ is measure, making $(\Phi,\Sigma_{\objbdy},\mu_{\objbdy})$ a measure space, with Lebesgue integral
\(
\int g(\mb x) \; d\mu_{\objbdy}(\mb x) \;=\; \sum_{i} \int_{U_i} g \circ \varphi_i  \; d\mu.
\)
This gives an integral over $\Phi$. It extends to an integral over $\objbdy$ as a whole: for $g : \objbdy \to \reals$, we define its integral to be the integral of its restriction to $\Phi$.

\section{Proof of Lemma \ref{lem:imaging}} \label{app:imaging-lem-pf} {We prove Lemma \ref{lem:imaging}, which writes the imaging map $\mb y[f] = \mc P ( \mc I - \mc T )^{-1} \mc D[f]$ as an integral of the form $\int f(\mb u) \, \barmb{y}[\mb u] \, d\sigma(\mb u)$, where $\barmb{y}[\mb u ] = \mc P (\mc I - \mc T)^{-1} \bar{\mc D}[\mb u]$:}

\begin{proof}
We will show that for any Lebesgue integrable $f$, 
\( 
\label{eqn:y-exp-leb}
\mb y[f] \;=\; \int f(\mb u) \, \bar{\mb y}[\mb u] \; d \sigma(\mb u).
\)
By Theorem \ref{thm:main-perturbation}, $\bar{\mb y}[\mb u]$ is continuous in $\mb u$. Since the product of a nonnegative Riemann integrable $f$ and $\bar{\mb y}[\mb u]$ is Riemann integrable, this is expression is equal to the Riemann integral
\(
\mb y[f] \;=\; \int f(\mb u)\, \bar{\mb y}[\mb u] \; d \mb u,
\)
as desired. To show \eqref{eqn:y-exp-leb}, we use Tonelli's theorem and monotone convergence. {It is not difficult to show that the Riemann-integrable function $f$ is $\Sigma_{\sphere^2}$-measurable, $\bar{\mc D}[\mb u](\mb x)$ is $\Sigma_{\sphere^2} \times \Sigma_{\objbdy}$ measurable, and $\kappa$ is $\Sigma_{\objbdy} \times \Sigma_{\objbdy}$-measurable.} By repeated application of Tonelli's theorem, and using that the integrands are nonnegative, we obtain 
\begin{eqnarray}
\mc T^i \mc D[f] &=&  \int \kappa(\mb x, \mb x_i) \dots  \left( \int \kappa( \mb x_2, \mb x_1 ) \left( \int \bar{\mc D}[\mb u](\mb x_1) \, f (\mb u) \, d \sigma(\mb u) \right) \; d \mu_{\objbdy}(\mb x_1) \right) \; \dots \; d \mu_{\objbdy}(\mb x_i) \nonumber \\
 &=& \int f(\mb u) \left( \int \kappa(\mb x, \mb x_i) \dots \left( \int \kappa( \mb x_2, \mb x_1) \bar{\mc D}[\mb u](\mb x_1) d \mu_{\objbdy}(\mb x_1)\right) \dots d \mu_{\objbdy}(\mb x_i) \right) d \sigma(\mb u) \nonumber \\
 &=& \int f(\mb u) \; (\mc T^i \bar{\mc D}[\mb u])(\mb x) \; d \sigma(\mb u). 
\end{eqnarray}
By monotone convergence, 
\(
\sum_{i=0}^\infty \mc T^i \mc D[f](\mb x) \;=\; \int f(\mb u) \left( \sum_{i=0}^\infty \mc T^i \bar{\mc D}[\mb u] \right)(\mb x) \; d \sigma(\mb u).
\)
One more application of Tonelli's theorem gives \eqref{eqn:y-exp-leb}, completing the proof.
\end{proof}

\section{Proofs from Section \ref{subsec:shadow-boundary}} \label{app:shadow-proofs}

\paragraph{Proof of Lemma \ref{lem:rc}.} 
\begin{proof} Let $\mb x \in S[ \mb u ]^c$. By definition, $\bar{\mc D}(\mb x) = 0$ for all $\mb x \in E$, and so we may assume that $\mb x \in \Phi = \objbdy \setminus E$. Hence, $\mb x \in \relint{\Delta}$ for some face $\Delta$. Moreover, the definition of a triangulated object implies that there exists $\tau > 0$ such that $\ball{\mb x}{\tau} \cap \objbdy \subseteq \Delta$. 

Suppose, for purposes of contradiction, that there does not exist $r_0 > 0$ such that $\ball{\mb x}{ r_0} \cap \objbdy \subseteq S[\mb u]^c$. Then there exists a sequence $(\mb x_i)_{i=1}^\infty \subset S[\mb u]$, with $\mb x_i \to \mb x$. By dropping finitely many terms, we may assume $\mb x_i \in \relativeinterior{\Delta}$ for all $i$. Since $\mb x \in S[\mb u]^c$, $\innerproduct{\mb n(\mb x)}{\mb u} > 0$, and so $\innerproduct{\mb n(\mb x_i)}{\mb u} > 0 $ for all $i$. Hence, since $\mb x_i \in S[\mb u]$, for each $i$ there exists $t_i > 0$ such that $\mb x_i + t_i \mb u \in \objbdy$. For all $i$ large enough that $\mb x_i \in \ball{\mb x}{ \tau/2 }$, we necessarily have $t_i > \tau/2$. On the other hand, $t_i$ is bounded above by the diameter of the object. Because $t_i$ is bounded, it has a convergent subsequence $t_{i_1}, t_{i_2}, \dots \to \hat{t}$, with $\hat{t} > \tau/2$. Moreover, we have
\(
\lim_{j \to \infty} \mb x_{i_j} + t_{i_j} \mb u \;=\; \mb x + \hat{t} \mb u.
\)
Because $\obj$ is closed, this point is in $\obj$. Because $\hat{t} > \tau/2 > 0$, this implies that $\mb x \notin S[\mb u]^c$, a contradiction. Hence, for every $\mb x \in S[\mb u]^c$, there exists $r_0 > 0$ such that $\ball{\mb x}{r_0} \cap \objbdy \subseteq S[\mb u]^c$, and so $S[\mb u]^c$ is relatively open and $S[\mb u]$ relatively closed. 
\end{proof}

\paragraph{Proof of Lemma \ref{lem:cast-attach}.}


We use the notation\footnote{``$B$'' can be taken to stand for ``back''.}
\(
B[\mb u] \;\doteq\; \set{ \mb x \in \Phi \mid \innerproduct{ \mb n(\mb x) }{\mb u } \le 0 }.
\)
This set contains those points $\mb x$ that are necessarily shadowed, because $\mb n(\mb x)$ has nonpositive inner product with the light direction. Intuitively, if we ignore $B[\mb u]$, the remainder $S[\mb u] \setminus B[\mb u]$ should contain cast shadows, and for any point $\mb x$ in this set, the shadow retraction $\mb x^{\mb u}$ should exist. Furthermore, if $\mb x$ lies on the boundary of a cast shadow, its shadow retraction should lie in some {\em edge} of the object. We next state an intermediate lemma which makes this precise:
\begin{lemma}\label{lem:edge-cast} For all $\mb u$, and all $\mb x \in \Phi \cap \left( \shadowboundary{\mb u} \setminus B[{\mb u}] \right)$, $\mb x^{{\mb u}}$ exists, and $\mb x^{{\mb u}} \in E$.
\end{lemma}
\begin{proof}
Fix $\mb u$. Consider $\mb x \in \Phi \cap \left( \partial S[\mb u] \setminus B[\mb u] \right)$. 
Since $\mb x \in \Phi$, $\mb n(\mb x)$ is well-defined. Since $\mb x \notin B[\mb u]$, $\innerproduct{ \mb n(\mb x) }{ \mb u} > 0$. Since $\mb x \in S[\mb u]$, $\innerprod{\mb n(\mb x)}{\mb u} \nu(\mb x, \mb u) = 0$. This implies that $\mb x^{\mb u}$ exists, {$\mb x^{\mb u} \ne \mb x$}, and $\mb x^{\mb u}$ is an element of $\objbdy$.  If $\mb x^{\mb u} \in E$, we are done. For purpose of contradiction, suppose that $\mb x^{\mb u} \in \relint{\Delta}$ for some face $\Delta$. Thus, there exists $\eps > 0$ such that $\ball{\mb x^{\mb u}}{ \eps } \cap \aff{\Delta} \subseteq \Delta$. Since $\mb x \in \Phi$, $\mb x \in \relint{\Delta'}$ for some face $\Delta'$. Hence, there also exists $\eps' < 0$ such that $\ball{\mb x}{\eps'} \cap \objbdy \subset \Delta'$. 

Choose $\mb w_1, \mb w_2 \in \reals^3$ such that 
$$\aff{\Delta} \;=\; \set{ \mb x^{\mb u} + \alpha_1 \mb w_1 + \alpha_2 \mb w_2 \mid \alpha_1, \alpha_2 \in \reals}.$$
Similarly, {choose orthonormal vectors} $\mb v_1, \mb v_2 \in \reals^3$ such that 
$$\aff{\Delta'} \;=\; \set{\mb x + \beta_1 \mb v_1 + \beta_2 \mb v_2 \mid \beta_1, \beta_2 \in \reals}.$$
We claim that $\mb u \notin \mathrm{span}(\mb w_1, \mb w_2)$. Indeed, if not, then for small $\delta_t$, $\mb x^{\mb u} - \delta_t \mb u \in \Delta$. This would imply that $\mb x + (t^\star(\mb x, \mb u) - \delta_t) \mb u \in \objbdy$, contradicting the minimality of $t^\star(\mb x, \mb u)$. So, 
\[
\mathrm{rank}\left(\;\left[ \; \mb w_1 \,\mid\, \mb w_2 \,\mid\, \mb u \;\right]\;\right) \;=\; 3.
\]
Consider a generic point $\mb x' = \mb x + [\mb v_1 \mid \mb v_2 ] \left[ \begin{array}{c} \beta_1 \\ \beta_2 \end{array} \right] \in \aff{\Delta'}$. We find $t$ such that $\mb x' + t \mb u \in \aff{ \Delta }$. This is possible iff the equation 
\[
\mb x^{\mb u} + \mb W \mb \alpha \;=\; \mb x + t \mb u + \mb V \mb \beta
\]
{(with $\mb W = [\mb w_1 \mid \mb w_2]$, $\mb V = [\mb v_1 \mid \mb v_2]$)} has a solution {$(\mb \alpha, t)$}. Rearranging, we obtain
\[
\left[ \mb W \mid \mb u \right] \left[ \begin{array}{c} \mb \alpha \\ -t \end{array} \right] \;=\; \left[ \mb V \mid \mb u \right] \left[ \begin{array}{c} \mb \beta \\ - t^\star(\mb x, \mb u) \end{array} \right].
\]
Since the matrix on the left has full rank three, we have 
\[
\left[ \begin{array}{c} \mb \alpha \\ -t \end{array} \right] \;=\; \left[ \mb W \mid \mb u \right]^{-1} \left[ \mb V \mid \mb u \right] \left[ \begin{array}{c} \mb \beta \\ - t^\star(\mb x, \mb u) \end{array} \right]. 
\]
When $\mb \beta = \mb 0$, the solution is $\mb \alpha = \mb 0$, $t = t^\star(\mb x, \mb u ) > 0$. Hence, we can find $\eps'' \in (0,\eps')$ such that $\norm{\mb \beta}{2} \le \eps''$ implies that (i) $t > 0$, (ii) $\mb x + \mb V \mb \beta \in \Delta'$, (iii) $\Delta \owns \mb x^{\mb u} + \mb W \mb \alpha = \mb x + t \mb u$. {Because $\eps'' < \eps'$,  every $\mb x' \in \ball{\mb x}{\eps''} \cap \objbdy$ lies in $\Delta'$, and therefore has an expression $\mb x' = \mb x + \mb V \mb \beta$ with $\norm{\mb \beta}{} < \eps''$. Here, the fact that $\norm{\mb \beta}{} \le \eps''$ follows because $\mb V$ has orthonormal columns. By properties (i)-(iii) above, ${\mb x'}^{\mb u}$ exists, $\mb x' \in S[\mb u]$, and so $\ball{\mb x}{\eps''} \cap \objbdy \subseteq S[\mb u]$.} This implies that $\mb x \notin \partial S[\mb u]$. Hence, if $\mb x^{\mb u} \notin E$, $\mb x \notin \partial S[\mb u]$, and the proof is complete. 
\end{proof}

\begin{proof}[Proof of Lemma \ref{lem:cast-attach}] 


If $\mb x \in \partial S[\mb u] \cap \Phi$, then $\mb x \in \relativeinterior{\Delta}$ for some face $\Delta$. If $\mb x \in B[\mb u]$, then $\relativeinterior{\Delta} \subseteq S[\mb u]$, and $\mb x \notin \partial S[\mb u]$. Hence, $\mb x \in B[\mb u]^c$. By Lemma \ref{lem:edge-cast}, $\mb x^{\mb u}$ exists, and is an element of $E$. 
\end{proof}


\section{Proof of Theorem \ref{thm:D}} \label{sec:direct}

\begin{proof} 

Our goal is to bound
\(
\norm{\directb{\mb u} -\directb{\mb u'}}{L^2}^2 = \int \magnitude{ \directb{\mb u}(\mb x) -\directb{\mb u'}(\mb x) }^2 d \mu_{\objbdy}(\mb x).
\)
\noindent {\bf Initial manipulations.} Notice that for $\mb x \in S[\mb u] \cap S[\mb u']$, $\directb{\mb u}(\mb x) = \directb{\mb u'}(\mb x) = 0$. For $\mb x \in S[\mb u]^c \cap S[\mb u']^c$,
\(
\magnitude{ \directb{\mb u}(\mb x) - \directb{\mb u'}(\mb x) } \;=\; \rho(\mb x) \magnitude{ \innerproduct{\mb n(\mb x)}{\mb u} - \innerproduct{\mb n(\mb x)}{\mb u'} } \;\le\; \rho(\mb x) \norm{\mb u - \mb u'}{2}. \label{eqn:lipschitz-bound}
\)
This bound also holds for $\mb x \in B[\mb u] \cup B[\mb u']$. Thus, setting
\(
\label{eqn:gamma-def}
\Gamma \;=\; B[\mb u] \bigcup B[\mb u'] \bigcup \left( S[\mb u] \cap S[\mb u'] \right) \bigcup \left( S[\mb u']^c \cap S[\mb u]^c \right),
\)
we obtain
\(
\label{eqn:gamma-part}
\norm{ \mc P_{\Gamma}\left( \directb{\mb u} - \directb{\mb u'} \right) }{L^2}^2 \;\le\; \rho_\star^2 \;\area{\Gamma} \, \norm{\mb u - \mb u'}{2}^2.
\)
\item Note that
\(
\label{eqn:partition}
\objbdy \;=\; \Gamma \biguplus \left\{ S[\mb u] \setminus \left( S[\mb u'] \cup B[\mb u] \right) \right\} \biguplus \left\{ S[\mb u'] \setminus \left( S[\mb u] \cup B[\mb u'] \right) \right\},
\)
where $\biguplus$ denotes disjoint union. Consider $S[\mb u] \setminus ( S[\mb u'] \cup B[\mb u] )$. Introduce a notation
\(
\bar{\mb u}(r) = \frac{(1-r) \mb u' + r \mb u}{\norm{(1-r) \mb u' + r \mb u}{2}}, \qquad r \in [0,1], 
\)
and set
\(
r^\star(\mb x) =  \inf \set{ r \in [0,1] \mid \mb x \in S[\barmb{u}(r)] }.
\)
For all $\mb x \in S[\mb u ] \setminus (S[\mb u'] \cup B[\mb u])$, since $\mb u = \barmb{u}(1)$ and $\mb x \in S[\mb u]$, $r^\star(\mb x) \le 1$ is finite. We have
\begin{eqnarray}
\magnitude{ \directb{\mb u}(\mb x) - \directb{\mb u'}(\mb x) } &=& \directb{\mb u'}(\mb x) \quad=\quad \rho(\mb x)\innerproduct{\mb n(\mb x)}{\mb u'} \nonumber \\
&=& \rho(\mb x)  \innerproduct{\mb n(\mb x)}{\mb u' - \barmb{u}(r^\star(\mb x))} + \rho(\mb x) \innerproduct{\mb n(\mb x)}{\barmb{u}(r^\star(\mb x))} \nonumber \\
&\le& \rho(\mb x) \norm{ \mb u' - \barmb{u}(r^\star(\mb x)) }{2} + \rho(\mb x) \innerproduct{\mb n(\mb x)}{\barmb{u}(r^\star(\mb x))}.
\end{eqnarray}
So, we have
\begin{eqnarray}
\lefteqn{\norm{ \mc P_{S[\mb u] \setminus (S[\mb u']\cup B[\mb u])} \left( \directb{\mb u} - \directb{\mb u'} \right) }{L^2}^2 } \nonumber \\ &\le& 2 \, \rho_\star^2 \cdot \area{S[\mb u] \setminus (S[\mb u']\cup B[\mb u] )} \norm{\mb u - \mb u'}{2}^2 \nonumber \\
&& + \; 2 \, \rho_\star^2 \cdot \int_{\mb x \in S[\mb u] \setminus (S[\mb u']\cup B[\mb u] )} \innerproduct{ \mb n(\mb x) }{ \barmb{u}(r^\star(\mb x))}^2 d \mu_{\objbdy}(\mb x) \label{eqn:SuSupBuBup}
\end{eqnarray}
and our task is to bound the final integral.\footnote{Under our assumptions, the function $\zeta(\mb x) = \innerproduct{ \mb n(\mb x) }{ \barmb{u}(r^\star(\mb x))}^2$ can be shown to be piecewise rational, with pieces defined on semialgebraic sets. This implies that $\zeta(\mb x)$ is measurable, and the integral in \eqref{eqn:SuSupBuBup} is well-defined.} We will show the following:
\begin{quote}
$(\triangle)$ For all $\mb x \in S[\mb u ] \setminus (S[\mb u'] \cup B[\mb u])$ such that $\mb x \in \Phi$, we have $\mb x \in \shadowboundary{\barmb{u}(r^\star(\mb x))}$.
\end{quote}
{\em The intuition behind this claim is straightforward: as the light direction moves from $\mb u'$ to $\mb u$, the first time that $\mb x$ falls into shadow, it must lie in the boundary of the shadow (imagine the boundary of the shadow sweeping across the face $\Delta_j$). Obtaining this as a rigorous consequence of our assumptions requires some manipulation, which we perform below. }

\paragraph{Proving $(\triangle)$.} Since $\mb x \in S[\mb u]$, and $\mb u = \barmb{u}(1)$, $r^\star(\mb x) \le 1$ is finite. Notice that $\barmb{u}(r)$ is continuous in $r$. Take $r_i \to r^\star(\mb x)$, with $\mb x \in S[\barmb{u}(r_i)]$. Then $\barmb{u}(r_i) \to \barmb{u}(r^\star(\mb x))$. If $\mb x \in S[\mb u] \setminus \left( S[\mb u'] \cup B[\mb u] \right)$ then $\< \mb n(\mb x), \mb u \> > 0$, $\< \mb n(\mb x), \mb u' \> > 0$, and so for any $r \in [0,1]$, $\< \mb n(\mb x), \bar{\mb u}(r) \> > 0$. Hence, for any $r$ such that $\mb x \in S[\barmb{u}(r)]$, it must be that $\nu(\mb x,\mb u) = 0$, and $\mb x^{\barmb{u}(r)}$ exists. So, for each of our sequence of $r_i$, $\mb x^{\barmb{u}(r_i)} \in \objbdy$ exists: 
\[
\mb x + t^\star(\mb x, \barmb{u}(r_i))\, \barmb{u}(r_i) \in \objbdy,
\]
where we recall that $t^\star(\mb x, \barmb{u}(r_i)) \,=\, \inf \set{ t > 0 \mid \mb x + t \, \barmb{u}(r_i) \in \objbdy }$. Let $\Delta$ be the face containing $\mb x$. Since $\mb x \in \Phi$, $\mb x \in \relint{\Delta}$, and there exists $\eps > 0$ such that $\ball{\mb x}{\eps} \cap \objbdy \subseteq{\Delta}$. Since $\< \mb n(\mb x), \barmb{u}(r_i)\> > 0$, if $t > 0$ is such that $\mb x + t \barmb{u}(r_i) \in \objbdy$, then $t > \eps$. Hence, for every $i$, $t^\star(\mb x, \barmb{u}(r_i) ) \ge \eps > 0$. Because $\obj$ is bounded, the $\beta_i \doteq t^\star( \mb x, \barmb{u}(r_i))$ are bounded. Hence, the sequence $(\beta_i)$ has a convergent subsequence $\beta_{i_j}$: $\lim_{j \to \infty} \beta_{i_j} = \beta_\star$ for some $\beta_\star$. From the previous discussion $\beta_\star \ge \eps$. Moreover, $\barmb{u}(r_{i_j}) \to \barmb{u}(r^\star(\mb x))$. Hence
\[
\mb x + \beta_{i_j} \barmb{u}(r_{i_j}) \to \mb x + \beta_\star \barmb{u}(r^\star(\mb x)).
\]
Because each element of the left hand side is in $\objbdy$, and $\objbdy$ is closed, the limit is in $\objbdy$. Because $\beta_\star > 0$, the right hand side is not equal to $\mb x$. We conclude that $\mb x \in S[\barmb{u}(r^\star(\mb x))]$. 

It is left to show that $\mb x$ lives in the relative boundary $\partial S[\barmb{u}(r^\star(\mb x))]$ of this set. Choose $r' \in (0,r^\star)$, and note that  $\innerprod{\barmb{u}(r')}{\mb n(\mb x)} > 0$. Notice that 
\[
\aff{\Delta} = \set{ \mb x' \mid \innerprod{\mb n(\mb x)}{\mb x'} = \innerprod{\mb n(\mb x)}{\mb x} }. 
\]
Hence, for any $\tau > 0$, if we set 
\[
s = \tau \frac{\innerprod{\mb n(\mb x)}{\barmb{u}(r^\star(\mb x))}}{\innerprod{\mb n(\mb x)}{\barmb{u}(r')}},
\]
then 
\[
\mb x' \doteq \mb x - \tau \barmb{u}(r^\star(\mb x)) + s \barmb{u}(r') \in \aff{ \Delta }.
\]
Suppose, for purpose of contradiction, that $\mb x \in \relint{S[\barmb{u}(r^\star(\mb x))]}$. Then, there exists $\tau_0 > 0$ such that for $\tau \in (0,\tau_0)$, we have 
\[
\mb x' \in S[\barmb{u}(r^\star(\mb x))].
\]
Moreover, if $\mb x' \in \ball{\mb x}{\eps/2}$, and $t > 0$ is such that $\mb x' + t \, \barmb{u}(r^\star(\mb x)) \in \objbdy$, then $t \ge \eps / 2$. Choose $\tau >0$ small enough that $\tau < \tau_0$, $\norm{\mb x ' - \mb x }{2} < \eps / 2$, and $\tau < \eps / 2$. 

With these choices, $\mb x' \in S[\barmb{u}(r^\star(\mb x))]$, and there exists $t \ge \eps / 2 > \tau$ such that 
\[
\mb x' + t \barmb{u}(r^\star(\mb x)) \in \objbdy.
\]
Write 
\[
\mb x' + t \barmb{u}(r^\star(\mb x)) = \mb x + s \barmb{u}(r') + (t - \tau) \barmb{u}(r^\star(\mb x)) = \mb x + s' \mb v,
\]
with 
\[
\mb v = \frac{ s \barmb{u}(r') + (t - \tau) \barmb{u}(r^\star) }{ \norm{ s \barmb{u}(r') + (t - \tau) \barmb{u}(r^\star)  }{2} }
\]
and
\[
s' = \norm{ s \barmb{u}(r') + (t - \tau) \barmb{u}(r^\star)  }{2}.
\]
Since $\mb v = \barmb{u}(r'')$ for some $r'' > 0$ which is strictly smaller than $r^\star(\mb x)$, and $\mb x \in S[\barmb{u}(r'')]$, this contradicts the definition of $r^\star$ as the infimum. Hence, $\mb x \in \shadowboundary{\barmb{u}(r^\star(\mb x))}$.

\paragraph{Bounding the integral.} {\em The main utility of $(\triangle)$ is that it allows us to organize our calculations in terms of the points that cast the shadows, rather than the points that are shadowed. We next reduce the problem of obtaining the desired bound to that of showing one key inequality, \eqref{eqn:ki-1}. We show how this inequality implies the desired result, and then return to show that this inequality indeed holds.}

To lighten the notation, we write
\(
H = S[\mb u] \setminus ( S[\mb u'] \cup B[\mb u] ), \qquad H_\Phi = H \cap \Phi.
\)
Since every $\mb x \in H_\Phi$ satisfies $\mb x \in \shadowboundary{\barmb u(r^\star(\mb x))}$, by the Lemma \ref{lem:edge-cast}, $\mb x^{\bar{\mb u}(r^\star(\mb x))}$ exists, and is an element of some edge $e^\star(\mb x)$. For each $e \in \mc E$, let
\(
\Xi_{e,\Delta} \;\doteq\; \set{ \mb x \in H_\Phi \cap \Delta \mid e^\star(\mb x) = e }.
\)
By the above discussion,
$$H_\Phi \;=\; \bigcup_{e,\Delta} \Xi_{e,\Delta}.$$
Below, we will demonstrate the following key inequality:\footnote{Below, we will show that the $\Xi_{e,\Delta}$ are semialgebraic sets, and hence measurable. Thus, the integrals in \eqref{eqn:ki-1} are well-defined.}
\(
\label{eqn:ki-1}
\sum_{e,\Delta} \int_{\mb x \in \Xi_{e,\Delta}} \innerproduct{\mb n(\mb x)}{\barmb{u}(r^\star(\mb x))}^2 d \mu_{\objbdy}(\mb x) \;\le\; 8 \sqrt{2} \cdot \diameter{\obj} \norm{\mb u - \mb u'}{2} \chi_\star.
\)
The proof of this inequality will consist of several steps, which are carried out below. We first show that this inequality implies the desired result. Notice that by definition,
\begin{eqnarray}
\int_{\mb x \in H}  \innerproduct{\mb n(\mb x)}{\barmb{u}(r^\star(\mb x))}^2 d \mu_{\objbdy}( \mb x ) &=& \int_{\mb x \in H_\Phi}  \innerproduct{\mb n(\mb x)}{\barmb{u}(r^\star(\mb x))}^2 d\mu_{\objbdy}( \mb x ) \nonumber \\
&\le& 8 \sqrt{2} \cdot \diameter{\obj} \norm{\mb u - \mb u'}{2} \chi_\star. 
\end{eqnarray}
Combining with \eqref{eqn:SuSupBuBup}, we obtain
\begin{eqnarray}
\lefteqn{ \norm{ \mc P_{S[\mb u] \setminus (S[\mb u']\cup B[\mb u] )} \left( \direct{\mb u} - \direct{\mb u'} \right) }{L^2}^2 } \nonumber \\
&\le& 2 \, \rho_\star^2  \, \area{S[\mb u] \setminus (S[\mb u']\cup B[\mb u] )} \norm{\mb u - \mb u'}{2}^2 \;+ \; 16 \sqrt{2} \, \rho_\star^2\, \diameter{\mc O}  \chi_\star \norm{\mb u -\mb u'}{2}. \label{eqn:SuSupBuBup-final} \quad
\end{eqnarray}
By symmetry, we also obtain
\begin{eqnarray}
\lefteqn{ \norm{ \mc P_{S[\mb u'] \setminus (S[\mb u] \cup B[\mb u'])} \left( \direct{\mb u} - \direct{\mb u'} \right) }{L^2}^2 } \nonumber \\
&\le& 2 \, \rho_\star^2 \, \area{S[\mb u'] \setminus (S[\mb u] \cup B[\mb u'])} \norm{\mb u - \mb u'}{2}^2 \;+\; 16 \sqrt{2} \, \rho_\star^2\, \diameter{\mc O}  \chi_\star \norm{\mb u -\mb u'}{2}. \label{eqn:SupSuBuBup-final} \quad
\end{eqnarray}
Combining these two bounds with \eqref{eqn:gamma-part} and \eqref{eqn:partition}, we obtain the claimed result. We are just left to verify \eqref{eqn:ki-1}.

\vspace{.25in}
\noindent{\bf Proof of Key Inequality \eqref{eqn:ki-1}.} {\em Rather than directly proving \eqref{eqn:ki-1}, which requires us to bound integrals over $\Xi_{e,\Delta}$, we first demonstrate a bound over a much simpler set (which turns out to be a quadrilateral $Q \subseteq \Delta$), and then show that we can arbitrarily well-approximate the domain of interest using finite collections of such quadrilaterals, to obtain the desired bound.}

For each $\mb z \in e$, and each $\Delta$, let
\(
\tau_{e,\Delta}(\mb z) = \set{ r \in [0,1] \mid \mb z_{\barmb{u}(r)} \;\text{\rm exists, and} \; \mb z_{\barmb{u}(r)}\in \Xi_{e,\Delta} }
\)
It is immediate that $\Xi_{e,\Delta} = \set{ \mb z_{\barmb{u}(r)} \mid \mb z \in e, r \in \tau_{e,\Delta}(\mb z)}$. Let $[\mb v, \mb w] = \conv\set{\mb v,\mb w}$. Simple geometric reasoning shows that if $[\mb v,\mb w] \subseteq e$, and {$[r_1,r_2] \subseteq \tau_{e,\Delta}(\mb z)$ for every $\mb z \in [\mb v,\mb w]$ (i.e., $[r_1,r_2] \subseteq \bigcap_{\mb z \in [\mb v,\mb w]} \tau_{e,\Delta}(\mb z)$),} the set
\(
\mc Q\left( [\mb v,\mb w], [r_1,r_2] \right) \;\doteq\;  \set{ \mb z_{\barmb{u}(r)} \mid r \in [r_1,r_2], \, \mb z \in [\mb v,\mb w] } \subset \Xi_{e,\Delta}
\)
is a quadrilateral. We will show the following:
\begin{quote}
$(\square)$ Let $[\mb v, \mb w] \subseteq e$. Suppose that $[r_1,r_2] \subseteq \bigcap_{\mb z\in [\mb v, \mb w]} \tau_{e,\Delta}(\mb z)$. Let
\[
Q \;=\; \mc Q([\mb v,\mb w],[r_1,r_2]).
\]
Then we have that
\(
\int_{\mb x\in Q} \innerproduct{\mb n(\mb x) }{ \barmb{u}(r^\star(\mb x)) }^2 d \mu_{\objbdy}(\mb x) \; \le\; 8 \sqrt{2} \cdot \diameter{\obj} \norm{\mb u - \mb u'}{2} \norm{\mb v - \mb w}{2} | r_2 - r_1 |.
\)
\end{quote}
We show the claim $(\square)$. If $\mb v - \mb w \in \mathrm{span}(\mb u, \mb u')$, then $Q$ has measure zero, the integral on the left hand side is zero, and the bound holds trivially. Suppose that $\mb v - \mb w \notin \mathrm{span}(\mb u, \mb u')$. Notice that if $\mb x^{\barmb{u}(r)} \in e$, then there exists a solution $(s_1,s_2,s_3)$ to the system of equations $\mb x + s_1 \mb u + s_2 \mb u' = \mb v s_3 + \mb w (1-s_3)$. When $\mb v - \mb w \notin \mathrm{span}(\mb u, \mb u')$ this system has at most one solution, and so for each $\mb x$ there is at most one $r$ such that
\[
\mb x^{\barmb{u}(r)} \in e.
\]
Now, for $\mb x \in Q \subseteq \Xi_{e,\Delta}$, we have $\mb x^{\barmb{u}(r^\star(\mb x))} \in e$. Moreover, by construction of $Q$, $\mb x = \mb z_{\barmb{u}(r)}$ for some $\mb z \in [\mb v, \mb w]$ and $r \in [r_1,r_2]$. Hence, it must be that $r = r^\star(\mb x)$, and so $r^\star(\mb x) \in [r_1,r_2]$. 

Set $\mb u_1 = \barmb{u}(r_1)$, $\mb u_2 = \barmb{u}(r_2)$. Notice that $\mb n(\mb x)$ is constant over $Q$. We abbreviate it by $\mb n$. Suppose that $\innerproduct{\mb n}{\barmb{u}(r)}$ is maximized over $[r_1,r_2]$ at $r = r_1$. Then
\(
\int_{\mb x \in Q} \innerproduct{ \mb n(\mb x) }{\barmb{u}( r^\star(\mb x))}^2 {d \mu_{\objbdy}( \mb x )} \;\le\; \left(\mb n^T \mb u_1\right)^2 \area{Q}.
\)
Let $\mb x_0$ be an arbitrary point in $\Delta$. Then 
\[
\aff{\Delta} \;=\; \set{ \mb x \mid \innerprod{ \mb n }{ \mb x }= \innerprod{\mb n}{ \mb x_0 } }.
\]
Using this expression, we can write the shadow projection $\mb z_{\mb u}$ as
\(
\mb z_{\mb u} = \left( \mb I - \frac{\mb u \mb n^T}{\mb n^T\mb u} \right) \mb z  + \frac{\mb n^T \mb x_0}{\mb n^T \mb u} \mb u
\)
The set $Q$ is a quadrilateral, with sides $[\mb w_{\mb u_1}, \mb v_{\mb u_1}]$, $[\mb v_{\mb u_1},\mb v_{\mb u_2}]$, $[\mb v_{\mb u_2}, \mb w_{\mb u_2}]$, $[\mb w_{\mb u_1}, \mb w_{\mb u_2}]$. We can calculate 
\begin{eqnarray*}
\mb v_{\mb u_1} - \mb v_{\mb u_2} &=& \left(\mb v - \mb u_1 \frac{\mb n^T \mb v}{\mb n^T \mb u_1} + \frac{\mb n^T \mb x_0}{\mb n^T \mb u_1} \mb u_1 \right) - \left( \mb v - \mb u_2 \frac{\mb n^T \mb v}{\mb n^T \mb u_2} + \frac{\mb n^T \mb x_0}{\mb n^T \mb u_2} \mb u_2 \right) \\
&=& \left( \frac{\mb u_1}{\mb n^T \mb u_1} - \frac{\mb u_2}{\mb n^T \mb u_2} \right) \left( \mb n^T ( \mb x_0 - \mb v ) \right),
\end{eqnarray*}
and similarly, $\mb w_{\mb u_1} - \mb w_{\mb u_2} =  \left( \frac{\mb u_1}{\mb n^T \mb u_1} - \frac{\mb u_2}{\mb n^T \mb u_2} \right) \left( \mb n^T ( \mb x_0 - \mb w ) \right)$. Since these differences are scalar multiples of the common vector $\frac{\mb u_1}{\mb n^T \mb u_1} - \frac{\mb u_2}{\mb n^T \mb u_2}$, the two sides $[\mb v_{\mb u_1}, \mb v_{\mb u_2}]$ and $[\mb w_{\mb u_1}, \mb w_{\mb u_2}]$ are parallel. 

Let $\ell^\perp$ denote the orthogonal length
\(
\ell^\perp\;\doteq\; \norm{\mb P_{(\mb v_{\mb u_1}-\mb v_{\mb u_2})^\perp} \left( \mb w_{\mb u_1} - \mb v_{\mb u_1} \right) }{2}.
\)
We have
\begin{eqnarray}
\area{Q}  &=& \frac{ \norm{\mb v_{\mb u_1} - \mb v_{\mb u_2} }{2} + \norm{\mb w_{\mb u_1} - \mb w_{\mb u_2} }{2} }{2} \times \ell^\perp \nonumber \\
                &=& \frac{ |\mb n^T (\mb v-\mb x_0)| + |\mb n^T (\mb w-\mb x_0)| }{2} \norm{ \frac{\mb u_1}{\mb n^T \mb u_1} - \frac{\mb u_2}{\mb n^T \mb u_2} }{2} \times \ell^\perp
\end{eqnarray}
Since $\mb v \in \obj$, and $\mb v_{\mb u_2} \in \Delta \subseteq \obj$, we have $\norm{\mb v - \mb v_{\mb u_2}}{2} \le \diameter{\obj}$. If we consider the right triangle formed by $\mb v$, $\mb v_{\mb u_2}$, and $\mb v_{\mathrm{proj}} = \mb v - \mb n \mb n^T (\mb v - \mb x_0)$ (the orthogonal projection of $\mb v$ onto $\mathrm{aff}(\Delta)$), we have
\begin{eqnarray}
|\mb n^T (\mb v-\mb x_0)| &=& \norm{ \mb v - \left( \mb v - \mb n \mb n^T ( \mb v - \mb x_0 ) \right)}{2}, \nonumber \\ 
   &=& \norm{\mb v - \mb v_{\mb u_2} }{2} \, \cdot \, \sin \angle\left( \mb v - \mb v_{\mb u_2}, \mb v_{\mathrm{proj}} - \mb v_{\mb u_2} \right), \nonumber \\ 
  &=& \norm{\mb v - \mb v_{\mb u_2}}{2} \times \mb n^T \mb u_2, \nonumber  \\ &\le& \diameter{\obj} \times \mb n^T \mb u_2.
\end{eqnarray}
A similar inequality holds for $|\mb n^T( \mb w - \mb x_0 ) |$. Together, this implies that
\(
\area{Q} \;\le\; \diameter{\obj} \times \mb n^T \mb u_2 \times \norm{\frac{\mb u_1}{\mb n^T \mb u_1} - \frac{\mb u_2}{\mb n^T \mb u_2} }{2} \times \ell^\perp.
\)
Thus, we have
\begin{eqnarray}
\int_{\mb x \in Q} \innerproduct{ \mb n(\mb x) }{\barmb{u}( r^\star(\mb x) ) }^2 d \mu_{\objbdy}(\mb x) &\le& \mb n^T \mb u_1 \times \norm{\mb n^T \mb u_2 \mb u_1 - \mb n^T \mb u_1 \mb u_2 }{2} \times \diameter{\obj} \times \ell^\perp. \qquad \label{eqn:integral-bnd-1}
\end{eqnarray}
Using the triangle inequality, it is easy to show that 
\(
\norm{\mb n^T \mb u_2 \mb u_1 - \mb n^T \mb u_1 \mb u_2}{2} \;\le\; 2 \, \| \mb u_1 - \mb u_2 \|_2. \label{eqn:triangle-1}
\)
 Using the general fact that for nonzero vectors $\mb a, \mb b$, 
\(
\norm{\frac{\mb a}{\norm{\mb a}{2}} - \frac{\mb b}{\norm{\mb b}{2}}}{2} \;\le\; \frac{2 \norm{\mb a - \mb b}{2}}{\max\set{\| \mb a \|_2, \| \mb b \|_2 }},
\)
and the fact that when $\norm{\mb u - \mb u'}{2} \le \sqrt{2}$, $\norm{r \mb u + (1-r) \mb u'}{2} \ge 1/\sqrt{2}$ for all $r$ in $[0,1]$, 
we have 
\begin{eqnarray}
\| \mb u_1 -\mb u_2 \|_2  &\le& \frac{2 \norm{  r_1 \mb u + (1-r_1) \mb u' - \left( r_2 \mb u - (1-r_2) \mb u' \right) }{2}}{\| r_1 \mb u + (1-r_1) \mb u' \|_2} \nonumber \\ &\le& 2 \sqrt{2} \cdot \| \mb u - \mb u' \|_2 \, | r_1 - r_2 |,  \label{eqn:u-diff}
\end{eqnarray}
Putting together \eqref{eqn:integral-bnd-1}, \eqref{eqn:triangle-1} and \eqref{eqn:u-diff}, we get
\begin{eqnarray*}
\int_{\mb x \in Q} \innerproduct{ \mb n(\mb x) }{\barmb{u}( r^\star(\mb x) ) }^2  d \mu_{\objbdy}(\mb x) &\le& 4 \sqrt{2} \times \mb n^T \mb u_1 \times \ell^\perp \times \diameter{\obj} \times \norm{\mb u- \mb u'}{2} \times |r_1 - r_2 |.
\end{eqnarray*}
Finally, using the expression for $\mb v_{\mb u_1}$ and $\mb w_{\mb u_1}$, we obtain
\begin{eqnarray}
\mb n^T \mb u_1 \times \ell^\perp &\le& \mb n^T \mb u_1 \norm{\mb v_{\mb u_1} - \mb w_{\mb u_1} }{2} \nonumber \\ &=& \norm{ ( \mb n^T \mb u_1) ( \mb v - \mb w ) - \mb u_1 \mb n^T ( \mb v - \mb w ) }{2} \nonumber \\ &\le& 2 \norm{\mb v - \mb w }{2}.
\end{eqnarray}
This completes the proof of $(\square)$ for the case when $\< \mb n, \barmb{u}(r) \>$ is maximized at $r = r_1$. 
If $\innerproduct{\mb n}{\barmb{u}(r)}$ is instead maximized $r = r_2$, we may simply repeat the above argument, interchanging $\mb u_1$ and $\mb u_2$. If $\innerproduct{\mb n}{\barmb{u}(r)}$ is instead maximized at some $r_0 \in (r_1,r_2)$, we may partition $Q$ into two sub-quadrilaterals, indexed by $[r_0,r_1]$ and $[r_1,r_2]$, respectively, and then apply the argument to each. This establishes $(\square)$.

Our approach, then, is to discretize the domain of integration and apply $(\square)$. We make the following technical claim regarding approximation of the domain of intergration $\Xi_{e,\Delta}$ by quadrilaterals:
\begin{quote}
$(\Diamond)$ For each edge $e$ and face $\Delta$, and any $\eps > 0$, there exists a finite collection of segments
\[
[\mb a_1,\mb b_1] \cup [\mb a_2, \mb b_2] \cup \dots \cup [\mb a_N, \mb b_N] \subseteq e
\]
with disjoint relative interiors, and a collection of interior-disjoint intervals
\[
\left( \left[r^{(1)}_{2j-1},r_{2j}^{(1)}\right] \right)_{j=1}^{n_1}, \dots, \left( [r_{2j-1}^{(N)},r_{2j}^{(N)}]\right)_{j=1}^{n_N}
\]
with the following properties:
\begin{eqnarray}
&\text{(i)}& [r_{2j-1}^{(i)},r_{2j}^{(i)}] \subseteq \bigcap_{\mb z \in [\mb a_i,\mb b_i]} \tau_{e,\Delta}(\mb z),
\end{eqnarray}
which implies that $Q_{ij} \doteq \mc Q\left( [\mb a_i,\mb b_i], [r_{2j-1}^{(i)},r_{2j}^{(i)}] \right) \subseteq \Xi_{e,\Delta}$, and 
\begin{eqnarray}
&\text{(ii)}& \mu_{\objbdy} \left( \Xi_{e,\Delta} \setminus \bigcup_{i,j} Q_{ij} \right) \;\le\; \eps.
\end{eqnarray}
\end{quote}
We will show $(\Diamond)$ below. Let us first examine its implications. We have
\begin{eqnarray*}
\lefteqn{ \int_{\mb x \in \Xi_{e,\Delta}} \innerproduct{\mb n(\mb x)}{ \barmb{u}(r^\star(\mb x))}^2 d \mu_{\objbdy}(\mb x) } \\
&\le&\left( \sum_{i,j} \int_{\mb x \in Q_{ij}} \innerproduct{\mb n(\mb x)}{ \barmb{u}(r^\star(\mb x))}^2 d\mu_{\objbdy}( \mb x ) \right) \;+\; \mu\left(\Xi_{e,\Delta} \setminus \bigcup_{ij} Q_{ij} \right) \sup_{\mb x \in \objbdy} \innerproduct{\mb n(\mb x) }{ \barmb{u}(r^\star(\mb x)) }^2, \\
&\le& \sum_{ij} 8 \sqrt{2} \, \diameter{\obj} \norm{\mb u - \mb u'}{2} \norm{\mb b_i- \mb a_i}{2} | r^{(i)}_{2j} - r^{(i)}_{2j-1} | \;\;+\;\; \eps, 
\end{eqnarray*}
where the first term follows from $(\square)$. 

Consider the product $e \times [0,1]$. The rectangles $[\mb a_i, \mb b_i] \times [r_{2j-1}^{(i)},r_{2j}^{(i)}]$ have disjoint interiors. So, 
\(
\sum_{ij} \norm{\mb b_i - \mb a_i}{2} | r_{2j}^{(i)} - r_{2j-1}^{(i)}| \;=\; \int_{\mb z \in e, r\in[0,1]} \indicator{(\mb z,r) \in \bigcup_{ij} \mathrm{int}\left( [\mb a_i, \mb b_i] \times [r_{2j-1}^{(i)},r_{2j}^{(i)}] \right) } d\mb z \, dr.
\)
By construction, for any pair $(\mb z,r) \in [\mb a_i,\mb b_i] \times [r_{2j-1}^{(i)},r_{2j}^{(i)}]$ we have $\mb z_{\barmb{u}(r)} \in \Xi_{e,\Delta}$. So,
\[
\indicator{(\mb z,r) \in \bigcup_{ij} \mathrm{int}\left( [\mb a_i,\mb b_i] \times [ r_{2j-1}^{(i)},r_{2j}^{(i)}] \right) } \;\le\; \indicator{\mb z_{\barmb{u}(r)} \in \Xi_{e,\Delta}},
\]
and
\begin{eqnarray*}
\int_{\mb x \in \Xi_{e,\Delta}} \innerproduct{\mb n(\mb x)}{ \barmb{u}(r^\star(\mb x))}^2 d \mu_{\objbdy}(\mb x) &\le& 8 \sqrt{2} \, \diameter{\obj} \norm{\mb u - \mb u'}{2} \int_{\mb z \in e, r \in [0,1]} \indicator{\mb z_{\bar{\mb u}(r)} \in \Xi_{e,\Delta}} d\mb z \, dr \;\;+\;\; \eps.
\end{eqnarray*}
Since this holds for every $\eps > 0$, we have
\(
\int_{\mb x \in \Xi_{e,\Delta}} \innerproduct{\mb n(\mb x)}{\barmb{u}(r^\star(\mb x))}^2 d \mu_{\objbdy}(\mb x) \;\le\;  8 \sqrt{2} \, \diameter{\obj} \norm{\mb u - \mb u'}{2} \int_{\mb z \in e, r \in [0,1]} \indicator{\mb z_{\bar{\mb u}(r)} \in \Xi_{e,\Delta}} d\mb z\,  dr.
\)
Summing over $e,\Delta$, we obtain
\begin{eqnarray}
\sum_{e,\Delta} \int_{\mb x \in \Xi_{e,\Delta}} \innerproduct{\mb n(\mb x)}{\barmb{u}(r^\star(\mb x))}^2 d \mu_{\objbdy}(\mb x) &\le&   8\sqrt{2} \, \diameter{\obj} \norm{\mb u - \mb u'}{2} \int_{r \in [0,1]} \left( \sum_e \int_{\mb z \in e} \sum_{\Delta} \indicator{\mb z_{\bar{\mb u}(r)} \in \Xi_{e,\Delta}} d\mb z \right) \, dr. \nonumber
\end{eqnarray}
Notice that for a given edge $e = [ \mb a, \mb b]$, if it happens that $ \mb b - \mb a \in \mathrm{span}\set{\mb u, \mb u'}$, $\Xi_{e,\Delta}$ has measure zero. So, letting $\mc E'$ denote the set of edges $[\mb a, \mb b]$ for which $\mb b - \mb a \notin \mathrm{span}\set{\mb u, \mb u'}$, we have
\begin{eqnarray}
\sum_{e,\Delta} \int_{\mb x \in \Xi_{e,\Delta}} \innerproduct{\mb n(\mb x)}{\barmb{u}(r^\star(\mb x))}^2 d\mu_{\objbdy}( \mb x ) &\le&   8\sqrt{2} \, \diameter{\obj} \norm{\mb u - \mb u'}{2} \int_{r \in [0,1]} \left( \sum_{e\in \mc E'} \int_{\mb z \in e} \sum_{\Delta} \indicator{\mb z_{\bar{\mb u}(r)} \in \Xi_{e,\Delta}} d\mb z \right) \, dr. \nonumber
\end{eqnarray}
It is not difficult to show that if $e \in \mc E'$, for each $\mb x \notin e$ there is at most one $r$ such that $\mb x^{\barmb{u}(r)} \in e$. So, if $\mb x = \mb z_{\barmb{u}(r)} \in \Xi_{e,\Delta}$, it must be that $r = r^\star(\mb x)$. This implies (via $(\triangle)$) that $\mb x \in \partial S[\barmb{u}(r)]$. Since $\mb x \in \Phi$ as well, and $\mb z = \mb x^{\barmb{u}(r)}$, we have $\mb z \in \chi[\barmb{u}(r)]$. So, for $e \in \mc E'$, $\mb z_{\barmb{u}(r)} \in \Xi_{e,\Delta}$ implies that $\mb z \in \chi[\barmb{u}(r)]$. This, together with the fact that the sets $\Xi_{e,\Delta}$ and $\Xi_{e,\Delta'}$ are disjoint whenever $\Delta \ne \Delta'$ gives that 
\[
\sum_{\Delta} \indicator{\mb z_{\barmb{u}(r)}\in \Xi_{e,\Delta}} \;\le\; \indicator{\mb z \in \chi[\barmb{u}(r)]},
\]
and
\begin{eqnarray}
\sum_{e,\Delta} \int_{\mb x \in \Xi_{e,\Delta}} \innerproduct{\mb n(\mb x)}{\barmb{u}(r^\star(\mb x))}^2 d \mu_{\objbdy}(\mb x)
&\le& 8\sqrt{2} \, \diameter{\obj} \norm{\mb u - \mb u'}{2} \int_{r \in [0,1]} \left( \sum_{e \in \mc E'} \int_{\mb z \in e} \indicator{\mb z \in \chi[\bar{\mb u}(r)] } d\mb z \right) dr \qquad \nonumber \\
&\le& 8 \sqrt{2} \, \diameter{\obj} \norm{\mb u - \mb u'}{2} \int_{r \in [0,1]} \length{\chi[\bar{\mb u}(r)] } \, dr  \nonumber \\
&\le& 8 \sqrt{2} \, \diameter{\obj} \norm{\mb u - \mb u'}{2}  \chi_\star,
\end{eqnarray}
as desired. To finish the proof, we are just left to show $(\Diamond)$.

\vspace{.25in}
\noindent {\bf Demonstrating $(\Diamond)$.} We begin with the definition of $\Xi_{e,\Delta}$:
\(
\Xi_{e,\Delta} = \set{ \mb x \;\middle| \begin{array}{l} \mb x \in \mathrm{relint}(\Delta) \\ \mb x \in S[\mb u ] \setminus \left( S[\mb u'] \cup B[\mb u] \right) \\ \mb x^{\barmb{u}( r^\star(\mb x))} \in e \end{array} }.
\)
If the outward normal $\mb n$ to $\Delta$ satisfies $\mb n^T \mb u \le 0$, $\mathrm{relint}(\Delta) \subseteq B[\mb u]$, and $\Xi_{e,\Delta}$ is empty, implying that $(\Diamond)$ is trivially satisfied. Similarly, if $\mb n^T \mb u' \le 0$, $\Delta \subseteq S[\mb u']$,  $\Xi_{e,\Delta}$ is empty, and $(\Diamond)$ is trivially satisfied. To fix notation, let $e = [\mb a, \mb b]$. If $\mb b - \mb a \in \mathrm{span} \set{ \mb u, \mb u' }$, then $\Xi_{e,\Delta}$ has measure zero, and $(\Diamond)$ is again trivially satisfied. 

It remains to consider the case when $\mb n^T \mb u > 0$ and $\mb n^T \mb u' > 0$, and $\mb b - \mb a \notin \mathrm{span}\set{\mb u, \mb u'}$. We will find it slightly more convenient to work with an unnormalized version of $\barmb{u}$, by setting 
\(
\tilde{\mb u}(r) = r\mb u + (1-r) \mb u'. 
\)
It is easy to check that {$\mb x^{\barmb{u}(r)}$ is defined if and only if } $\mb x^{\tilde{\mb u}(r)}$ is defined, and $\mb x^{\tilde{\mb u}(r)} = \mb x^{\barmb{u}(r)}$. So, we can rephrase our expression for $\Xi_{e,\Delta}$ as 
\(
\Xi_{e,\Delta} = \set{ \mb x \; \middle| \begin{array}{l} \mb x \in \mathrm{relint}(\Delta) \\ \mb x \in S[\mb u ] \setminus \left( S[\mb u'] \cup B[\mb u] \right) \\ \mb x^{\tilde{\mb{u}}( r^\star(\mb x))} \in e \end{array} }.\label{eqn:xi-props}
\)
To show that $\Xi_{e,\Delta}$ can be well-approximated by quadrilaterals of the desired form, it will be useful to work in coordinates. Let 
\[
\Delta \;=\; \conv\set{ \mb v_1, \mb v_2, \mb v_3 }. 
\]
We can parameterize $\aff{\Delta}$ in terms of $\mb w \in \reals^2$ via 
\(
\mb x( \mb w ) \;=\; \mb v_1 w_1 + \mb v_2 w_2 + \mb v_3 (1 - w_1 - w_2 ) \;=\; \mb V \mb w + \mb v_3,
\)
with $\mb V = [\mb v_1 - \mb v_3 \mid \mb v_2 - \mb v_3 ] \in \reals^{3 \times 2}$. Then  
\(
\Delta = \set{ \mb x(\mb w) \mid w_1 \ge 0, w_2 \ge 0, w_1 + w_2 \le 1 }.
\)
Similarly, parameterize $e$ via 
\(
\mb z(s) = s \mb a + (1-s ) \mb b.
\)

Let 
\(
W_{e,\Delta} \;\doteq\; \set{ \mb w \mid \mb x(\mb w) \in \Xi_{e,\Delta} } \;\subset\; \reals^2.
\)
We will show that $W_{e,\Delta}$ is a {\em semialgebraic set} \cite{coste2000introduction}. As we will see, semialgebraic sets are sufficiently well-behaved to admit the approximation promised by $(\Diamond)$. To show that $W_{e,\Delta}$ is semialgebraic, we take the conditions in \eqref{eqn:xi-props} one at a time. First, notice that 
\(
\mb x(\mb w) \in \mathrm{relint}(\Delta) \quad\iff\quad w_1 > 0, \; w_2 > 0, \;\; \text{and} \,\; w_1 + w_2 < 1.
\)
Set 
\(
W_1 = \set{ \mb w \in \reals^2 \mid w_1 > 0, \; w_2 > 0, \;\; \text{and} \,\; w_1 + w_2 < 1 }.
\)
The set $W_1$ is semialgebraic. 

Now, take the second condition in \eqref{eqn:xi-props}: $\mb x \in S[\mb u]  \setminus \left( S[\mb u'] \cup B[\mb u] \right)$. Since $\mb n^T \mb u > 0$, this condition reduces to $\mb x \in S[\mb u] \setminus S[\mb u']$. Moreover, since $\mb n^T \mb u' > 0$, this condition further reduces to 
\(
\exists \; t > 0 \;\; \mathrm{s.t.} \;\; \mb x + t\mb u \in \objbdy, \qquad \text{and} \qquad \nexists \; t' > 0 \;\; \mathrm{s.t.} \;\; \mb x + t' \mb u' \in \objbdy. 
\)
For each $\Delta'$, define a coordinate map $\mb x_{\Delta'}(\mb w')$ in the same manner as $\mb x(\mb w)$. For each $\Delta'$, write 
\(
S_{2,\Delta'} = \set{ \mb w, \mb w', t \; \middle| \begin{array}{l} \mb x(\mb w) + t \mb u = \mb x_{\Delta'}(\mb w') \\ t > 0 \\ w'_1 \ge 0, \; w'_2 \ge 0, \; w'_1 + w'_2 \le 1 \end{array} } \subset \reals^5.
\)
Write $\mc P_{\mb w}$ for the projection onto the $\mb w$ coordinates, and 
\(
W_{2,\Delta'} = \mc P_{\mb w} S_{2,\Delta'}. 
\)
With this definition, notice that $\exists \, t > 0$ such that $\mb x(\mb w) + t \mb u \in \objbdy$ if and only if $\mb w \in \bigcup_{\Delta' \ne \Delta} W_{2,\Delta}$. Moreover, since each $S_{2,\Delta'}$ is defined by finitely many polynomial inequalities, each $S_{2,\Delta'}$ is semialgebraic. By the Tarski-Seidenberg theorem, each $W_{2,\Delta'}$ is also semialgebraic. 

In a similar manner, define 
\(
S_{3,\Delta'} = \set{ \mb w, \mb w', t' \; \middle| \begin{array}{l} \mb x(\mb w) + t'\mb u' = \mb x_{\Delta'}(\mb w') \\ t' > 0 \\ w'_1 \ge 0, \; w'_2 \ge 0,\; w'_1 + w'_2 \le 1 \end{array} }
\)
and $W_{3,\Delta'} = \mc P_{\mb w} S_{3,\Delta'}$. The $W_{3,\Delta'}$ are also semialgebraic. Combining these sets, we have that $\mb x(\mb w) \in S[\mb u] \setminus S[\mb u']$ if and only if 
\(
\mb w \in \left( \bigcup_{\Delta'} W_{2,\Delta'} \right) \setminus \left( \bigcup_{\Delta'} W_{3,\Delta'} \right) \;\doteq\; W_4.
\)
The set $W_4$ is produced from semialgebraic sets via finitely many set operations, and hence is semialgebraic. 

The final condition in \eqref{eqn:xi-props} that we need to consider is that $\mb x^{\tilde{\mb u}(r^\star(\mb x))} \in e$. Because $\mb b - \mb a \notin \mathrm{span} \set{\mb u, \mb u'}$, for each $\mb x$ there exists at most one pair $(\hat{t},\hat{r})$ with $\hat{t} > 0, \hat{r} \in [0,1]$ such that $\mb x + \hat{t} \tilde{\mb u}(\hat{r}) \in e$. Hence, there exists at most one $\hat{r} \in [0,1]$ such that $\mb x^{\tilde{\mb u}(\hat{r})} \in e$. For any given $\hat{r}$, $\mb x^{\tilde{\mb u}(\hat{r})} \in e$ if and only if the following two conditions are satisfied:
\begin{eqnarray}
\exists \, \hat{t} \; &\text{s.t.}& \; \hat{t} > 0, \; \mb x + \hat{t} \tilde{\mb u}(\hat{r}) \in e \\
\nexists\, (\hat{t},t') \; &\text{s.t.}& \; 0 < t' < \hat{t}, \; \mb x + \hat{t} \tilde{\mb u}(\hat{r}) \in e, \; \mb x + t' \tilde{\mb u}(\hat{r}) \in \objbdy.
\end{eqnarray}
The first condition ensures that the ray $\mb x + \reals_{++} \tilde{\mb u}(\hat{r})$ intersects $e$, while the second ensures that no other point of $\objbdy$ lies between $\mb x$ and this intersection on the ray. 

Set 
\(
S_5 = \set{ \hat{r}, \hat{t}, \mb w, s \; \middle| \begin{array}{l} \mb x(\mb w) + \hat{t} \tilde{\mb u}(\hat{r}) = \mb z(s) \\ s \in [0,1] \\ \hat{t} > 0 \\ \hat{r} \in [0,1] \end{array} }
\)
and $W_5 = \mc P_{\mb w} S_5$. Write
\(
S_{6,\Delta'} = \set{ \hat{r}, \hat{t}, \mb w, s, t', \mb w' \; \middle| \begin{array}{l} \mb x(\mb w) + \hat{t} \tilde{\mb u}(\hat{r}) = \mb z(s) \\ s \in [0,1] \\ \hat{t} > 0 \\ \hat{r} \in [0,1] \\ 0 < t' < \hat{t} \\ w'_1 \ge 0, \; w'_2 \ge 0, \; w'_1 + w'_2 \le 1 \\ \mb x(\mb w) + t' \tilde{\mb u}(\hat{r}) = \mb x_{\Delta'}(\mb w') \end{array} },
\)
and $W_{6,\Delta'} = \mc P_{\mb w} S_{6,\Delta'}$. Then there exists $\hat{r} \in [0,1]$ such that $\mb x(\mb w)^{\tilde{\mb u}(\hat{r})} \in e$ if and only if 
\(
\mb w \in W_5 \setminus \bigcup_{\Delta' \ne \Delta} W_{6,\Delta'} \;\doteq\; W_7.
\)
Again, the set $W_7$ is semialgebraic. 

Consider $\mb w \in W_7$. By construction this means that there exists $\hat{r}$ such that $\mb x(\mb w)^{\tilde{\mb u}(\hat{r})} \in e$. Moreover, by the above reasoning, this $\hat{r}$ is the only $r$ with this property. Is $\hat{r} = r^\star( \mb x(\mb w))$? This is true if and only if there does not exist $r' \in (0,\hat{r})$ and $t' > 0$ with $\mb x(\mb w) + t' \tilde{\mb u}(r') \in \objbdy$. Let 
\(
S_{8,\Delta'} = \set{ \hat{r}, \hat{t}, \mb w, s, r', t', \mb w' \; \middle| \begin{array}{l} \mb x(\mb w) + \hat{t} \tilde{\mb u}(\hat{r}) = \mb z(s) \\ s \in [0,1] \\ \hat{t} > 0 \\ \hat{r} \in [0,1] \\ t' >0 \\ 0 < r' < \hat{r} \\ w'_1 \ge 0, \; w'_2 \ge 0, \; w'_1 + w'_2 \le 1 \\ \mb x(\mb w) + t' \tilde{\mb u}(\hat{r}) = \mb x_{\Delta'}(\mb w') \end{array} }
\)
and again set $W_{8,\Delta'} = \mc P_{\mb w} S_{8,\Delta'}$. Then we have 
\(
\mb x(\mb w)^{\tilde{\mb u}(r^\star(\mb x(\mb w)))} \in e \quad \iff \quad \mb w \in W_7 \setminus \bigcup_{\Delta' \ne \Delta} W_{8,\Delta'} \quad \doteq \quad W_9.
\)
Hence, setting $W_{e,\Delta} = W_1 \cap W_4 \cap W_9$, we have 
\(
\mb x(\mb w) \in \Xi_{e,\Delta} \quad \iff \quad \mb w \in W_{e,\Delta},
\)
and the set $W_{e,\Delta}$ is semialgebraic. 

Our next task is to rewrite $W_{e,\Delta}$ in terms of the parameters $s,r$. Since $\mb n^T \mb u > 0 $ and $\mb n^T \mb u' > 0$, for every $r \in [0,1]$, $\mb n^T \tilde{\mb u}(r) > 0$. Hence, for every $r,s$ there is a unique $t$ such that 
\[
\mb z(s) - t \tilde{\mb u}(r) \in \aff{\Delta}.
\]
In particular, there exists a unique $(\mb w,t) \in \reals^2 \times \reals$ such that 
\[
\mb z(s) - t \tilde{\mb u}(r) = \mb x(\mb w).
\]
This gives a system of equations 
\(
\bigl[\;  \mb V \mid \tilde{\mb u}(r) \; \bigr] \left[ \begin{array}{c} \mb w \\ t \end{array} \right] \;=\; \mb z(s) - \mb v_3.
\)
Under our assumptions, it is not difficult to show that there exists $\zeta > 0$ such that 
\(
\sigma_{\min}\left( \bigl[ \; \mb V \mid \tilde{\mb u}(r) \;\bigr]\right) \;\ge\; \zeta 
\)
for all $r \in [0,1]$.\footnote{Because $\mb n^T \tilde{\mb u}(r) > 0$ for all $r \in [0,1]$, the matrix $[ \, \mb V \, \mid \, \tilde{\mb u}(r) \, ]$ is full rank for all $r$, and hence for all $r$ its smallest singular value is positive. Noting that $\sigma_{\min}(\mb M)$ is a continuous function of $\mb M$ and $\tilde{\mb u}(r)$ a continuous function of $r$, it must be that $\sigma_{\min}([\mb V \mid \tilde{\mb u}(r) ])$ achieves its infimum over $[0,1]$, and hence its infimum is strictly larger than zero.} 
Hence, we can write the coordinates $\mb w$ in this unique pair explicitly as a function of $(r,s)$:
\[
\mb w \;=\; \Upsilon(r,s) \;=\; \left[ \begin{array}{ccc} 1 & 0 & 0 \\ 0 & 1 & 0 \end{array} \right] \left[ \begin{array}{c|c} \mb V & \tilde{\mb u}(r) \end{array} \right]^{-1} \left( \mb z(s) - \mb v_3 \right). 
\]
The components of $\Upsilon$ are rational functions of $(r,s)$, the denominator of which does not vanish. This implies that $\Upsilon$ is a {\em semialgebraic map} \cite{coste2000introduction}. Set
\(
A_{e,\Delta} \;=\; \Upsilon^{-1}\left[ W_{e,\Delta} \right] \;\subseteq\; [0,1]^2.
\)
Since $\Upsilon$ is a semialgebraic map, $A_{e,\Delta}$ is also semialgebraic. $A_{e,\Delta} \subseteq [0,1]^2$ is also bounded. Because semialgebraic sets are finite unions of intersections of sublevel sets of finitely many polynomial (and hence continuous) functions, bounded semialgebraic sets are Jordan measurable. This implies that for every $\eta > 0$, there exists an $M \in \bb Z$ and a collection of $M$ interior disjoint rectangles $R_i = [r_{i,1},r_{i,2}] \times [s_{i,1},s_{i,2}]$ $(i=1\dots m)$ such that 
\[
\mu\left( A_{e,\Delta} \setminus \bigcup_{i=1}^M R_i \right) \;\le\; \eta. 
\]

The function $\Upsilon$ is differentiable on $(r,s) \in [0,1]^2$. Moreover, it is not difficult to show that there exists $\xi < + \infty$ such that 
\(
\sup_{(r,s) \in [0,1]^2} \magnitude{ \det \left( \frac{\partial \Upsilon}{\partial (r,s)}(r,s) \right)  } \;\le\; \xi.
\)
Furthermore, for some $\xi'$, the map $\mb w \mapsto \mb x(\mb w)$ satisfies 
\( 
\magnitude{ \det \left(\frac{\partial}{\partial \mb w}(\varphi^{-1} \circ \mb x)(\mb w) \right) } \;\le\; \xi'
\)
for all $\mb w$. So, noting that $\Xi_{e,\Delta} = \mb x\left( \Upsilon[ A_{e,\Delta}] \right)$, we have 
\(
\mu_{\objbdy} \left( \Xi_{e,\Delta} \setminus \bigcup_{i=1}^M \mb x(\Upsilon[R_i]) \right) \;\le\; \xi \xi' \eta.
\)
Choose $\eta = \frac{\eps }{ \xi \xi' }$. 

Order all of the endpoints $s_{i,j}$ of the $R_i$, to produce $0 \le s_1 < s_2 < \dots < s_N \le 1$. Set 
\(
R'_{i,j} = R_j \cap \left( [s_i,s_{i+1}] \times [0,1] \right).
\)
Set $\mb a_i = \mb z(s_i)$, $\mb b_i = \mb z(s_{i+1})$. Each $R'_{i,j}$ either has empty interior, or has the form $[s_i,s_{i+1}] \times [r_{i,j},r_{i,j+1}]$. Hence, there exists a collection of disjoint intervals $[r_{1},r_2], [r_{3},r_{4}], \dots, [r_{2 n_i-1},r_{2 n_i} ]$ such that $\cup_j R'_{i,j} = \cup_{j} [s_i,s_{i+1}] \times [r_{2j-1},r_{2j}]$. This collection of intervals has the desired properties. 
\end{proof}

\section{Proof of Lemma \ref{lemma:T}} \label{app:T}

\begin{proof} Using the change of variables formula, it is not difficult to show that 
\(
\tilde{\nu}(\mb x) = 1 - \frac{1}{\pi} \int \frac{ \innerproduct{\mb n(\mb x)}{\mb y - \mb x} \innerproduct{ \mb n(\mb y)}{\mb x - \mb y} }{\norm{\mb y - \mb x}{2}^4} \,V(\mb x, \mb y) \,  d \mu_{\objbdy}(\mb y). 
\)
For any $\mb y$, we have 
\begin{eqnarray}
\int_{\mb y} \kappa(\mb x, \mb y) \, d \mu_{\objbdy}(\mb y) &=& \; \frac{\rho(\mb x) }{\pi} \int_{\mb y} \frac{\innerproduct{\mb n(\mb x)}{\mb y - \mb x} \innerproduct{\mb n(\mb y)}{\mb x -\mb y}}{\norm{\mb x - \mb y}{2}^4} V(\mb x, \mb y) \, d \mu_{\objbdy}(\mb y) \nonumber \\
&=& \; \rho(\mb x) ( 1 - \tilde{\nu}(\mb x) ) \nonumber \\
&\le& \; \rho_\star ( 1 - \nu_\star ).
\end{eqnarray}
Similarly, for any $\mb x \in \objbdy$, we have
\begin{eqnarray}
\int_{\mb x} \kappa(\mb x, \mb y) \; d \mu_{\objbdy}(\mb x) &=& \frac{1}{\pi} \int_{\mb x} \rho(\mb x) \frac{ \innerproduct{ \mb n(\mb x) }{\mb y - \mb x} \innerproduct{\mb n(\mb y)}{\mb x- \mb y}}{\norm{\mb x- \mb y}{2}^4 } V(\mb x, \mb y ) \, d \mu_{\objbdy}(\mb x) \nonumber \\
&\le& \frac{\rho_\star}{\pi} \int_{\mb x} \frac{\innerproduct{\mb n(\mb x)}{\mb y - \mb x} \innerproduct{\mb n(\mb y)}{\mb x - \mb y}}{\norm{\mb x - \mb y}{2}^4} V(\mb x, \mb y) \, d\mu_{\objbdy}(\mb x) \nonumber \\
&\le& \rho_\star ( 1 - \tilde{\nu}(\mb x)) \nonumber \\
&\le& \rho_\star ( 1 - \nu_\star ).
\end{eqnarray}
By Theorem II.1.6 of \cite{Conway}, this implies that for $g \in L^2[\objbdy]$, $\mc T[g] \in L^2[\objbdy]$, and $\norm{\mc T}{L^2 \to L^2} \le \rho_\star( 1- \nu_\star )$. 
\end{proof}

\section{Proof of Lemma \ref{lemma:P}} \label{sec:D}

We can obtain Lemma \ref{lemma:P} of the paper using the change of variables formula. Before jumping into the proof of this bound, we record a quick lemma:

\begin{lemma}\label{lem:P} Let $\mb u$, $\mb v \in \reals^n$ ($n > 1$) such that $\mb u^* \mb v \ne 0$. Then 
\(
\norm{\mb I - \frac{\mb u \mb v^*}{\mb u^* \mb v} }{} = \norm{\frac{\mb u \mb v^*}{\mb u^* \mb v}}{} = \frac{\norm{\mb u}{2} \norm{\mb v}{2} }{\magnitude{ \mb u^* \mb v }}.
\)
\end{lemma}
\begin{proof} If $\mb u$ and $\mb v$ are linearly dependent, the result is immediate. Let us assume they are linearly independent. Let $\mb M = \mb I - \frac{\mb u \mb v^*}{\mb u^* \mb v}$, and consider the eigenvalues of 
\(
\mb M \mb M^* \;=\; \mb I - \frac{1}{\mb u^* \mb v} \left( \mb u \mb v^* + \mb v \mb u^* \right) + \frac{ \mb u \mb v^* \mb v \mb u^* }{(\mb u^* \mb v)^2}
\)
Notice that if $\mb x \perp \mb u, \mb v$, we have $\mb M\mb M^* \mb x = \mb x$. We can find an orthonormal basis of $n -2$ vectors for $(\mb u, \mb v)^\perp$. Since any orthonormal collection of eigenvectors of a symmetric matrix can be completed to an orthonormal basis of eigenvectors, there must exist two linearly independent eigenvectors lying in
$\mathrm{span}\set{ \mb u, \mb v}$. If $\mb x$ is an eigenvector with eigenvalue $\lambda$, we have 
\( \label{eqn:eval-exp}
(1-\lambda) \mb x = \left[ \frac{1}{\mb u^* \mb v} \left( \mb u \mb v^* + \mb v \mb u^* \right) - \frac{\mb u \mb v^* \mb v \mb u^*}{(\mb u^* \mb v)^2} \right] \mb x.
\)
For $\mb x \in \mathrm{span}\set{\mb u, \mb v}$, write $\mb x = \alpha \mb u + \beta \mb v$. Plugging into the above equation and using linear independence of $\mb u$ and $\mb v$, we obtain
\begin{eqnarray}
(1-\lambda) \alpha &=& \left( 1 - \frac{\mb v^* \mb v \mb u^* \mb u}{(\mb u^* \mb v)^2} \right) \alpha, \\
(1-\lambda) \beta  &=& \frac{\mb u^* \mb u}{\mb u^* \mb v} \alpha + \beta.
\end{eqnarray}
The first equation implies that either $\alpha = 0$, or $\lambda = \norm{\mb v}{2}^2 \norm{\mb u}{2}^2 / ( \mb u^* \mb v)^2$. If $\alpha = 0$, this implies that $\mb v$ is an eigenvector with eigenvalue $\lambda = 1$. Plugging back into \eqref{eqn:eval-exp}, and simplifying, we get $\mb v = \mb 0$, contradicting $\mb u^* \mb v \ne 0$. Hence, $\alpha = 0$ cannot give an eigenvector under our assumptions, and it must be that the eigenvalue is $\lambda = \norm{\mb v}{2}^2 \norm{\mb u}{2}^2 / ( \mb u^* \mb v)^2$. By Cauchy-Schwarz, this quantity is strictly larger than one, and hence it is the largest eigenvalue of $\mb M \mb M^*$. Hence, $\norm{\mb M}{} = \norm{\mb v}{2} \norm{\mb u}{2} / | \mb u^* \mb v |$. It is straightforward to observe that this quantity is also the norm of $\mb u \mb v^* / \mb u^* \mb v$. 
\end{proof}

Using this lemma and the change of variables formula, we can control the norm of the maps $\mc P_i$:

\begin{lemma}\label{lem:Pi-individual} For each $i$, $\norm{\mc P_i}{L^2 \to \reals} \;\le\; 2^{1/4} \beta f s / \ell$. 
\end{lemma}
\begin{proof}
We can define a restricted perspective projection $\tilde{\mf p} :\objbdy_+\to \Pi_I$ via 
\(
\tilde{\mf p}(\mb x) = - f \frac{ \mb x}{\innerproduct{ \mb x}{\mb e_3}}.
\) 
Here, $\objbdy_+$ stands for the visible part of the object from the camera, defined as 
\[
\objbdy_+=\set{\mb x\in\objbdy\mid \conv\{\mb x,\mb 0\}\cap\obj=\mb x}.
\] 
The image coordinates are read off as the first two values $\mf p (\mb x) = \mb P_{12} \, \tilde{\mf p} (\mb x)$, via
\(
\mb P_{12} = \left[\begin{array}{ccc} 1 & 0 & 0 \\ 0 & 1 & 0 \end{array} \right].
\)
The map $\mf p$ is injective, and its inverse $\mf p^{-1} : \mathrm{im}(\mf p) \subseteq \reals^2 \to \objbdy_+$ exists. In our sensor model, we can write the value of the $i$-th pixel as 
\begin{eqnarray}
\mc P_i[g] &=& \beta \int_{\mb z \in I_i \cap \mathrm{im}( \mf p )} g( \mf p^{-1} \mb z ) \innerproduct{ \frac{\mb z}{\norm{\mb z}{2}} }{\mb e_3 }^4 \; d \mu(\mb z) \nonumber \\
&\doteq& \beta \int_{\mb z \in I_i \cap \mathrm{im}(\mf p)} g(\mf p^{-1} \mb z) \, \cos^4\left( \alpha( \mb z ) \right) \; d \mu(\mb z). 
\end{eqnarray}
We can change variables as above. Write 
\begin{eqnarray*} 
\mc P_i[g] &=& \beta \sum_j \int_{\mb z \in I_i \cap \mf p\left[\objbdy_+ \cap \varphi_j[ U_j ]\right]} g( \mf p^{-1} \mb z ) \; \cos^4\left( \alpha( \mb z ) \right) \; d \mu(\mb z ). 
\end{eqnarray*}

Here, $\varphi_j : U_j \to \Delta_j$ is defined as in Appendix \ref{app:int-obj-bdy}. Using the change of variables formula, this becomes 
\begin{eqnarray}
\mc P_i[g] &=& \beta \sum_j \int_{\mb w \in \varphi_j^{-1} \left[ {\objbdy_+} \cap \mf p^{-1} [ I_i ] \right]} g \circ \varphi_j(\mb w) \; \cos^4\left( \alpha( \mf p \, \varphi_j \, \mb w ) \right)\; \magnitude{\det\left( \frac{\partial \, \mf p \circ \varphi_j}{\partial \mb w}(\mb w)\right) } \; d \mu(\mb w). \nonumber
\end{eqnarray}
Writing 
\(
\zeta_j(\mb z) \; = \indicator{ \mb z \in I_i \cap \mf p[{\objbdy_+} \cap \varphi_j[U_j] ] }.
\)
The above expression becomes 
\(
\mc P_i[g] \;=\; \beta \sum_j \int_{\mb w \in U_j} g \circ \varphi_j(\mb w) \;\; \zeta_j \circ \mf p \circ \varphi_j(\mb w) \;\; \cos^4( \alpha( \mf p  \,\varphi_j\, \mb w )) \magnitude{ \det\left( \frac{\partial \, \mf p\circ \varphi_j}{\partial \mb w}( \mb w)\right) } \; d \mu(\mb w ). 
\)
From this, we have 
\(
\norm{ \mc P_i }{L^2\to \reals}^2 \;=\; \beta^2 \sum_j \int_{\mb w \in U_j} \; ( \zeta_j \circ \mf p \circ \varphi_j )^2(\mb w) \; \cos^8( \alpha( \mf p \, \varphi_j \, \mb w  )) \; \magnitude{ \det \left(\frac{\partial \, \mf p \circ \varphi_j}{\partial \mb w}(\mb w)\right) }^2 \; d \mu(\mb w). 
\)
To evaluate this, we can change variables again. Write 
\begin{eqnarray}
\lefteqn{\norm{ \mc P_i }{L^2 \to \reals}^2 \quad=\quad  \beta^2 \sum_j \int_{\mb z \in \reals^2} \zeta_j^2(\mb z) \; \cos^8\left( \alpha(\mb z) \right) \; \magnitude{ \det \left( \frac{ \partial \, \mf p \circ \varphi_j }{ \partial \mb w}\right)( \varphi_j^{-1} \mf p^{-1} \mb z)  }^2 \magnitude{ \det\left( \frac{\partial ( \mf p \circ \varphi_j )^{-1} }{\partial \mb z}(\mb z) \right) } \; d \mu( \mb z ) }\nonumber \\
 &=& \beta^2 \sum_j \int_{\mb z \in \reals^2} \zeta_j^2(\mb z) \; \cos^8\left(\alpha(\mb z)\right) \; \magnitude{ \det\left(  \frac{ \partial\, \mf p \circ \varphi_j }{ \partial \mb w}\right)( \varphi_j^{-1} \mf p^{-1} \mb z) }\; d \mu(\mb z). \hspace{1.5in}
\end{eqnarray}
To finish, we need to get a bound on the determinant term. For this, we use the fact that 
\(
\frac{\partial \tilde{\mf p}}{\partial x} \;=\; - \frac{f}{\innerproduct{\mb e_3}{\mb x}} \left( \mb I - \frac{\mb x \mb e_3^*}{\innerproduct{\mb e_3}{\mb x}} \right),
\)
and 
\begin{eqnarray}
 \det \left( \frac{ \partial\, \mf p \circ \varphi_j }{ \partial \mb w} \right) &=& \det \left( \mb P_{12} \frac{\partial \tilde{\mf p}}{\partial \mb x} \mb U_j \right) \quad\le\quad \norm{ \mb P_{12} \frac{\partial \tilde{ \mf p }}{\partial \mb x} \mb U_j }{}^2 \quad\le\quad \norm{ \frac{\partial \tilde{\mf p}}{\partial \mb x} }{}^2 \nonumber \\ &=& \frac{f^2}{(\mb e_3^* \mb x)^2} \norm{ \mb I - \frac{\mb x \mb e_3^*}{\mb e_3^* \mb x} }{}^2 \quad=\quad \frac{f^2}{(\mb e_3^* \mb x)^2} \norm{ \frac{\mb x}{\mb e_3^* \mb x} }{2}^2, 
\end{eqnarray}
where in the final line, we have used the above lemma. To bound the terms in this expression, notice that since $\mb x \in \obj$, $\mb e_3^* \mb x \ge \ell$. Notice also that 
\(
-\frac{\mb x}{\mb e_3^* \mb x} \;=\; \frac{1}{f} \tilde{\mf p}(\mb x).
\)
We have 
\(
\norm{\tilde{\mf p}(\mb x)}{2} \;=\; \sqrt{ f^2 + f^2 \tan^2 \alpha }. 
\)
So, finally, we obtain 
\(
\magnitude{ \det \left( \frac{\partial \, \mf p \circ \varphi_j}{\partial \mb w} \right) } \;\le\; \frac{f^2}{\ell^2} \, ( 1 + \tan \alpha ). 
\)
It is not difficult to show\footnote{In fact, the right hand side of \eqref{eqn:cos-8} can be easily tightened to $\cos^8 \alpha \times ( 1 + \tan \alpha) \le c < 1.1 $. We will not pursue tight constants here, however.} that for all $\alpha$, 
\( \label{eqn:cos-8}
| \cos^8 \alpha \times ( 1 + \tan \alpha ) | \;\le\; \sqrt{2}.
\)
Combining everything together, we obtain 
\begin{eqnarray}
\norm{ \mc P_i }{L^2 \to \reals}^2 &\le&  \sqrt{2} \frac{\beta^2 f^2}{ \ell^2} \sum_j \int_{\mb z} \zeta_j^2(\mb z) \; d\mu(\mb z) \quad\le\quad \sqrt{2} \frac{\beta^2 f^2}{\ell^2} s^2,
\end{eqnarray}
giving the desired result.
\end{proof}

This gives us a fairly direct proof of Lemma \ref{lemma:P} of the paper:
\begin{proof}[Proof of Lemma \ref{lemma:P}] We have
\(
\norm{\mc P}{L^2 \to \ell^2} \;=\; \sup_{\norm{g}{L^2} \le 1} \norm{ \left[ \begin{array}{c} \mc P_1[g] \\ \vdots \\ \mc P_m[g] \end{array} \right] }{2}.
\)
Because the $I_i$ are disjoint and $\mf p$ is injective, the sets $\Xi_i = \mf p^{-1}[I_i \cap \mathrm{im}(\mf p)]$ are disjoint, and
\(
\norm{g}{L^2}^2 \;\ge\; \sum_i \norm{ g \, \indicator{\Xi_i} }{L^2}^2.
\)
From Lemma \ref{lem:Pi-individual}, we have
\(
\magnitude{ \mc P_i[ g ] } \;\le\; \norm{ \mc P_i[g] }{L^2 \to L^2} \norm{ g \, \indicator{\Xi_i} }{L^2} \;\le\; ( 2^{1/4}\beta s f / \ell) \norm{ g \, \indicator{\Xi_i} }{L^2},
\)
and so
\(
\norm{ \mc P[g] }{\ell^2}^2 \;\le\; (2^{1/4} \beta s f / \ell )^2 \sum_i \norm{ g \, \indicator{\Xi_i} }{L^2}^2 \;\le\; (2^{1/4} \beta s f / \ell )^2 \norm{g}{L^2}^2,
\)
completing the proof.
\end{proof}

\input{app_complexity_reduction}

%% file: app_complexity_reduction.tex
\section{Proofs from Section \ref{sec:Complexity-Reduction}}

\paragraph{Proof of Lemma \ref{lemma:lrsd}.}
\begin{proof}

Consider any $\widehat{\mathbf{A}}\in\Omega_{1}.$ By the definition
of $\Omega_{1},$ we have
\[
\sup_{\mathbf{y}\,\in\,{\rm cone}\left(\bar{\mathbf{A}}\right),\,\|\mathbf{y}\|\le1}d\left(\mathbf{y},\,{\rm cone}\left(\widehat{\mb{A}}\right)\right)\;\le\;\gamma'\le\frac{\gamma}{\gamma+1}\le\gamma.
\]
We also have that
\[
\sup_{\mathbf{y}\in{\rm cone}\left(\widehat{\mathbf{A}}\right),\,\|\mathbf{y}\|\le1}d\left(\mathbf{y},\,{\rm cone}\left(\bar{\mb{A}}\right)\right)\le\max_{\mathbf{x}\ge\mathbf{0},\;\|\widehat{\mathbf{A}}\mathbf{x}\|_{2}\le1}\|\bar{\mb{A}}\mathbf{x}-\widehat{\mathbf{A}}\mathbf{x}\|_{2}.
\]
Since $\forall$ $\mathbf{x}\ge\mathbf{0}$, $\|\bar{\mb{A}}\mathbf{x}-\widehat{\mathbf{A}}\mathbf{x}\|_{2}\le\gamma'\|\bar{\mb{A}}\mathbf{x}\|_{2}$,
we have
\begin{eqnarray*}
 \left\{ \mathbf{x}\;\middle|\; \mathbf{x}\ge\mathbf{0},\|\widehat{\mathbf{A}}\mathbf{x}\|_{2}\le1\right\}
 & \subseteq & \left\{ \mathbf{x}\;\middle|\;\mathbf{x}\ge\mathbf{0},\|\bar{\mb{A}}\mathbf{x}\|_{2}\le1+\|\bar{\mb{A}}\mathbf{x}-\widehat{\mathbf{A}}\mathbf{x}\|_{2}\right\} \\
 & \subseteq & \left\{ \mathbf{x}\;\middle|\;\mathbf{x}\ge\mathbf{0},\|\bar{\mb{A}}\mathbf{x}\|_{2}\le\frac{1}{1-\gamma'}\right\} .
\end{eqnarray*}
Thus
\begin{eqnarray*}
\sup_{\mathbf{y}\,\in\,{\rm cone}\left(\widehat{\mathbf{A}}\right),\,\|\mathbf{y}\|\le1}d\left(\mathbf{y},\,{\rm cone}\left(\bar{\mb{A}}\right)\right)
& \le & \max_{\mathbf{x}\ge\mathbf{0},\;\|\widehat{\mathbf{A}}\mathbf{x}\|_{2}\le1}\|\bar{\mb{A}}\mathbf{x}-\widehat{\mathbf{A}}\mathbf{x}\|_{2}\\
& \le & \max_{\mathbf{x}\ge\mathbf{0},\;\|\bar{\mb{A}}\mathbf{x}\|_{2}\le\frac{1}{1-\gamma'}}\|\bar{\mb{A}}\mathbf{x}-\widehat{\mathbf{A}}\mathbf{x}\|_{2}\\
& = & \frac{1}{1-\gamma'}\times \max_{\mathbf{x}\ge\mathbf{0},\;\|\bar{\mb{A}}\mathbf{x}\|_{2}\le1}\|\bar{\mb{A}}\mathbf{x}-\widehat{\mathbf{A}}\mathbf{x}\|_{2}\\
 & = & \frac{\gamma'}{1-\gamma'}\quad\le\quad \gamma.
\end{eqnarray*}
Therefore, $\widehat{\mathbf{A}}\in\Omega_{1}$ implies $\widehat{\mathbf{A}}\in\Omega_{0}$ as desired.
\end{proof}

\paragraph{Proof of Lemma \ref{lemma:lifting}.}
\begin{proof}
By making the transformation $\mathbf{X}=\mathbf{x}\mathbf{x}^{T}$,
we have that $\max_{\mathbf{x}\ge\mathbf{0},\;\|\bar{\mb A}\mathbf{x}\|_{2}\le1}\|\widehat{\mb A}\mathbf{x}-\bar{\mb A}\mathbf{x}\|_{2}$
equals
\begin{eqnarray*}
 & \max & \left\langle\left(\widehat{\mathbf{A}}-\bar{\mb A}\right)^{T}\left(\widehat{\mathbf{A}}-\bar{\mb A}\right),\,\mathbf{X}\right\rangle\\
 & {\rm s.t.} & \langle\bar{\mb A}^{T}\bar{\mb A},\,\mathbf{X}\rangle\le1\\
 &  & \mathbf{X}\ge\mathbf{0},\,\mathbf{X}\succeq\mathbf{0}\\
 &  & {\rm rank\left(\mathbf{X}\right)=1}.
\end{eqnarray*}
Dropping the rank constraint gives the result.
\end{proof}

\paragraph{Proof of Theorem \ref{thm:cp-lrsd}}We will prove the theorem via two lemmas below:
\begin{lemma} Consider $\Omega_{3}\doteq\left\{ \widehat{\mathbf{A}}\;\middle|\; f\left(\widehat{\mathbf{A}}\right)\le \bar\gamma\right\} ,$
where $ \bar\gamma=(\gamma')^2$ and
\begin{eqnarray*}
f\left(\widehat{\mathbf{A}}\right)&\doteq&\min_{\left(\mathbf{\mu},\beta\right)} \beta\\
 && {\rm {s.t.}}  \left[\begin{matrix}
\mathbf{I} & \widehat{\mathbf{A}}-\bar{\mb A}\\
\left(\widehat{\mathbf{A}}-\bar{\mb A}\right)^{T} & \beta\bar{\mb A}^{T}\bar{\mb A}-\mathbf{\mu}
\end{matrix}\right]{ \succeq\mathbf{0}, \;  \mathbf{\mathbf{\mu}}\ge\mathbf{0},\,\beta\ge0}.
\end{eqnarray*}
Then $\Omega_{3}=\Omega_{2}$.
\end{lemma}

\begin{proof}
The above lemma follows directly from the fact that the
dual problem of
\begin{eqnarray}
\left(\mbox{P}\right) & \max & \left\langle\left(\widehat{\mathbf{A}}-\bar{\mb A}\right)^{T}\left(\widehat{\mathbf{A}}-\bar{\mb A}\right),\,\mathbf{X}\right\rangle\nonumber \\
 & \mbox{s.t.} & \left\langle\bar{\mb A}^{T}\bar{\mb A},\,\mathbf{X}\right\rangle\le1\label{eq:inner_optimization}\\
 &  & \mathbf{X}\ge\mathbf{0},\,\mathbf{X}\succeq\mathbf{0},\nonumber
\end{eqnarray}
can be written as
\begin{eqnarray*}
\left(\mbox{D}\right) & \min & \beta\\
 & \mbox{s.t.} & \left[\begin{array}{cc}
\mathbf{I} & \widehat{\mathbf{A}}-\bar{\mb A}\\
\left(\widehat{\mathbf{A}}-\bar{\mb A}\right)^{T} & \beta\bar{\mb A}^{T}\bar{\mb A}-\mathbf{\mu}
\end{array}\right]\succeq\mathbf{0}\\
 &  & \mathbf{\mathbf{\mu}}\ge\mathbf{0},\,\beta\ge0,
\end{eqnarray*}
with zero duality gap.

To derive that dual reformulation, we first reformulate problem \eqref{eq:inner_optimization} as:
\begin{eqnarray}
 & \max & \left\langle\left(\widehat{\mathbf{A}}-\bar{\mb A}\right)^{T}\left(\widehat{\mathbf{A}}-\bar{\mb A}\right),\,\mathbf{X}\right\rangle\nonumber \\
 & \mbox{s.t.} & \left\langle\bar{\mb A}^{T}\bar{\mb A},\,\mathbf{X}\right\rangle-1\le0\label{eq:X=00003DY}\\
 &  & \mathbf{X}-\mathbf{Y}=\mathbf{0}\nonumber \\
 &  & \mathbf{Y}\in\Re_{+}^{n\times n},\,\mathbf{X}\in S_{+}^{n}.\nonumber
\end{eqnarray}
Let $\beta\in\Re$ and $\mathbf{\mu}\in\Re^{n\times n}$ correspond
to the inequality constraint and equality constraint. Then the dual
problem of \eqref{eq:X=00003DY} could be written in the following
min-max form:
\begin{equation}
\min_{\begin{array}{c}
\beta\ge0\\
\mathbf{\mu}\mbox{ free}
\end{array}}\max_{\begin{array}{c}
\mathbf{X}\in S_{+}^{n}\\
\mathbf{Y}\in\Re_{+}^{n\times n}
\end{array}}\left\langle\left(\widehat{\mathbf{A}}-\bar{\mb A}\right)^{T}\left(\widehat{\mathbf{A}}-\bar{\mb A}\right),\,\mathbf{X}\right\rangle-\beta\left(\langle\bar{\mb A}^{T}\bar{\mb A},\,\mathbf{X}\rangle-1\right)+\langle\mathbf{\mu},\,\mathbf{X}-\mathbf{Y}\rangle.\label{eq:minmax dual}
\end{equation}
By verifying that
\begin{eqnarray*}
 &  & \max_{\begin{array}{c}
\mathbf{X}\in S_{+}^{n},\mathbf{Y}\in\Re_{+}^{n\times n}\end{array}}\left\langle\left(\widehat{\mathbf{A}}-\bar{\mb A}\right)^{T}\left(\widehat{\mathbf{A}}-\bar{\mb A}\right),\,\mathbf{X}\right\rangle-\beta\left(\langle\bar{\mb A}^{T}\bar{\mb A},\,\mathbf{X}\rangle-1\right)+\langle\mathbf{\mu},\,\mathbf{X}-\mathbf{Y}\rangle\\
 & = & \max_{\begin{array}{c}
\mathbf{X}\in S_{+}^{n},\mathbf{Y}\in\Re_{+}^{n\times n}\end{array}}\left\langle\left(\widehat{\mathbf{A}}-\bar{\mb A}\right)^{T}\left(\widehat{\mathbf{A}}-\bar{\mb A}\right)-\beta\bar{\mb A}^{T}\bar{\mb A}+\mathbf{\mu},\,\mathbf{X}\right\rangle+\langle-\mathbf{\mu},\,\mathbf{Y}\rangle+\beta\\
 & = & \begin{cases}
\beta & \mbox{\mbox{if }}-\left(\widehat{\mathbf{A}}-\bar{\mb A}\right)^{T}\left(\widehat{\mathbf{A}}-\bar{\mb A}\right)+\beta\bar{\mb A}^{T}\bar{\mb A}-\mathbf{\mu}\in S_{+}^{n}\mbox{ and }\mathbf{\mu}\in\text{\ensuremath{\Re}}{}_{+}^{n\times n}\\
+\infty & \mbox{otherwise}
\end{cases},
\end{eqnarray*}
we can write \eqref{eq:minmax dual} as
\begin{eqnarray}
 & \min & \beta\nonumber \\
 & \mbox{s.t.} & -\left(\widehat{\mathbf{A}}-\bar{\mb A}\right)^{T}\left(\widehat{\mathbf{A}}-\bar{\mb A}\right)+\beta\bar{\mb A}^{T}\bar{\mb A}-\mathbf{\mu}\succeq\mathbf{0}\label{eq:middle_dual}\\
 &  & \mathbf{\mu}\ge\mathbf{0},\,\beta\ge0.\nonumber
\end{eqnarray}
Because of Schur's complement, $-\left(\widehat{\mathbf{A}}-\bar{\mb A}\right)^{T}\left(\widehat{\mathbf{A}}-\bar{\mb A}\right)+\beta\bar{\mb A}^{T}\bar{\mb A}-\mathbf{\mu}\succeq\mathbf{0}$
if and only if
\[
\left[\begin{array}{cc}
\mathbf{I} & \widehat{\mathbf{A}}-\bar{\mb A}\\
\left(\widehat{\mathbf{A}}-\bar{\mb A}\right)^{T} & \beta\bar{\mb A}^{T}\bar{\mb A}-\mathbf{\mu}
\end{array}\right]\succeq\mathbf{0}.
\]
 Thus \eqref{eq:middle_dual} is equivalent to
\begin{eqnarray*}
 & \min & \beta\\
 & \mbox{s.t.} & \left[\begin{array}{cc}
\mathbf{I} & \widehat{\mathbf{A}}-\bar{\mb A}\\
\left(\widehat{\mathbf{A}}-\bar{\mb A}\right)^{T} & \beta\bar{\mb A}^{T}\bar{\mb A}-\mathbf{\mu}
\end{array}\right]\succeq\mathbf{0}.\\
 &  & \mathbf{\mu}\ge\mathbf{0},\,\beta\ge0.
\end{eqnarray*}

Moreover, it can be easily verified that $\mathbf{X}=\frac{1}{2\langle\bar{\mb A}^{T}\bar{\mb A},\mathbf{1}\mathbf{1}^{T}+\mathbf{I}_{n\times n}\rangle}\left(\mathbf{1}\mathbf{1}^{T}+\mathbf{I}_{n\times n}\right)$
is a interior point in the feasible set of \eqref{eq:inner_optimization}.
Thus by Slater's condition, the duality gap is zero.
\end{proof}

\noindent Hence, instead of solving (\ref{eq:low_rank+sparse}),
we can work with
\[
\min\left\{ \|\mathbf{L}\|_{\star}+\lambda\|\mathbf{S}\|_{1}\;\middle|\;\mathbf{L}+\mathbf{S}=\widehat{\mathbf{A}},\,\widehat{\mathbf{A}}\in\Omega_{3}\right\}.
\]
The following lemma completes our proof of Theorem \ref{thm:cp-lrsd}:
\begin{lemma}Our relaxed convex optimization problem
\begin{equation}
\min\left\{ \|\mathbf{L}\|_{\star}+\lambda\|\mathbf{S}\|_{1}\;\middle|\;\mathbf{L}+\mathbf{S}=\widehat{\mathbf{A}},\,\widehat{\mathbf{A}}\in\Omega_{3}\right\} \label{eq:relaxed_problem}
\end{equation}
 is equivalent to problem \eqref{eqn:cp-lrsd}.
\end{lemma}

\begin{proof}

Problem (\ref{eq:relaxed_problem})

\begin{eqnarray*}
 & \min & \|\mathbf{L}\|_{*}+\lambda\|\mathbf{S}\|_{1}\\
 & {\rm {s.t.}} & \mathbf{L}+\mathbf{S}=\widehat{\mathbf{A}}\\
 &  & f\left(\widehat{\mathbf{A}}\right)\le\bar\gamma,
\end{eqnarray*}
can be easily written as
\begin{eqnarray}
 & \min & \|\mathbf{L}\|_{*}+\lambda\|\mathbf{S}\|_{1}\nonumber \\
 & {\rm {s.t.}} & \left[\begin{array}{cc}
\mathbf{I} &  \mathbf{L}+\mathbf{S}-\bar{\mb A}\\
\left( \mathbf{L}+\mathbf{S}-\bar{\mb A}\right)^{T} & \beta\bar{\mb A}^{T}\bar{\mb A}-\mathbf{\mu}
\end{array}\right]\succeq\mathbf{0}\label{eq:intermediate}\\
 &  & \mathbf{\mathbf{\mu}}\ge\mathbf{0},\,\bar\gamma\ge\beta\ge0.\nonumber
\end{eqnarray}

Whenever $\left(\mathbf{L}^{\star},\mathbf{S}^{\star},\beta^{\star},\mathbf{\mu}^{\star}\right)$
is an optimal solution to problem (\ref{eq:intermediate}), $\left(\mathbf{L}^{\star},\mathbf{S}^{\star},\bar\gamma,\mathbf{\mu}^{\star}\right)$is
still feasible by noting that
\[
\left[\begin{array}{cc}
\boldsymbol{I} & \mathbf{L}^{\star}+\mathbf{S}^{\star}-\bar{\mb A}\\
\left(\mathbf{L}^{\star}+\mathbf{S}^{\star}-\bar{\mb A}\right)^{T} & \bar\gamma\bar{\mb A}^{T}\bar{\mb A}-\mathbf{\mu}^{\star}
\end{array}\right]=\left[\begin{array}{cc}
\boldsymbol{I} & \mathbf{L}^{\star}+\mathbf{S}^{\star}-\bar{\mb A}\\
\left(\mathbf{L}^{\star}+\mathbf{S}^{\star}-\bar{\mb A}\right)^{T} & \beta^{\star}\bar{\mb A}^{T}\bar{\mb A}-\mathbf{\mu}^{\star}
\end{array}\right]+\left[\begin{array}{cc}
\boldsymbol{0} & \boldsymbol{0}\\
\boldsymbol{0} & \left(\bar\gamma-\beta^{\star}\right)\bar{\mb A}^{T}\bar{\mb A}
\end{array}\right]\succeq\boldsymbol{0}.
\]
Moreover, the objective value does not change. Thus $\left(\mathbf{L}^{\star},\mathbf{S}^{\star},\bar\gamma,\mathbf{\mu}^{\star}\right)$
is also an optimal solution. Therefore, we can rewrite problem (\ref{eq:intermediate}) as
\begin{eqnarray*}
 & \min & \|\mathbf{L}\|_{*}+\lambda\|\mathbf{S}\|_{1}\\
 & {\rm {s.t.}} & \left[\begin{array}{cc}
\mathbf{I} &  \mathbf{L}+\mathbf{S}-\bar{\mb A}\\
\left( \mathbf{L}+\mathbf{S}-\bar{\mb A}\right)^{T} & \bar\gamma\bar{\mb A}^{T}\bar{\mb A}-\mathbf{\mu}
\end{array}\right]\succeq\mathbf{0}\\
 &  & \mathbf{\mathbf{\mu}}\ge\mathbf{0},
\end{eqnarray*}
\end{proof}

\section{Scalable Complexity Reduction using L-ADMM}\label{sec:algorithm}
Per Nesterov's advice that ``\dots the proper use of the problem's
structure can lead to efficient optimization methods\dots'' \cite{nesterov2007gradient}, we would
like to search for a scalable algorithm that takes full advantage of the structure of \eqref{eqn:cp-lrsd}. The problem can be rephrased as
\begin{eqnarray}
 & \min & \|\mathbf{L}\|_{\star}+\lambda\|\mathbf{S}\|_{1}+\mathcal{I}\left(\mathbf{Z}\succeq\mathbf{0}\right)+\mathcal{I}\left(\mathbf{\mu}\ge\mathbf{0}\right)\label{eq:Linearized_ADMM-1}\\
 & \mbox{s.t.} & \mathbf{Z}-\left[\begin{array}{cc}
\mathbf{I} & \mathbf{L}+\mathbf{S}-\bar{\mb A}\\
\left(\mathbf{L}+\mathbf{S}-\bar{\mb A}\right)^{T} & \bar\gamma\bar{\mb A}^{T}\bar{\mb A}-\mathbf{\mu}
\end{array}\right]=0,\nonumber
\end{eqnarray}
where the indicator function $\mathcal{I}\left(x\in\mathcal{X}\right)$
is defined as
\[
\mathcal{I}\left(x\in\mathcal{X}\right)=\begin{cases}
0, & \mbox{if }x\in\mathcal{X}\\
+\infty, & \mbox{otherwise.}
\end{cases}
\]
The most important structure in \eqref{eq:Linearized_ADMM-1} seems to be that the objective function and constraints are {\em separable}. This naturally suggests the use of alteranting directions methods. We will adopt the recently proposed
Linearized Alternating Direction Method of Multipliers (L-ADMM) \cite{zhang2010bregmanized,zhang2011unified,ADMM_L_MA} for \eqref{eq:Linearized_ADMM-1}.
The L-ADMM is well adapted for problems of this form \eqref{eq:Linearized_ADMM-1}. This method works with the Augmented Lagrangian,
\begin{align*}
\mathcal{L}_{\rho}\left(\mathbf{Z},\mathbf{L},\mathbf{S},\mathbf{\mu};\mathbf{Y}\right) \quad\doteq\quad \|\mathbf{L}\|_{*}+\lambda\|\mathbf{S}\|_{1}+ & \mathcal{I}\left(\mathbf{Z}\succeq\mathbf{0}\right)+\mathcal{I}\left(\mathbf{\mu}\ge\mathbf{0}\right) \\ & +\left\langle\mathbf{Y},\,\mathbf{Z}-\left[\begin{array}{cc}
\mathbf{I} & \mathbf{L}+\mathbf{S}-\bar{\mb A}\\
\left(\mathbf{L}+\mathbf{S}-\bar{\mb A}\right)^{T} &  \bar\gamma\bar{\mb A}^{T}\bar{\mb A}-\mathbf{\mu}
\end{array}\right]\right\rangle\\
 & +\frac{\rho}{2}\left\|\mathbf{Z}-\left[\begin{array}{cc}
\mathbf{I} & \mathbf{L}+\mathbf{S}-\bar{\mb A}\\
\left(\mathbf{L}+\mathbf{S}-\bar{\mb A}\right)^{T} &  \bar\gamma\bar{\mb A}^{T}\bar{\mb A}-\mathbf{\mu}
\end{array}\right]\right\|_{F}^{2}.
\end{align*}
Here, $\mathbf{Y}$ is the multiplier of the linear constraint, and
$\rho>0$ is the penalty parameter. For notational convenience, partition
$\mathbf{Z}$ as
\[
\mb Z = \left[\begin{array}{cc}
\mathbf{Z}_{11} & \mathbf{Z}_{12}\\
\mathbf{Z}_{21} & \mathbf{Z}_{22}
\end{array}\right],
\]
in accordance with the block structure of
\[
\left[\begin{array}{cc}
\mathbf{I} & \mathbf{L}+\mathbf{S}-\bar{\mb A}\\
\left(\mathbf{L}+\mathbf{S}-\bar{\mb A}\right)^{T} & \bar\gamma\bar{\mb A}^{T}\bar{\mb A}-\mathbf{\mu}
\end{array}\right].
\]
Following the same rule, we define $\mathbf{Y}_{11}$, $\mathbf{Y}_{12}$,
$\mathbf{Y}_{21}$ and $\mathbf{Y}_{22}$ correspondingly.

The L-ADMM algorithm, operating on $\mathcal{L}_{\rho}\left(\mathbf{Z},\mathbf{L},\mathbf{S},\mathbf{\mu};\mathbf{Y}\right)$,
consists of the following three steps:
\begin{enumerate}
\item Minimize $\mathcal{L}_{\rho}\left(\mathbf{Z},\mathbf{L},\mathbf{S},\mathbf{\mu};\mathbf{Y}\right)$
with respect to $\mathbf{Z}$, while keeping all the other variables fixed:
\begin{eqnarray*}
\mathbf{Z}^{k+1} & = & \arg \min_{\mathbf{Z}} \; \mathcal{L}_{\rho}\left(\mathbf{Z},\mathbf{L}^{k},\mathbf{S}^{k},\mathbf{\mu}^{k};\mathbf{Y}^{k}\right)\\
 & = &\arg \min_{\mb Z} \; \mathcal{I}\left(\mathbf{Z}\succeq\mathbf{0}\right)+\frac{\rho}{2}\left\|\mathbf{Z}-\left(\left[\begin{array}{cc}
\mathbf{I} & \mathbf{L}^{k}+\mathbf{S}^{k}-\bar{\mb A}\\
\left(\mathbf{L}^{k}+\mathbf{S}^{k}-\bar{\mb A}\right)^{T} &  \bar\gamma\bar{\mb A}^{T}\bar{\mb A}-\mathbf{\mu}^{k}
\end{array}\right]-\frac{\mathbf{Y}^{k}}{\rho}\right)\right\|_{F}^{2}\\
 & = & \mathbf{Q}^{k}\left(\mathbf{\Lambda}^{k}\right)_{+}\left(\mathbf{Q}^{k}\right)^{T},
\end{eqnarray*}
where $\mathbf{Q}^{k}\mathbf{\Lambda}^{k}\left(\mathbf{Q}^{k}\right)^{T}$
is any eigenvalue decomposition of
\[
\left[\begin{array}{cc}
\mathbf{I} & \mathbf{L}^{k}+\mathbf{S}^{k}-\bar{\mb A}\\
\left(\mathbf{L}^{k}+\mathbf{S}^{k}-\bar{\mb A}\right)^{T} &  \bar\gamma\bar{\mb A}^{T}\bar{\mb A}-\mathbf{\mu}^{k}
\end{array}\right]-\frac{\mathbf{Y}^{k}}{\rho},
\]
and $\mathbf{\Lambda}^{k}=\mbox{diag}\left(\left\{ \lambda_{i}^{k}\right\} _{i=1}^{m+n}\right)$
and $\left(\mathbf{\Lambda}^{k}\right)_{+}=\mbox{diag}\left(\left\{ \max\left(\lambda_{i}^{k},0\right)\right\} _{i=1}^{m+n}\right)$.

\item Fix variables $\mathbf{Z}$ and $\mathbf{Y}$ and update $\mathbf{L}$, $\mathbf{S}$ and $\mathbf{\mu}$.
Instead of minimizing $\mathcal{L}_{\rho}(\mathbf{Z}^{k+1},\mathbf{L},\mathbf{S},\mathbf{\mu};\mathbf{Y}^{k})$
directly, we construct a surrogate function $\widehat{\mathcal{L}_{\rho}}(\mathbf{Z}^{k+1},\mathbf{L},\mathbf{S},\mathbf{\mu};\mathbf{Y}^{k})$
by linearizing $\mathcal{L}_{\rho}(\mathbf{Z}^{k+1},\mathbf{L},\mathbf{S},\mathbf{\mu};\mathbf{Y}^{k})$,
and then minimize $\widehat{\mathcal{L}_{\rho}}(\mathbf{Z}^{k+1},\mathbf{L},\mathbf{S},\mathbf{\mu};\mathbf{Y}^{k})$
in $\left(\mathbf{L},\mathbf{S},\mathbf{\mu}\right)$-direction.
This gives
\[
\left(\begin{array}{c}
\mathbf{L}^{k+1}\\
\mathbf{S}^{k+1}\\
\mathbf{\mu}^{k+1}
\end{array}\right)\;=\;\mbox{argmin}_{\left(\mathbf{L},\mathbf{S},\mathbf{\mu}\right)}\,\widehat{\mathcal{L}_{\rho}}\left(\mathbf{Z}^{k+1},\mathbf{L},\mathbf{S},\mathbf{\mu};\mathbf{Y}^{k}\right) \;=\; \left(\begin{array}{c}
\mbox{argmin}_{\mathbf{L}}\left(\|\mathbf{L}\|_{*}+\frac{\rho}{2\tau}\|\mathbf{L}-\mathbf{F}^{k}\|_{F}^{2}\right)\\
\mbox{argmin}_{\mathbf{S}}\left(\lambda\|\mathbf{S}\|_{1}+\frac{\rho}{2\tau}\|\mathbf{S}-\mathbf{G}^{k}\|_{F}^{2}\right)\\
\mbox{argmin}_{\mathbf{\mu}}\left(\mathcal{I}\left(\mathbf{\mu}\ge\mathbf{0}\right)+\frac{\rho}{2\tau}\|\mathbf{\mu}-\mathbf{K}^{k}\|_{F}^{2}\right)
\end{array}\right)
\]
where
\begin{align*}
\mathbf{F}^{k} & :=\mathbf{L}^{k}+2\tau\left(\mathbf{Z}_{12}^{k+1}-\mathbf{L}^{k}-\mathbf{S}^{k}+\bar{\mb A}+\mathbf{Y}_{12}^{k}/\rho\right)\\
\mathbf{G}^{k} & :=\mathbf{S}^{k}+2\tau\left(\mathbf{Z}_{12}^{k+1}-\mathbf{L}^{k}-\mathbf{S}^{k}+\bar{\mb A}+\mathbf{Y}_{12}^{k}/\rho\right)\\
\mathbf{K}^{k} & :=\mathbf{\mu}^{k}-\tau\left(\mathbf{Z}_{22}^{k+1}- \bar\gamma\bar{\mb A}^{T}\bar{\mb A}+\mathbf{\mu}+\mathbf{Y}_{22}^{k}/\rho\right)
\end{align*}
and $\tau\ge0$ is a given step size to be discussed later. The advantage of this linearization is that the sub-minimizations over $\mb L$, $\mb S$, and $\mb \mu$ have efficient, closed-form solutions.
Let $\mathcal{S}_{\theta}:\reals\to\reals$ denote the shrinkage
operator
\[
\mathcal{S}_{\theta}\left(x\right)=\mbox{sgn}\left(x\right)\max\left(|x|-\theta,0\right)
\]
and extend it to matrices by applying it componentwise. It is easy
to show that $\mathbf{S}^{k+1}=\mathcal{S}_{\tau\lambda/\rho}(\mathbf{G}^{k})$.
Similarly, for any matrix $\mathbf{X}$, denote $\mathcal{D}_{\theta}$
as the singular value thresholding operator $\mathcal{D}_{\theta}\left(\mathbf{X}\right)=\mathbf{U}\mathcal{S}_{\theta}\left(\mathbf{\Sigma}\right)\mathbf{V}^{\star}$
where $\mathbf{X}=\mathbf{U}\mathbf{\Sigma}\mathbf{V}^{\star}$
is any singular value decomposition. It is not difficult to show that
$\mathbf{L}^{k+1}=\mathcal{D}_{\tau/\rho}(\mathbf{F}^{k})$.
For $\mathbf{\mu}^{k+1}$ , we simply have $\mathbf{\mu}^{k+1}=(\mathbf{K}^{k})_{+}=[\max(\mathbf{K}_{ij}^{k},0)]_{ij}$.

\item Update the dual multiplier
\[
\mathbf{Y}^{k+1}=\mathbf{Y}^{k}+\rho\left(\mathbf{Z}^{k+1}-\left[\begin{array}{cc}
\mathbf{I} & \mathbf{L}^{k+1}+\mathbf{S}^{k+1}-\bar{\mb A}\\
\left(\mathbf{L}^{k+1}+\mathbf{S}^{k+1}-\bar{\mb A}\right)^{T} &  \bar\gamma\bar{\mb A}^{T}\bar{\mb A}-\mathbf{\mu}^{k+1}
\end{array}\right]\right)
\]
\end{enumerate}

Putting these results together,  we obtain an L-ADMM iterates as follows:

\begin{algorithm}[h]
\textbf{Step 1. }Generate $\mathbf{Z}^{k+1}$:
\[
\mathbf{Z}^{k+1}=\mathbf{Q}^{k}\left(\mathbf{\Lambda}^{k}\right)_{+}\left(\mathbf{Q}^{k}\right)^{T}.
\]

\textbf{Step 2. }Generate $\mathbf{L}^{k+1}$, $\mathbf{S}^{k+1}$
and $\mathbf{\mu}^{k+1}$:
\[
\begin{cases}
\mathbf{L}^{k+1}=\mathcal{D}_{\tau/\rho}\left(\mathbf{F}^{k}\right)\\
\mathbf{S}^{k+1}=\mathcal{S}_{\tau\lambda/\rho}\left(\mathbf{G}^{k}\right)\\
\mathbf{\mu}^{k+1}=(\mathbf{K}^{k})_{+}
\end{cases}
\]

\textbf{Step 3. }Update the multiplier $\mathbf{Y}^{k+1}$:
\[
\mathbf{Y}^{k+1}=\mathbf{Y}^{k}+\rho\left(\mathbf{Z}^{k+1}-\left[\begin{array}{cc}
\mathbf{I} & \mathbf{L}^{k+1}+\mathbf{S}^{k+1}-\bar{\mb A}\\
\left(\mathbf{L}^{k+1}+\mathbf{S}^{k+1}-\bar{\mb A}\right)^{T} &  \bar\gamma\bar{\mb A}^{T}\bar{\mb A}-\mathbf{\mu}^{k+1}
\end{array}\right]\right).
\]

\caption{\label{alg:Linearized-ADMM}L-ADMM for \eqref{eq:Linearized_ADMM-1} }
\end{algorithm}

It can be shown that, with a proper choice of $\tau$, our L-ADMM algorithm converges globally with rate $O\left(1/k\right)$. The proper $\tau$ is dictated by the following lemma, which bounds the norm of the operator in the linear constraint in \eqref{eq:Linearized_ADMM-1}:

\begin{lemma} \label{lem:G-op} Let $\mc G : \reals^{m\times n}\times \reals^{m\times n}\times\reals^{n\times n}\to\reals^{\left(m+n\right)\times\left(m+n\right)}$
such that
\[
\mc G\left(\mathbf{L},\mathbf{S},\mathbf{\mu}\right):=\left[\begin{array}{cc}
\mathbf{0} & \mathbf{L}+\mathbf{S}\\
\left(\mathbf{L}+\mathbf{S}\right)^T & -\mathbf{\mu}
\end{array}\right].
\]
Then we have operator norm $\|\mc G\|=2$.
\end{lemma}

\begin{proof} Set
$\mc G_{1}\left(\mathbf{L+S}\right)\doteq\left[\begin{array}{cc}
\mathbf{0} & \mathbf{L+S}\\
\left(\mathbf{L+S}\right)^T & \mathbf{0}
\end{array}\right]$ and $\mc G_{2}\left(\mathbf{\mu}\right)\doteq\left[\begin{array}{cc}
\mathbf{0} & \mathbf{0}\\
\mathbf{0} & -\mathbf{\mu}
\end{array}\right]$. It can be easily verified that operator norms $\|\mc G_{1}\|=2$
and $\|\mc G_{2}\|=1$. Then we have
\begin{eqnarray*}
 & \|\mc G\|^{2} & =\sup_{\|\mathbf{\mu}\|_{F}^{2}+\|\left(\mathbf{L},\mathbf{S}\right)\|_{F}^{2}=1}\left\|\left[\begin{array}{cc}
\mathbf{0} & \mathbf{L+S}\\
\left(\mathbf{L+S}\right)^T & -\mathbf{\mu}
\end{array}\right]\right\|_{F}^{2}\\
 &  & =\sup_{\|\mathbf{\mu}\|_{F}^{2}+\|\left(\mathbf{L},\mathbf{S}\right)\|_{F}^{2}=1}\left(\left\|\left[\begin{array}{cc}
\mathbf{0} & \mathbf{L+S}\\
\left(\mathbf{L+S}\right)^T & \mathbf{0}
\end{array}\right]\right\|_{F}^{2}+\left\|\left[\begin{array}{cc}
\mathbf{0} & \mathbf{0}\\
\mathbf{0} & -\mathbf{\mu}
\end{array}\right] \right\|_{F}^{2}\right)\\
 &  & =\sup_{\|\mathbf{\mu}\|_{F}^{2}+\|\left(\mathbf{L},\mathbf{S}\right)\|_{F}^{2}=1}\left(\|\mc G_{1}\|^{2}\|\left(\mathbf{L},\mathbf{S}\right)\|_{F}^{2}+\|\mc G_{2}\|^{2}\|\mathbf{\mu}\|_{F}^{2}\right)\\
 &  & =\max\left(\|\mc G_{1}\|^{2},\|\mc G_{2}\|^{2}\right)\\
 &  & =4,
\end{eqnarray*}
\end{proof}
Combining it with convergence results from Appendix A of \cite{ADMM_L_MA}, and results on convergence rate from \cite{he20121} (Theorem 4.1), we obtain the following convergence guarantee for our algorithm:

\begin{theorem}(Convergence Results) Suppose $0<\tau<0.25$. Then
the sequence $\left\{ \left(\mathbf{Z}^{k},\mathbf{L}^{k},\mathbf{S}^{k},\mathbf{\mu}^{k}\right)\right\} $
produced by Alg-\ref{alg:Linearized-ADMM} from any starting point
converges to an optimal solution with rate $O\left(1/k\right)$.\end{theorem}

%% file: arxiv_draft.bbl
\newcommand{\etalchar}[1]{$^{#1}$}
\begin{thebibliography}{WWG{\etalchar{+}}12}

\bibitem[BI76]{Bronshteyn1976}
E.~M. Bronshteyn and L.~D. Ivanov.
\newblock The approximation of convex sets by polyhedra.
\newblock {\em Siberian Mathematical Journal}, 16(5):852--853, 1976.

\bibitem[BJ03]{Basri2003-PAMI}
R.~Basri and D.~W. Jacobs.
\newblock Lambertian reflectance and linear subspaces.
\newblock {\em IEEE Transactions Pattern Analysis and Machine Intelligence
  ({PAMI})}, 25(2):218--233, 2003.

\bibitem[BK98]{Belhumeur1998-IJCV}
P.~N. Belhumeur and D.~J. Kriegman.
\newblock What is the set of images of an object under all possible
  illumination conditions?
\newblock {\em International Journal of Computer Vision ({IJCV})},
  28(3):245--260, 1998.

\bibitem[CET01]{AAM}
T.~F. Cootes, G.~J. Edwards, and C.~J. Taylor.
\newblock Active appearance models.
\newblock {\em IEEE Transactions Pattern Analysis and Machine Intelligence
  ({PAMI})}, 23(6):681--685, 2001.

\bibitem[CLMW11]{Candes2011-JACM}
E.~Cand{\`{e}s}, X.~Li, Y.~Ma, and J.~Wright.
\newblock Robust principal component analysis?
\newblock {\em Journal of the Association for Computing Machinery}, 58(3),
  2011.

\bibitem[Con90]{Conway}
J.~B. Conway.
\newblock {\em A Course in Functional Analysis}.
\newblock Springer, 1990.

\bibitem[Cos00]{coste2000introduction}
M.~Coste.
\newblock {\em An Introduction to Semialgebraic Geometry}.
\newblock Istituti Editoriali e Poligrafici Internazionali, 2000.

\bibitem[CSPW11]{Chandrasekharan2011-SJO}
V.~Chandrasekaran, S.~Sanghavi, P.~Parillo, and A.~Wilsky.
\newblock Rank-sparsity incoherence for matrix decomposition.
\newblock {\em SIAM Journal on Optimization}, 21(2):572--596, 2011.

\bibitem[CYZ{\etalchar{+}}05]{Chen2005}
T.~Chen, W.~Yin, X.~S. Zhou, D.~Domaniciu, , and T.~S. Huang.
\newblock Illumination normalization for face recognition and uneven background
  correction using total variation based image models.
\newblock In {\em IEEE Conference on Computer Vision and Pattern Recognition
  ({CVPR})}, 2005.

\bibitem[EHY95]{Epstein1995}
R.~Epstein, P.W. Hallinan, and A.L. Yuille.
\newblock $5 \pm 2$ eigenimages suffice: an empirical investigation of
  low-dimensional lighting models.
\newblock In {\em Physics-Based Modeling in Computer Vision}, 1995.

\bibitem[FSB04]{Frolova04-ECCV}
D.~Frolova, D.~Simakov, and R.~Basri.
\newblock Accuracy of spherical harmonic approximations for images of
  lambertian objects under far and near lighting.
\newblock In {\em European Conference on Computer Vision ({ECCV})}, 2004.

\bibitem[GBK01]{Georghiades2001-PAMI}
A.~S. Georghiades, P.~N. Belhumeur, and D.~J. Kriegman.
\newblock From few to many: Illumination cone models for face recognition under
  variable lighting and pose.
\newblock {\em IEEE Transactions Pattern Analysis and Machine Intelligence
  ({PAMI})}, 23(6):643--660, 2001.

\bibitem[Hor86]{Horn}
B.~K. Horn.
\newblock {\em Robot Vision}.
\newblock McGraw-Hill Higher Education, 1986.

\bibitem[HY12]{he20121}
B.~He and X.~Yuan.
\newblock On the o(1/n) convergence rate of the douglas-rachford alternating
  direction method.
\newblock {\em SIAM Journal on Numerical Analysis}, 50(2):700--709, 2012.

\bibitem[LHK05]{Lee2005-PAMI}
K.-C. Lee, J.~Ho, and D.~J. Kriegman.
\newblock Acquiring linear subspaces for face recognition under variable
  lighting.
\newblock {\em IEEE Transactions Pattern Analysis and Machine Intelligence
  ({PAMI})}, 27(5):684--698, 2005.

\bibitem[Low04]{Lowe2004-IJCV}
D.~G. Lowe.
\newblock Distinctive image features from scale-invariant keypoints.
\newblock {\em International Journal of Computer Vision ({IJCV})},
  60(2):91--110, 2004.

\bibitem[MLJ09]{Mei2009-ICCV}
X.~Mei, H.~Ling, and D.~W. Jacobs.
\newblock Sparse representation of cast shadows via l1-regularized least
  squares.
\newblock In {\em IEEE International Conference on Computer Vision ({ICCV})},
  2009.

\bibitem[MXZ12]{ADMM_L_MA}
S.~Ma, L.~Xue, and H.~Zou.
\newblock Alternating direction methods for latent variable gaussian graphical
  model selection.
\newblock {\em arXiv preprint arXiv:1206.1275}, 2012.

\bibitem[Nes07]{nesterov2007gradient}
Y.~Nesterov.
\newblock Gradient methods for minimizing composite objective function.
\newblock {\em Center for Operations Research and Econometrics}, 2007.

\bibitem[Ram02]{Ramamoorthi2002-PAMI}
R.~Ramamoorthi.
\newblock Analytic pca construction for theoretical analysis of lighting
  variability in images of a lambertian object.
\newblock {\em IEEE Transactions Pattern Analysis and Machine Intelligence
  ({PAMI})}, 24(10):1322--1333, 2002.

\bibitem[RKB]{Ramamoorthi2005-PAMI}
R.~Ramamoorthi, M.~Koudelka, and P.~Belhumeur.
\newblock A fourier theory for cast shadows.
\newblock {\em IEEE Transactions Pattern Analysis and Machine Intelligence
  ({PAMI})}, 27(2):288--295.

\bibitem[SRR01]{Shashua2001-PAMI}
A.~Shashua and T.~Riklin-Raviv.
\newblock The quotient image: Class-based re-rendering and recognition with
  varying illuminations.
\newblock {\em IEEE Transactions Pattern Analysis and Machine Intelligence
  ({PAMI})}, 23(2):129--139, 2001.

\bibitem[SSTB12]{Bosphorus}
A.~Savran, B.~Sankur, and M.~Taha~Bilge.
\newblock Comparative evaluation of 3d vs. 2d modality for automatic detection
  of facial action units.
\newblock {\em Pattern Recognition}, 45(2):767--782, 2012.

\bibitem[WGS{\etalchar{+}}10]{Wu2010-ACCV}
L.~Wu, A.~Ganesh, B.~Shi, Y.~Matsushita, Y.~Wang, and Y.~Ma.
\newblock Robust photometric stereo via low-rank matrix completion and
  recovery.
\newblock In {\em Asian Conference on Computer Vision ({ACCV})}, 2010.

\bibitem[WWG{\etalchar{+}}12]{Wagner2012-PAMI}
A.~Wagner, J.~Wright, A.~Ganesh, Z.~Zhou, H.~Mobahi, and Y.~Ma.
\newblock Toward a practical face recognition system: Robust alignment and
  illumination by sparse representation.
\newblock {\em IEEE Transactions Pattern Analysis and Machine Intelligence
  ({PAMI})}, 34(2):372--386, 2012.

\bibitem[WYG{\etalchar{+}}09]{Wright2009-PAMI}
J.~Wright, A.~Yang, A.~Ganesh, S.~Sastry, and Y.~Ma.
\newblock Robust face recognition via sparse representation.
\newblock {\em IEEE Transactions Pattern Analysis and Machine Intelligence
  ({PAMI})}, 31(2):210--227, 2009.

\bibitem[WZL{\etalchar{+}}]{Wang2009-PAMI}
Y.~Wang, L.~Zhang, Z.~Liu, G.~Hua, Z.~Wen, Z.~Zhang, and D.~Samaras.
\newblock Face relighting from a single image under arbitrary unknown lighting
  conditions.
\newblock {\em IEEE Transactions Pattern Analysis and Machine Intelligence
  ({PAMI})}, 31(11):1968--1984.

\bibitem[ZBBO10]{zhang2010bregmanized}
X.~Zhang, M.~Burger, X.~Bresson, and S.~Osher.
\newblock Bregmanized nonlocal regularization for deconvolution and sparse
  reconstruction.
\newblock {\em SIAM Journal on Imaging Sciences}, 3(3):253--276, 2010.

\bibitem[ZBO11]{zhang2011unified}
X.~Zhang, M.~Burger, and S.~Osher.
\newblock A unified primal-dual algorithm framework based on bregman iteration.
\newblock {\em Journal of Scientific Computing}, 46(1):20--46, 2011.

\bibitem[ZYZ{\etalchar{+}}13]{Zhuang2013-CVPR}
L.~Zhuang, A.~Yang, Z.~Zhou, S.~Sastry, and Y.~Ma.
\newblock Single-sample face recognition with image corruption and misalignment
  via sparse illumination transfer.
\newblock In {\em Computer Vision and Pattern Recognition ({CVPR})}, 2013.

\end{thebibliography}
